\documentclass{article}

\usepackage[final]{neurips_2021}

\usepackage[utf8]{inputenc} 
\usepackage[T1]{fontenc}    
\usepackage{hyperref}       
\usepackage{url}            
\usepackage{booktabs}       
\usepackage{amsfonts}       
\usepackage{nicefrac}       
\usepackage{microtype}      
\usepackage{xcolor}         
\usepackage{lscape}

\usepackage{amsmath}
\usepackage{amsthm}
\usepackage{amssymb}
\usepackage{mathabx} 
\usepackage{nccmath}
\usepackage{empheq} 
\usepackage{color,colortbl}
\usepackage{nicefrac}
\usepackage{multirow}
\usepackage{rotating}
\usepackage{bm}
\usepackage{dsfont}
\usepackage{hyperref}
\usepackage{subcaption}
\usepackage{xspace}

\newtheorem{lemma}{Lemma}
\newtheorem{theorem}{Theorem}

\DeclareMathOperator*{\argmax}{{argmax}}

\newcommand{\Rb}{\mathbb{R}}
\newcommand{\Dc}{\mathcal{D}}

\newcommand{\Xc}{\mathcal{X}}
\newcommand{\Yc}{\mathcal{Y}}
\newcommand{\Pc}{\mathcal{P}}
\newcommand{\sumj}[1]{\sum_{j=1}^M#1}
\newcommand{\prodj}[1]{\prod_{j=1}^M#1}
\newcommand{\sumi}[1]{\sum_{i=1}^n#1}
\newcommand{\Ex}{\mathbb{E}}
\newcommand{\wa}{w(\alpha)}
\newcommand{\ca}{c(\alpha)}

\newcommand*\pFq[2]{{}_{#1}F_{#2}}

\newcommand{\kl}{{\rm kl}}

\newcommand{\eqdef}{\triangleq}
\newcommand{\ie}{{\it i.e.}\xspace}
\newcommand{\wrt}{{\it w.r.t.}\xspace}
\newcommand{\eg}{{\it e.g.}\xspace}
\newcommand{\aka}{{\it a.k.a.}\xspace}
\newcommand{\coolname}{{\sc StocMV\xspace}}

\definecolor{mkfred}{rgb}{0.7, 0.0, 0.0}

\title{Learning Stochastic Majority Votes by\\Minimizing a PAC-Bayes Generalization Bound}

\author{%
  Valentina Zantedeschi$^{123}$ \hfill Paul Viallard$^4$ \hfill Emilie Morvant$^4$ \hfill R{\'e}mi Emonet$^4$\\
  \textbf{Amaury Habrard$^4$ \hfill Pascal Germain$^5$ \hfill Benjamin Guedj$^{123}$}\\
 ${}^1$ Inria, Lille - Nord Europe research centre, France \hfill
${}^2$ The Inria London Programme, France and UK\\
${}^3$ University College London, Department of Computer Science, Centre for Artificial Intelligence, UK\\
${}^4$ Univ Lyon, UJM-Saint-Etienne, CNRS, Institut d Optique Graduate School,\protect\\ Laboratoire Hubert Curien UMR 5516, F-42023, Saint-Etienne, France\\
${}^5$ Département d'informatique et de génie logiciel, Université Laval, Québec, Canada%
}

\begin{document}
\date{}
\maketitle

\begin{abstract}
We investigate a stochastic counterpart of majority votes over finite ensembles of classifiers, and study its generalization properties. While our approach holds for arbitrary distributions, we instantiate it with Dirichlet distributions: this allows for a closed-form and differentiable expression for the expected risk, which then turns the generalization bound into a tractable training objective.
The resulting stochastic majority vote learning algorithm 
achieves state-of-the-art accuracy and benefits from (non-vacuous) tight generalization bounds, in a series of numerical experiments when compared to competing algorithms which also minimize PAC-Bayes objectives -- both with uninformed (data-independent) and informed (data-dependent) priors.
\end{abstract}


\section{Introduction}
\label{sec:intro}

By combining the outcomes of several predictors, ensemble methods~\citep{dietterich2000ensemble} have been shown to provide models that are more accurate and more robust than each predictor taken singularly.
The key to their success lies in harnessing the diversity of the set of predictors~\citep{Kuncheva2014}.
Among ensemble methods, weighted Majority Votes (MV) classifiers assign a score to each base classifier (\aka\ voter) and output the most common prediction, given by the weighted majority.
When voters have known probabilities of making an error and make independent predictions, the optimal weighting is given by the so-called Naive Bayes rule~\citep{berend15a}. 
However, in most situations these assumptions are not satisfied, giving rise to the need for techniques that estimate the optimal combination of voter predictions from the data.

Among them, PAC-Bayesian based methods are 
well-grounded approaches for optimizing the voter weighting.
Indeed, PAC-Bayes theory (introduced by \citealp{TaylorWilliamson1997,McAllester1999} -- we refer to \citealp{guedj2019primer} for a recent survey and references therein) provides not only bounds on the true error of a MV through 
Probably Approximately Correct (PAC) generalization bounds (see \eg \cite{Catoni2007,seeger2002pac,maurer2004note,langford2003pac,germain2015risk}), but is also suited to derive theoretically grounded learning algorithms (see \eg \citet{Germain2009ICML},
\citet{roy2011pac,parrado2012pac,alquier2016properties}).
Contrary to the most classical PAC bounds~\citep{valiant1984theory}, as VC-dimension~\citep{VapnikBook} or Rademacher-based bounds~\citep{MohriBook}, PAC-Bayesian guarantees do not stand for all hypotheses (\ie are not expressed as a worst-case analysis) but stand in expectation over the hypothesis set. 
They involve a hypothesis space (formed by the base predictors), a prior distribution on it (\ie an {\it a priori} weighting) and a posterior distribution (\ie an {\it a posteriori} weighting) evaluated on a learning sample.
The prior brings some prior knowledge on the combination of predictors, and the posterior distribution is learned (adjusted) to lead to good generalization guarantees; the deviation between the prior and the posterior distributions plays a role in generalization guarantee and is usually captured by the Kullback-Leibler (KL) divergence.
In their essence, PAC-Bayesian results do not bound directly
the risk of the deterministic MV, but bound the expected risk of one (or several) base voters randomly drawn according to the weight distribution of the MV~\citep{langford2003pac,lacasse2006pac,lacasse2010learning,germain2015risk,NEURIPS2020_38685413}.
This randomization scheme leads to
upper bounds on the true risk of the MV that are then used as a proxy to derive PAC-Bayesian generalization bounds.
However, the obtained risk certificates are generally not tight, as they depend on irreducible constant factors, and when optimized they can lead to sub-optimal weightings.
Indeed, by considering a random subset of base predictors, state-of-the-art methods do not fully leverage the diversity of the whole set of voters.
This is especially a problem when the voters are weak, and learning to combine their predictions is critical for good performance.

\paragraph{Our contributions.}
In this paper, we propose a new randomization scheme.
We consider the voter weighting associated to a MV as a realization of a distribution of voter weightings.
More precisely, we analyze with the PAC-Bayesian framework the expected risk of a MV drawn from the posterior distribution of MVs. 
The main difference with the literature is that we propose a stochastic MV, while previous works aim at studying randomized evaluations of the true risk of the deterministic MV.
Doing so, we are able to derive tight empirical PAC-Bayesian bounds for our model directly on its expected risk, in Section~\ref{sec:bounds}.
We further propose, in Section~\ref{sec:stoc} two approaches for optimizing the generalization bounds, hence learning the optimal posterior:
the first optimizes an analytical and differentiable form of the empirical risk that can be derived when considering Dirichlet distributions;
the second optimizes a Monte Carlo approximation of the expected risk and can be employed with any form of posterior.
In our experiments, reported in Section~\ref{sec:expe}, we first compare these two approaches, highlighting in which regimes one is preferable to the other.
Finally, we assess our method's performance on real benchmarks \wrt the performance of PAC-Bayesian approaches also learning MV classifiers.
These results indicate that our models enjoy generalization bounds that are consistently tight and non-vacuous both when studying ensembles of data-independent predictors and when \mbox{studying ensembles of data-dependent ones}.

\paragraph{Societal impact.} Our work abides by the ethical guidelines enforced in contemporary research in machine learning. 
Given the theoretical nature of our contributions we do not foresee immediate potentially negative societal impact.



\section{Notation and background}
In this section, we formally define weighted Majority Vote (MV) classifiers and review the principal PAC-Bayesian approaches for learning them.

\subsection{Weighted majority vote classifiers}
\label{sec:related}
Consider the data random variable $(X, Y)$, taking values in $\Xc {\times} \Yc$ with $\Xc\subseteq \Rb^d$ a $d$-dimensional representation space and $\Yc$ the set of labels. 
We denote $\Pc$ the (unknown) data distribution of $(X, Y)$.
We define a set (dictionary) of base classifiers
$D {=} \{h_j: \Xc \to \Yc\}_{j=1}^M$.
The weighted majority vote classifier is a convex combination of the base classifiers from $D$.
Formally, a MV is parameterized by a weight vector $\theta \in [0,1]^M$, such that $\sumj \theta_j = 1$ hence lying in the ($M\text{-}1$)-simplex $\Delta^{M-1}$, as follows:
\begin{align}
    f_\theta(x) &= \argmax_{y \in \Yc} 
    \sumj \theta_j \; \mathds{1}( h_j(x) = y),
\end{align}
where $\mathds{1}(\cdot)$ is the indicator function.
Let $W_\theta(X, Y)$ be the random variable corresponding to the total weight assigned to base classifiers that predict an incorrect label on $(X,Y)$, that is
\begin{equation} 
W_\theta(X, Y) = \sumj \theta_j \mathds{1}(h_j(X) \neq Y).
\end{equation}
In binary classification with $|\Yc|{=}2$, the MV errs 
whenever $W_\theta(X, Y)\geq 0.5$ \citep{lacasse2010learning, NEURIPS2020_38685413}.
Hence the true risk  (\wrt\ $01$-loss) of the MV classifier can be expressed as 
\begin{align}\label{eq:truerisk}
    &R(f_\theta) \eqdef \Ex_\Pc \;\mathds{1}(W_\theta(X, Y) \geq 0.5) = \mathbb{P}(W_\theta \geq 0.5). 
\end{align}
Similarly, the empirical risk of $f_\theta$ on a $n$-sample $S {=} \{(x_i, y_i) {\sim} \Pc\}_{i=1}^n$ is given by 
\begin{equation*}\label{eq:emprisk}
    \hat{R}(f_\theta) = \sumi \;\mathds{1}(W_\theta(x_i, y_i) \geq 0.5).
\end{equation*}
Note that the results we introduce in the following are stated for binary classification, but are valid also in the multi-class setting ($|\Yc|{>}2$).
Indeed, in this context Equation~\eqref{eq:truerisk} becomes a surrogate of the risk: we have $R(f_\theta) \leq \mathbb{P}(W_\theta \geq 0.5)$ (see the notion of $\omega$-margin with $\omega=0.5$ proposed by~\citet{laviolette2017risk}).



\begin{figure}[t]
\centering
\begin{minipage}{.35\textwidth}
  \centering
  \includegraphics[width=\textwidth]{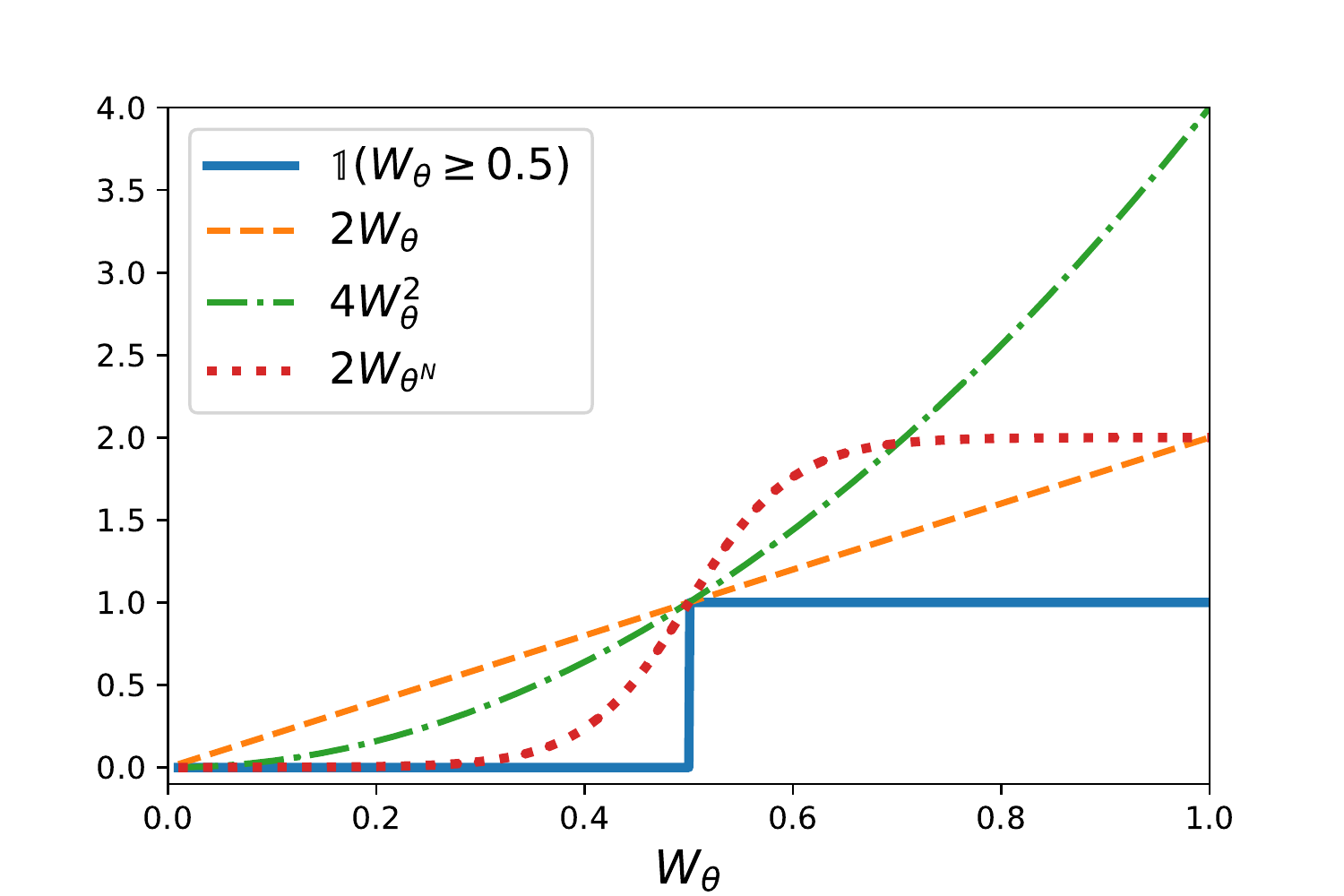}
  \caption{Oracle upper bounds for the true risk. The risks are the areas under the respective curves, for an arbitrary distribution of $W_\theta$ (typically different from the uniform).\label{fig:moments}}
\end{minipage}\hfill
\begin{minipage}{.61\textwidth}
  \centering
  \includegraphics[width=\textwidth]{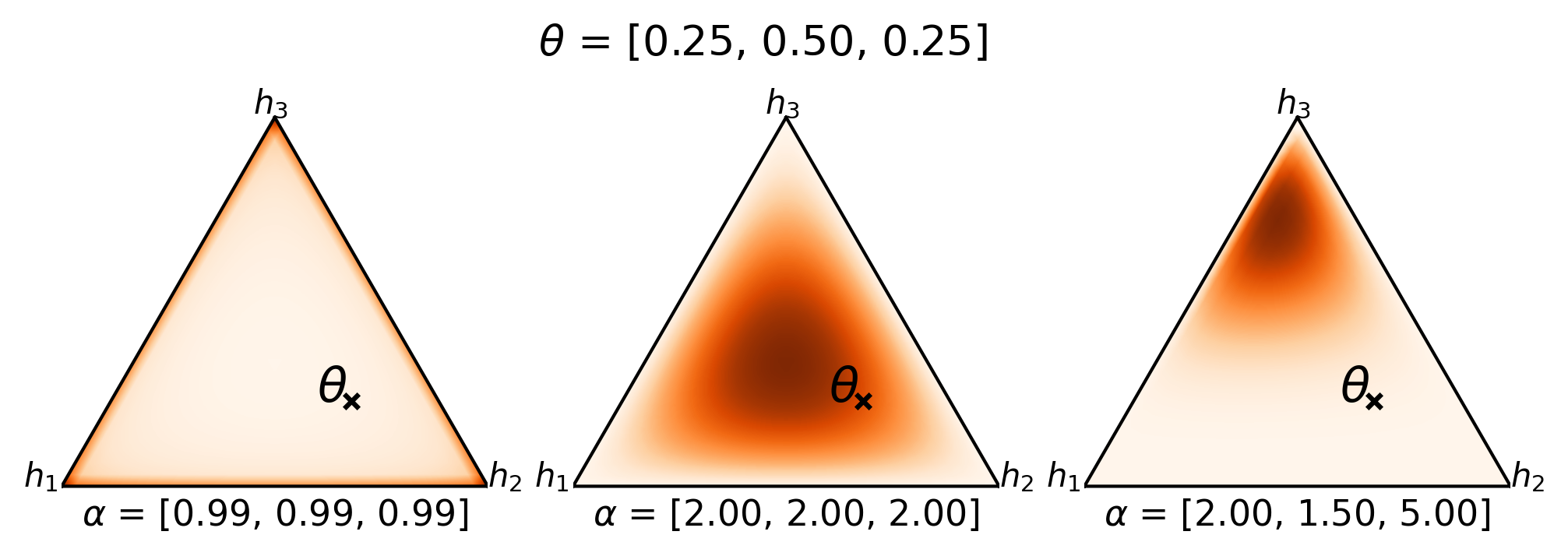}
    \caption{Visualization of the density measure $\rho : \Delta^2 \to \Rb_+$ taking the form of a Dirichlet distribution, with concentration parameters $\alpha$. 
    The darker the color, the higher $\rho(\theta)$.
    Each $\theta$ on the simplex corresponds to a majority vote classifier $f_\theta$ and has an associated probability depending on~$\alpha$. \label{fig:simplex}}
\end{minipage}
\end{figure}

\subsection{A PAC-Bayesian perspective on the majority vote}

The current paper stands in the context of PAC-Bayesian MV learning.
The PAC-Bayesian framework has been employed to study the generalization guarantees of randomized classifiers.
It is known to provide tight risk certificates that can be used to derive self-bounding learning algorithms~\cite[\eg][]{Germain2009ICML,dziugaite2017computing,perez2020tighter}.
In the PAC-Bayes literature of the analysis of MV, 
$\theta$ is interpreted as the parameter of a categorical distribution $\mathcal{C}(\theta)$ over the set of base classifiers $D$ \cite[\eg][]{germain2015risk,lorenzen2019pac,NEURIPS2020_38685413}.
In this sense, $f_\theta$ corresponds to the MV 
predictor $f_\theta(x) {=} \argmax_{y \in \Yc} \Ex_{h \sim \mathcal{C}(\theta)}  \mathds{1}( h(x) {=} y)$ and $W_\theta$ corresponds to the expected ratio of wrong predictors $W_\theta(X, Y) = \Ex_{h \sim \mathcal{C}(\theta)} \mathds{1}(h(X) \neq Y)$.
The PAC-Bayesian analysis provides a sensible way to find such a categorical distribution, called \textit{posterior} distribution, that leads to a model with low true risk $R(f_\theta)$.
However, an important caveat is that the PAC-Bayesian generalization bounds cannot be derived directly on the true risk $R(f_\theta)$, without making assumptions on the distribution of $W_\theta$
and raising fidelity problems.
Thus, a common approach is to consider oracle\footnote{Oracle bounds are expressed in terms of the unknown data distribution; their exact value cannot be computed.} upper bounds on the true risk in terms of quantities from which PAC-Bayesian  bounds can be derived, that are typically related to the statistical moments of $W_\theta$. 
By doing so, oracle bounds act as a proxy for estimating the cumulative density function for $W_\theta \geq 0.5$ \citep{langford2003pac, germain2007pac,lacasse2010learning,NEURIPS2020_38685413}.
Generalization guarantees for $R(f_\theta)$ are hence derived, involving the empirical counterpart of the oracle bound, the KL-divergence between the \emph{posterior} categorical distribution $\mathcal{C}(\theta)$ and a \emph{prior} one. 
An overview of the existing oracle bounds described below is represented in Figure~\ref{fig:moments}.

\paragraph{First order bound.}\label{sec:FO}
The most classical ``factor two'' oracle bound~\citep{langford2003pac} is derived considering the relation between MV's risk and the first moment of $W_\theta$, \aka~the expected risk of the randomized classifier:
$$R_1(\theta) \eqdef  \Ex_{h \sim \mathcal{C}(\theta)} \mathds{1}(h(X) {\neq} Y ) = \Ex_\Pc \;W_\theta.$$
By applying Markov's inequality, we have $R(f_\theta) \leq 2 R_1(\theta)$.
This ``factor two'' bound is close to zero only if the
expectation of the risk of a single base classifier drawn according to $\mathcal{C}(\theta)$ is itself close to zero. 
Therefore, it does not take into account the correlation between base predictors, which is key to characterize how a MV classifier can achieve strong predictions even when its base classifiers are individually weak (as observed when performing \eg \emph{Boosting}~\citep{schapire1999boosting}).
This explains why $R_1(\theta)$ can be a very loose estimate of $R(f_\theta)$ when the base classifiers \mbox{are adequately decorrelated}.  

\paragraph{Binomial bound.}\label{sec:Bin}
A generalization of the first order approach was proposed in~\citet{DBLP:journals/jmlr/Shawe-TaylorH09,lacasse2010learning}, where the true risk of the MV is estimated by drawing multiple ($N$) base hypotheses and computing the probability that at least $\frac{N}{2}$ make an error (which is given by the binomial with parameter $W_\theta$):
$$W_{\theta^N}(X, Y) \eqdef \sum_{k=\frac{N}{2}}^N {N \choose k} W_\theta^k (1-W_\theta)^{(N-k)}.$$
The higher $N$, the better $W_{\theta^N}(X, Y)$ approximates the true risk, but the looser the bound, as the KL term worsens by a factor of $N$.
Moreover, with this approach it is not possible to derive generalization bounds directly on the true risk, and the corresponding oracle bound still presents a factor two: $R(f_\theta) \leq 2\; \Ex_\Pc W_{\theta^N}$.

\paragraph{Second order bound.}\label{sec:SO}
A parallel line of works focuses on improving the bounds by accounting for voter correlations, \ie considering the agreement and/or disagreement of two random voters.
\cite{NEURIPS2020_38685413} recently proposed a new upper bound for the true risk depending on the second moment of $W_\theta$,
\aka~tandem loss or joint error, the risk that two predictors make a mistake on the same point: 
$$R_2(\theta) {\eqdef} \Ex_\Pc \;\Ex_{h \sim \mathcal{C}(\theta), h' \sim \mathcal{C}(\theta)} \mathds{1}(h(X) {\neq} Y \!\land\! h'(X) {\neq} Y) {=} \Ex_\Pc \;W_\theta^2.$$
By applying the second order Markov's inequality, we have $R(f_\theta) {\leq} 4 R_2(\theta)$.
\cite{NEURIPS2020_38685413} show that in the worst case (\ie when the base classifiers are identical) the second order bound could be twice worse than the first order bound ($R_2(\theta){\approx}2R_1(\theta)$), while in the best case (\ie the $M$ base classifiers are perfectly decorrelated), the second order bound could be \emph{an order of magnitude} \mbox{tighter ($R_2(\theta){\approx} \frac 1M R_1(\theta)$)}.

\paragraph{C-bound.}\label{sec:CB}
Originally proposed by~\citet{breiman2001random} in the context of Random Forest, a bound known as the C-Bound in the PAC-Bayesian literature~\citep{lacasse2006pac,roy2011pac,germain2015risk,laviolette2017risk,viallard_Cbound} is derived by considering explicitly the joint error and disagreement between two base predictors.
Using Chebyshev-Cantelli's inequality, the C-Bound can be written in terms of the first and second moments of $W_\theta$:
$$R(f_\theta) \leq \frac{R_2(\theta) - \Ex_{h \sim \mathcal{C}(\theta)} \left(\Ex_\Pc \mathds{1}(h(X) \neq Y)\right)^2}{R_2(\theta) - R_1(\theta) + \frac{1}{4}}.$$
This bound is tighter than the second order one if the disagreement $\Ex_{h, h'} \Ex_\Pc (\mathds{1}(h(x) \neq h'(X)))$ is greater than $R_2(\theta)$~\citep{NEURIPS2020_38685413}. \\




From a practical point of view, one could minimize the generalization bounds of one of the above methods to learn a weight distribution over an ensemble of predictors. However, this could lead to sub-optimal MV classifiers.
To illustrate this behavior we plot in Figure~\ref{fig:moons-pred} the decision surfaces learnt by the minimization of a PAC-Bayesian bound on each of the aforementioned oracle bounds.
These plots provide evidence that, when the base classifiers are weak, state-of-the-art PAC-Bayesian methods do not necessarily build powerful ensembles (failing to improve upon a Naive Bayes approach~\citep{berend15a}).
\emph{First Order} concentrates on few base classifiers, as previously observed by~\citet{lorenzen2019pac} and \citet{NEURIPS2020_38685413}, 
while \emph{Second Order} and \emph{C-Bound} fail to leverage the diversity in the whole set of classifiers.
Indeed, in this setting the base classifiers are weak, but diverse enough so that there exists an optimal combination of them that  perfectly splits the two classes without error. However, the optimization of the PAC-Bayes guarantees over \emph{Second Order} and \emph{C-Bound} are shown to select a small subset of base classifiers which is not enough to achieve good performance.
On the contrary, \emph{Binomial} is able to fit the problem by drawing more than $2$ voters, but it provides generalization bounds that are loose even when the learned model exhibits good generalization capabilities, as in this case.

\begin{figure}[t]
    \centering
    \includegraphics[width=0.16\textwidth]{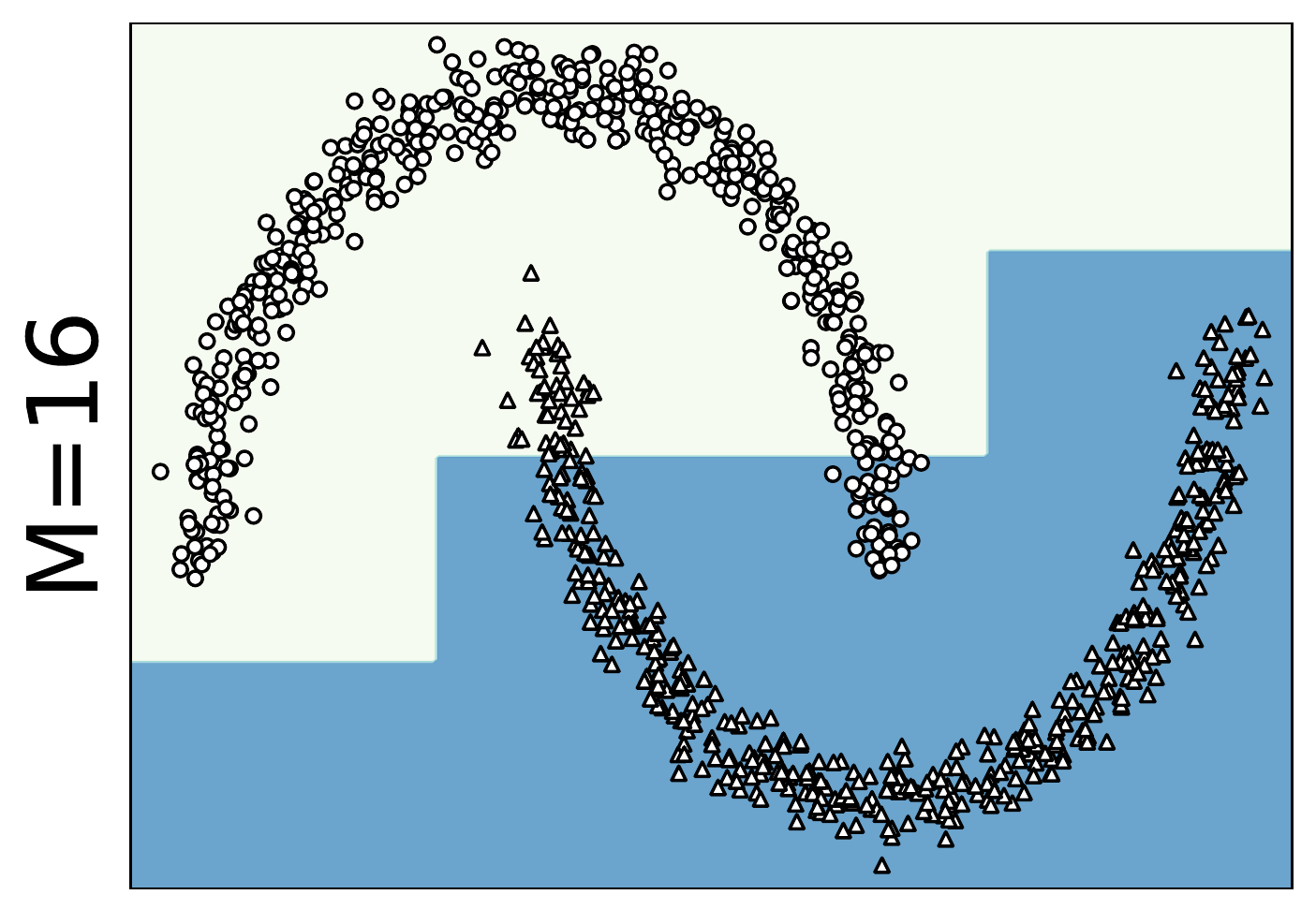}
    \includegraphics[width=0.15\textwidth]{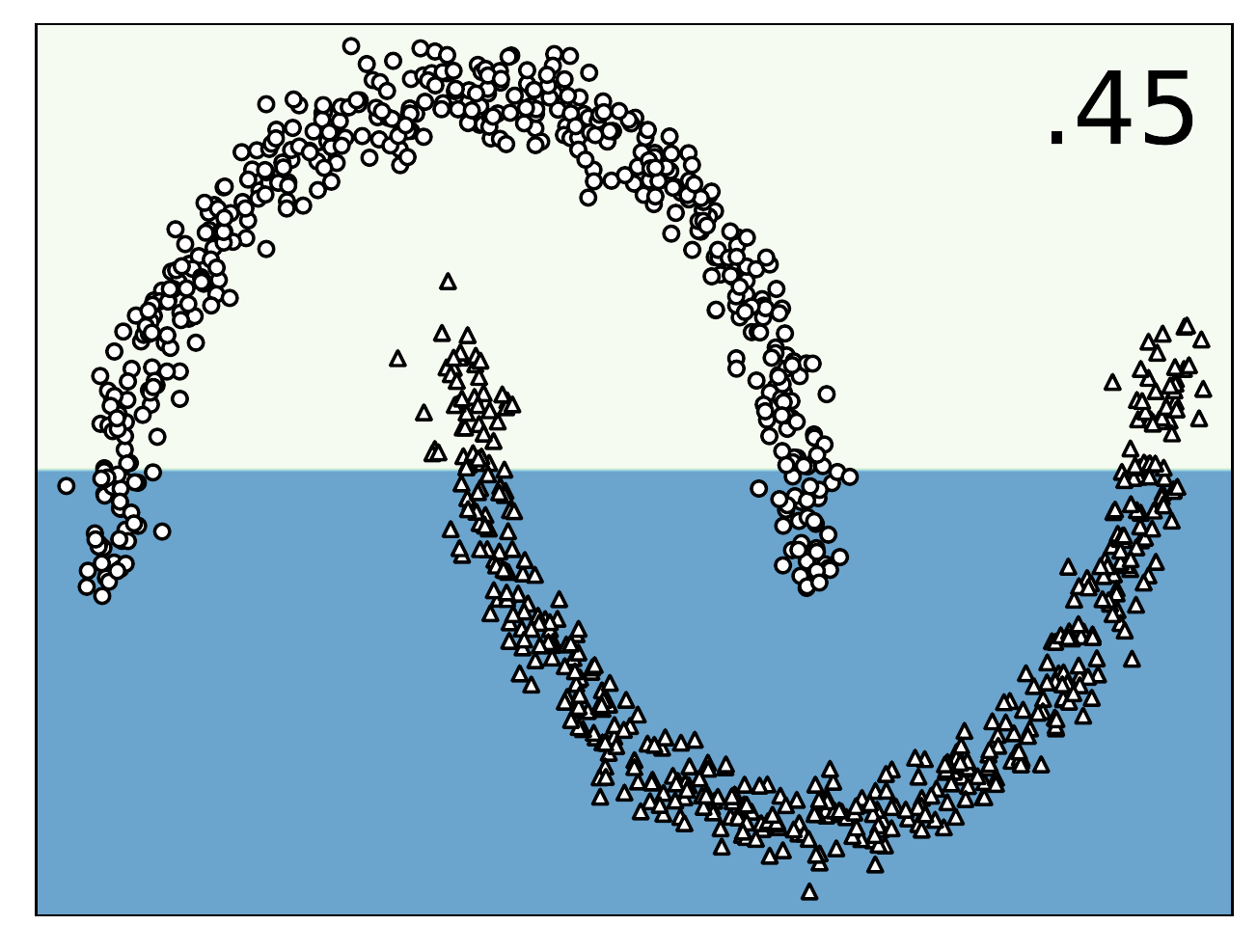}
    \includegraphics[width=0.15\textwidth]{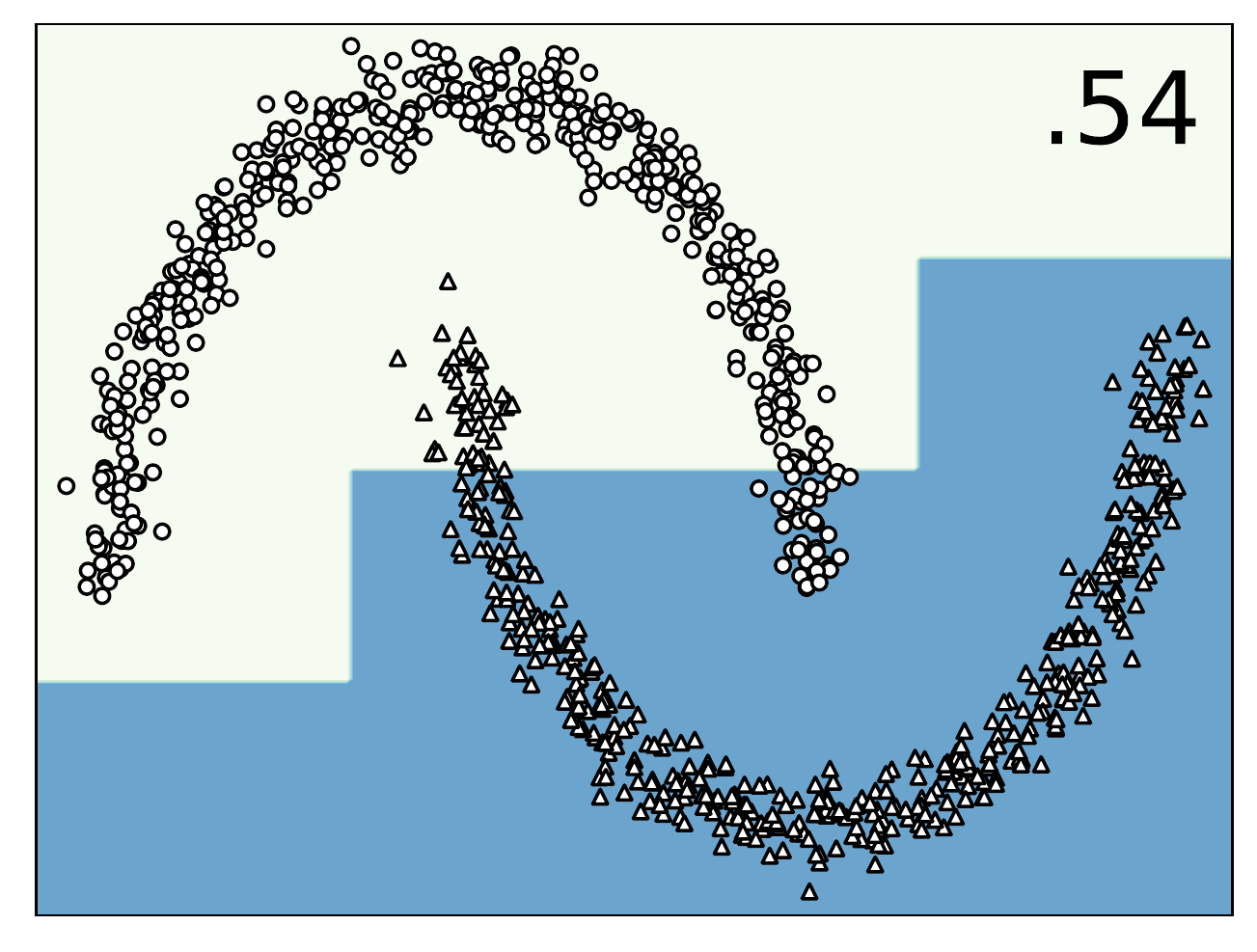}
    \includegraphics[width=0.15\textwidth]{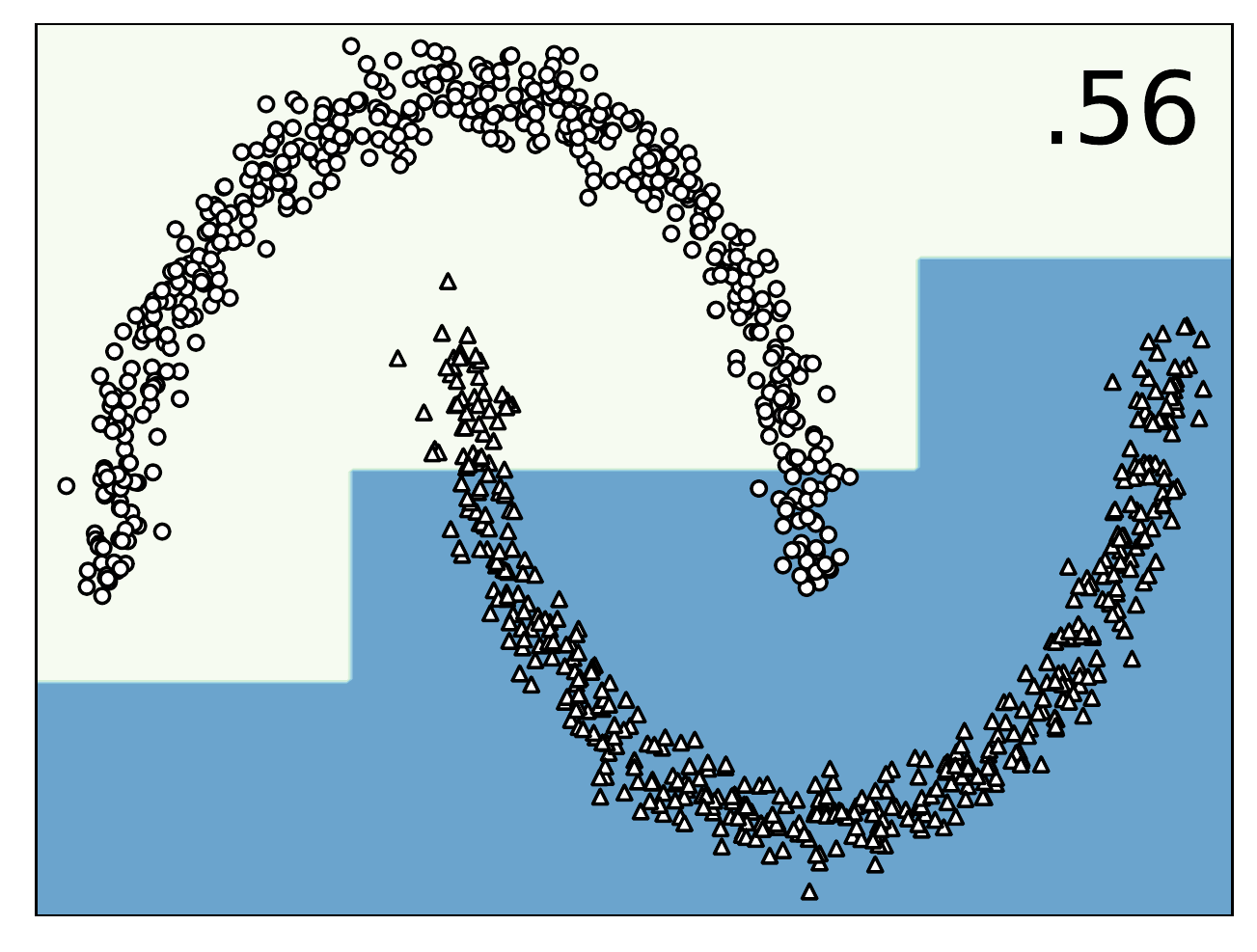}
    \includegraphics[width=0.15\textwidth]{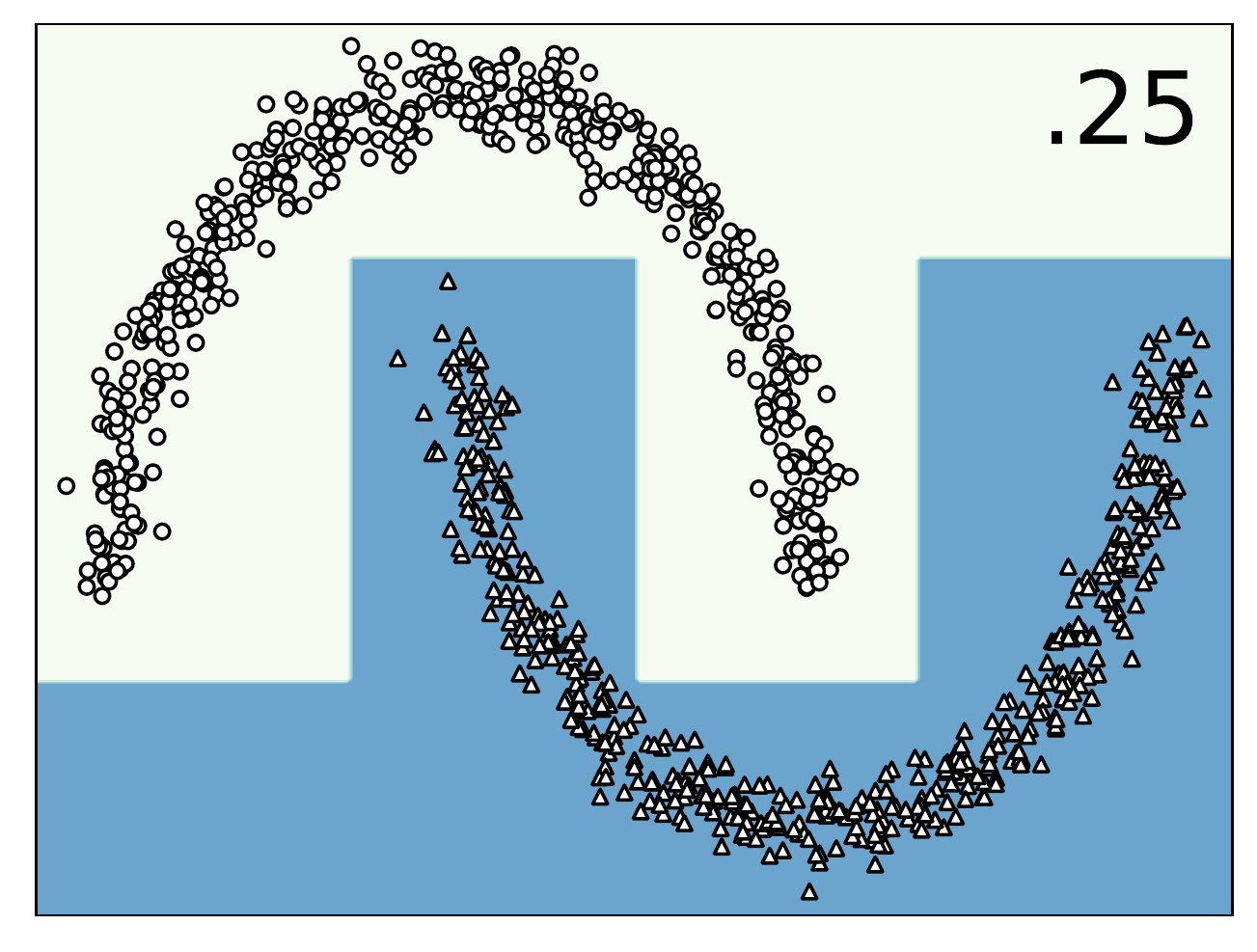}
    \includegraphics[width=0.15\textwidth]{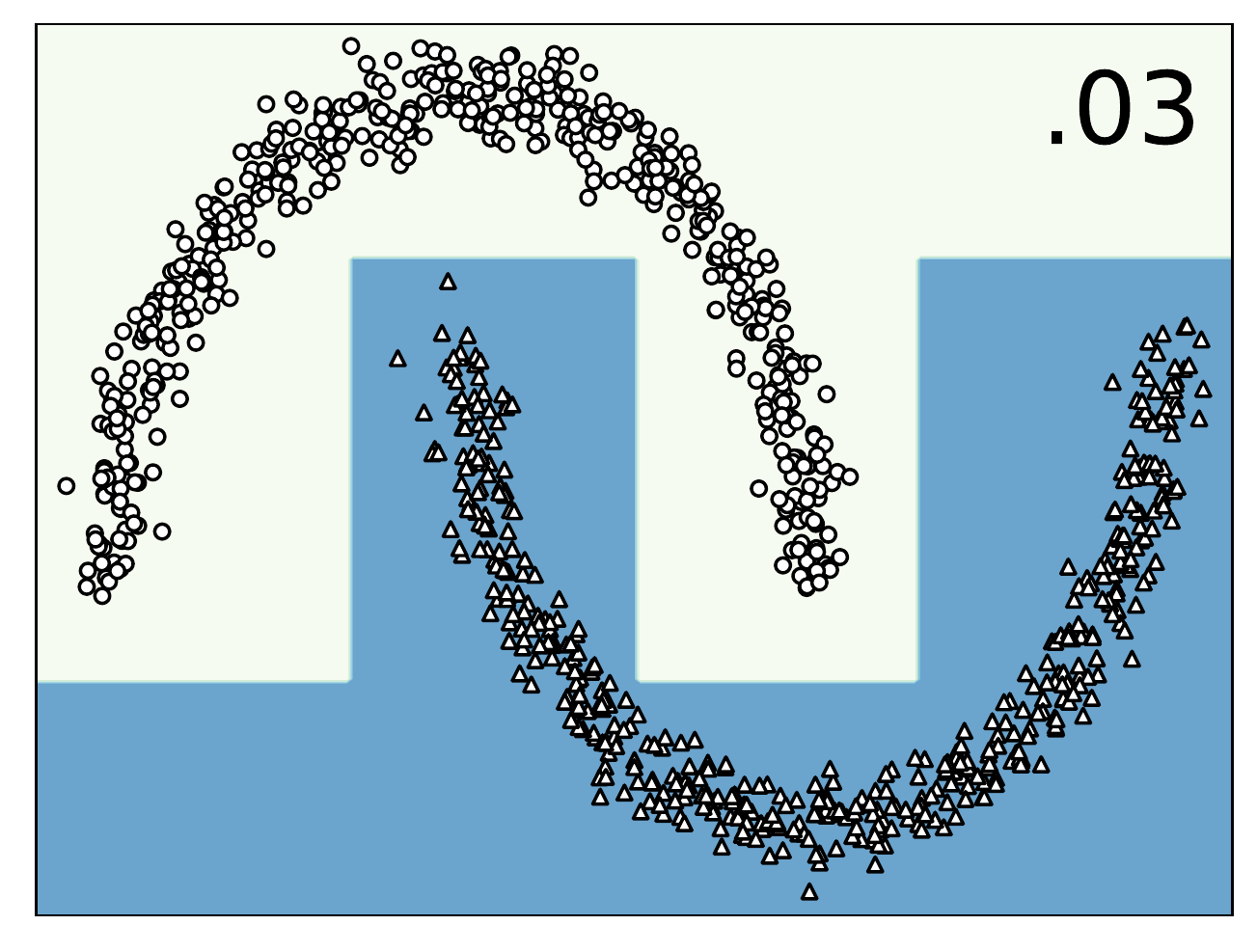}\\
    \subcaptionbox{Bayesian NB}{\includegraphics[width=0.16\textwidth]{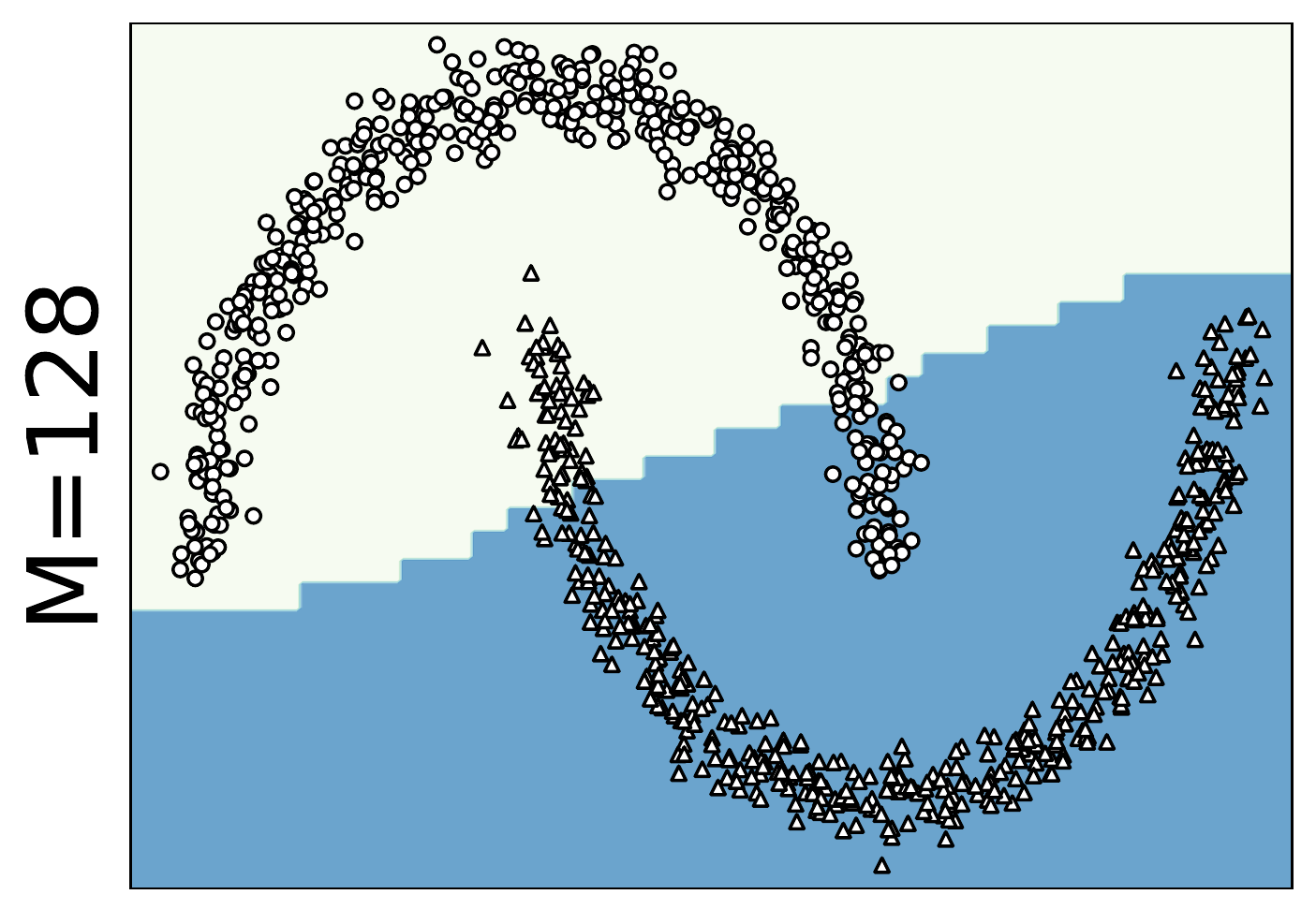}}
    \subcaptionbox{First Order}{\includegraphics[width=0.15\textwidth]{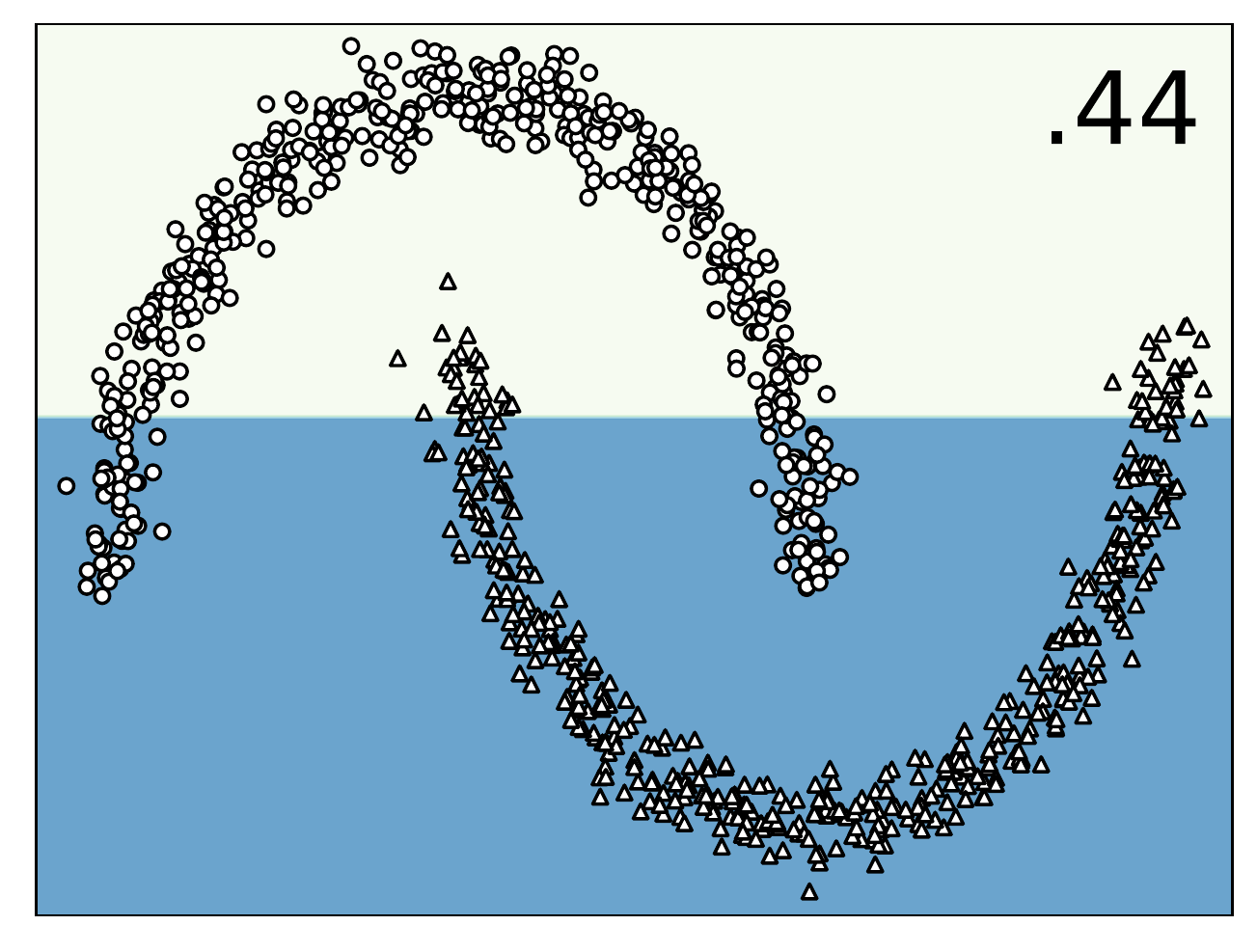}}
    \subcaptionbox{Second Order}{\includegraphics[width=0.15\textwidth]{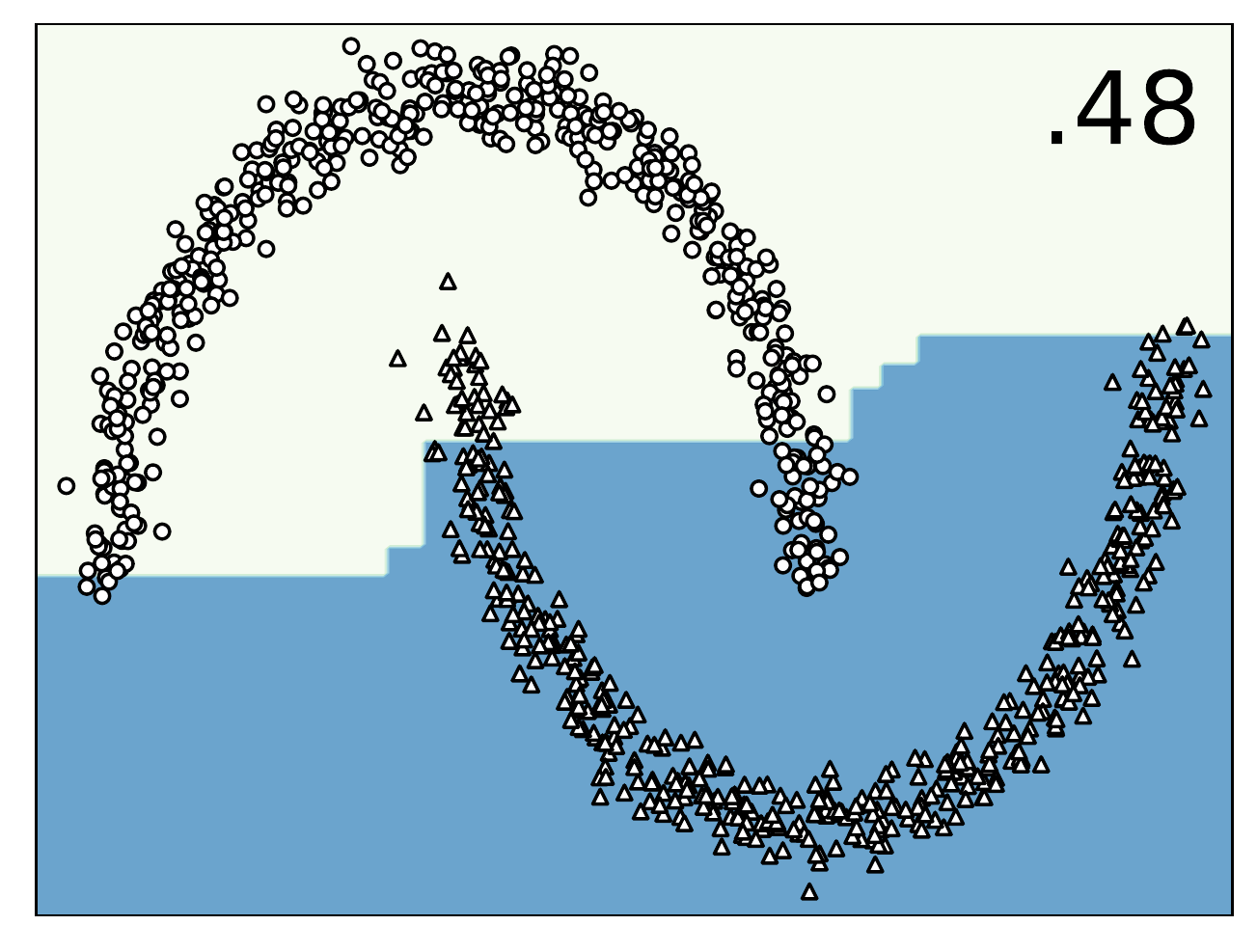}}
    \subcaptionbox{C-Bound}{\includegraphics[width=0.15\textwidth]{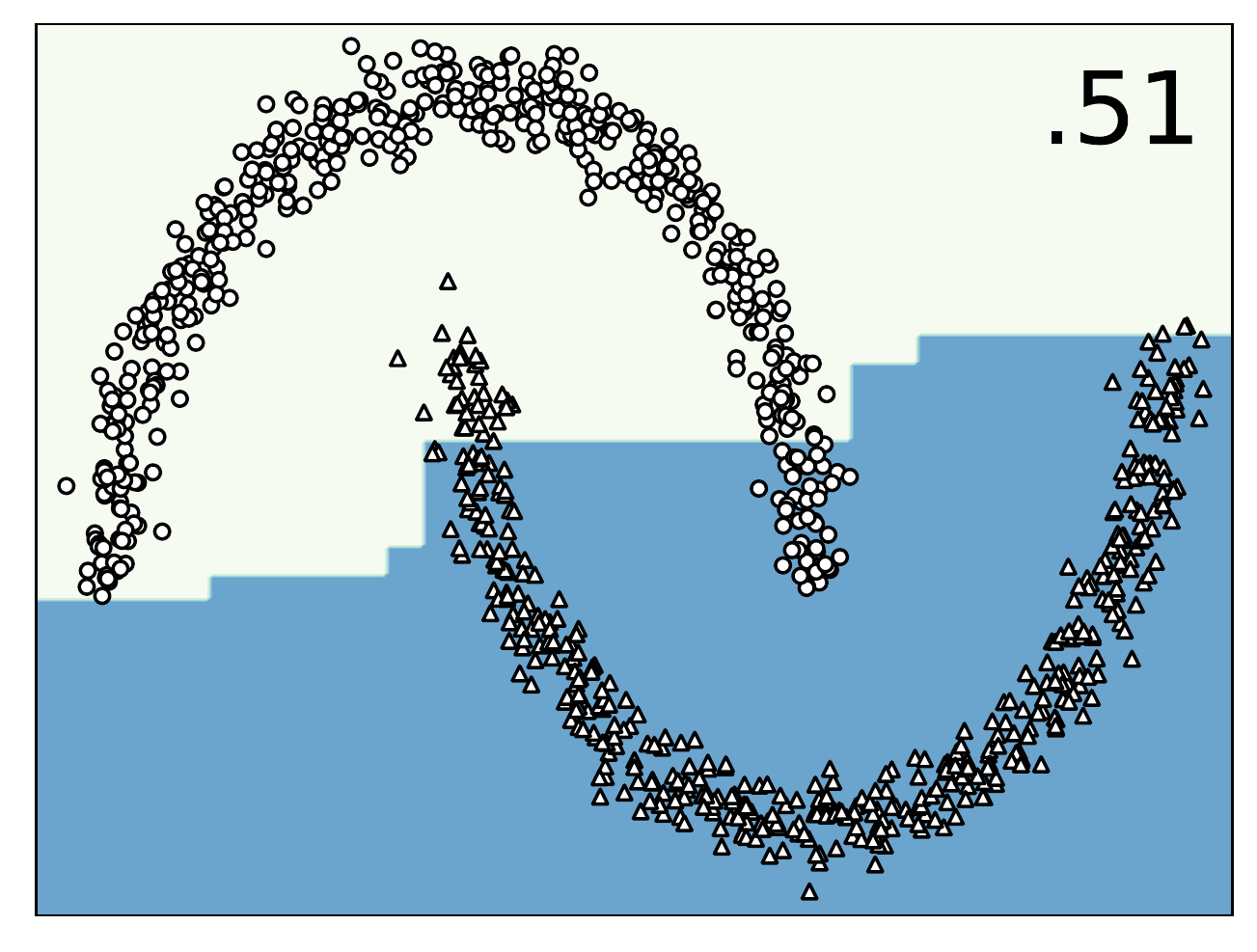}}
    \subcaptionbox{Binomial}{\includegraphics[width=0.15\textwidth]{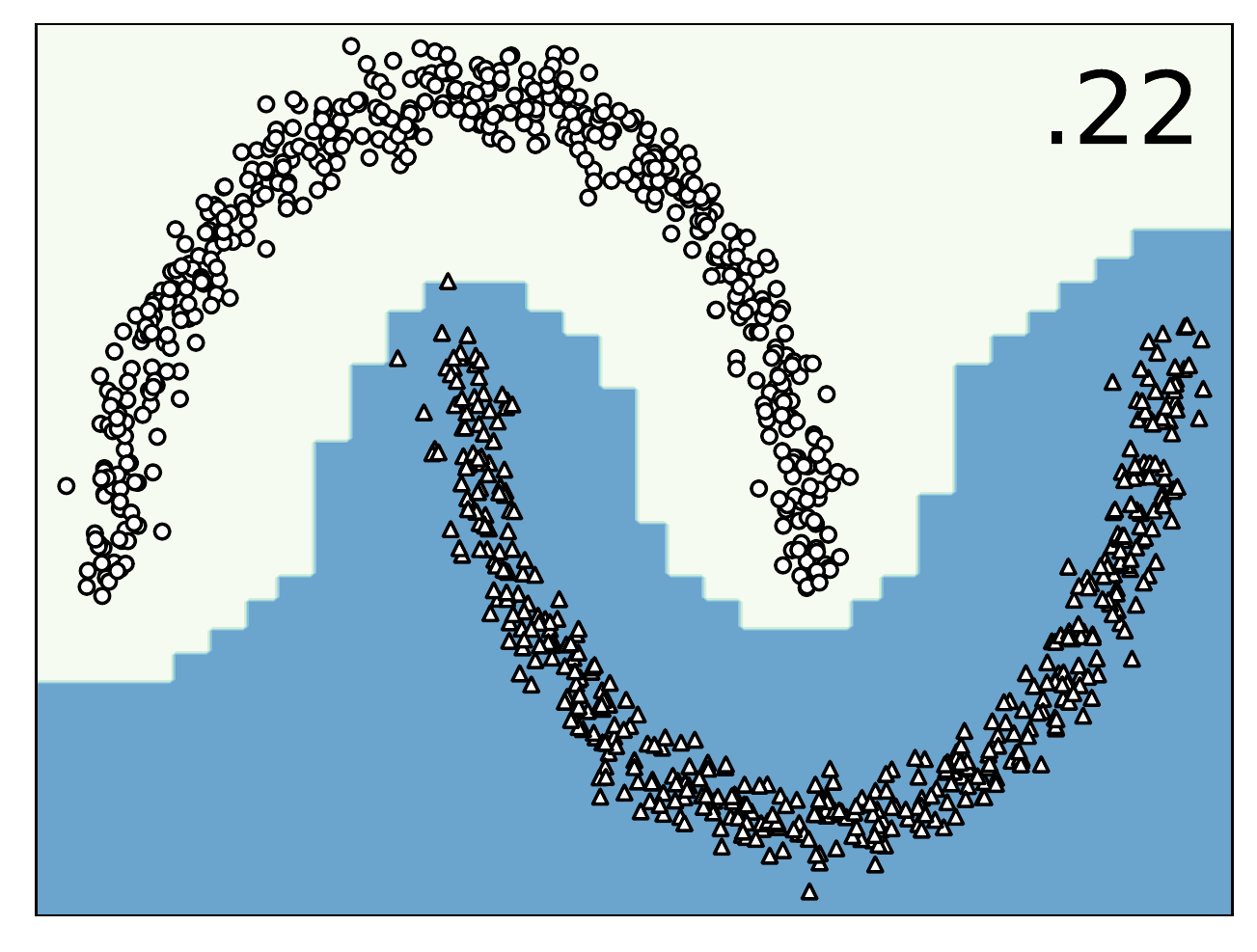}}
    \subcaptionbox{ours}{\includegraphics[width=0.15\textwidth]{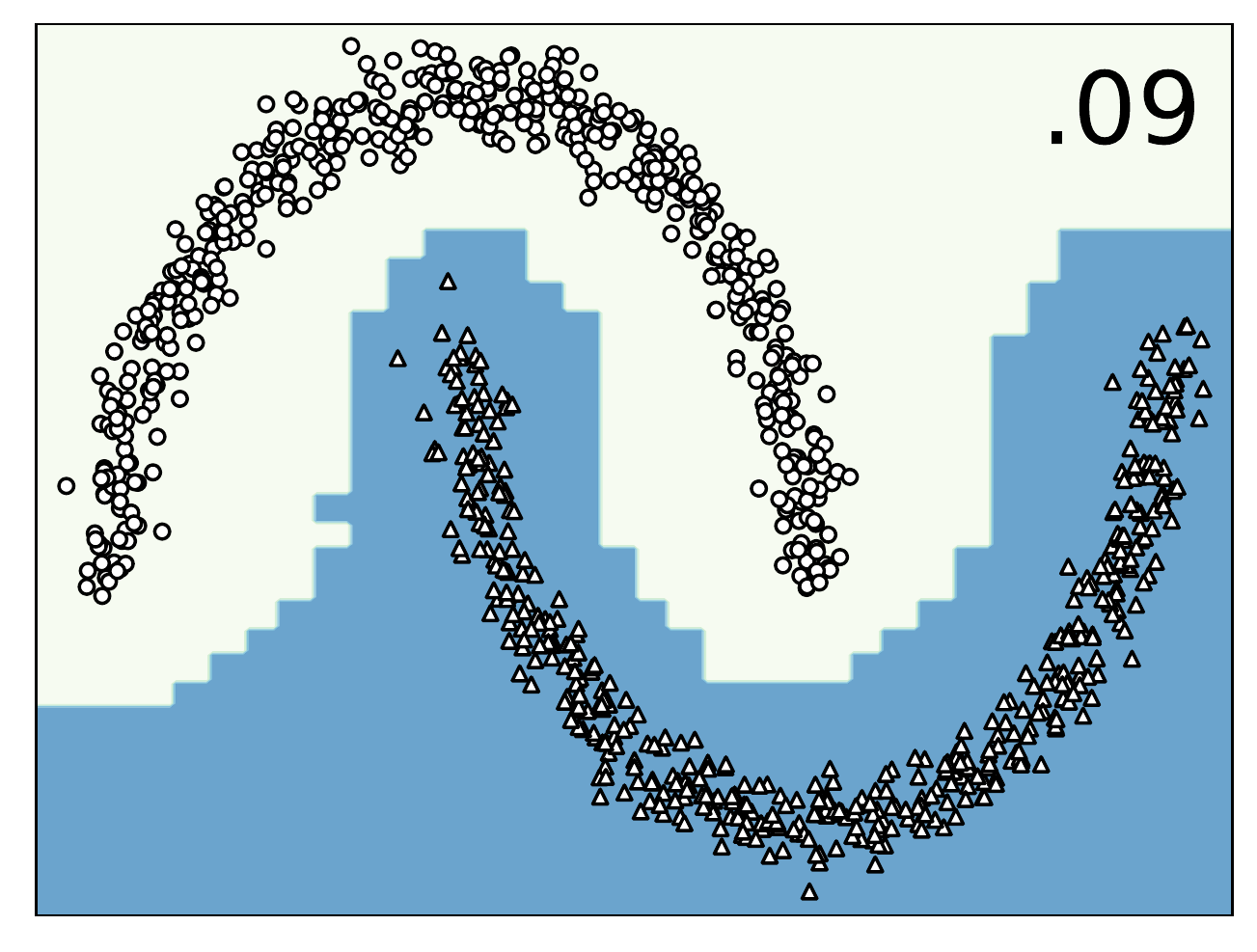}}
    \caption{Decision surface for \emph{Bayesian Naive Bayes}, the  PAC-Bayesian methods \emph{First Order}, \emph{Second Order}, \emph{C-Bound} and \emph{Binomial} (with $N=100$) and our method on the \emph{two-moons} dataset (where each half-circle is a class and inputs lie in $\mathcal{X} = [-2, 2]^2$) with $16$ (top) and $128$ (bottom) decision stumps as base classifiers (axis-aligned and evenly distributed over the input space). 
    Predicted labels are plotted with different colors, and training points are marked in black.
    When available, the value of the generalization bound is marked in the right-hand-side-top corner.
    }
    \label{fig:moons-pred}
\end{figure}

\section{Our approach: stochastic weighted majority vote}
\label{sec:stoc}

In our framework for deriving generalization bound for Majority Vote classifiers, we consider $f_\theta$ as a realization of a distribution of majority vote classifiers, with parameter $\theta \in \Theta \subseteq \Rb^M$ and probability measure $\rho$, as represented in Figure~\ref{fig:simplex}.
The main advantage of considering a stochastic majority vote is that it allows to derive and optimize PAC-Bayesian generalization bounds directly on its true risk and to fully leverage the set of base classifiers.
The true risk of the proposed stochastic weighted majority vote takes into account the whole distribution of MVs $\rho(\theta)$, as follows:
\begin{align}\label{eq:stoctruerisk}
    \int_\Theta R(f_\theta) \rho(d\theta) = \int_\Theta \Ex_\Pc \;\mathds{1}(W_\theta(X, Y) \geq 0.5) \rho(d\theta) = \Ex_\Pc \;\int_\Theta \;\mathds{1}(W_\theta(X, Y) \geq 0.5) \rho(d\theta).
\end{align}


Given its stochastic nature, in order to evaluate and/or minimize Equation~\eqref{eq:stoctruerisk} we can either \mbox{(i) compute} its closed form or (ii) approximate it (\eg\ through Monte Carlo methods).
In both cases, assumptions have to be made on the form of $\rho$.
As the components of $\theta$ are constrained to sum to one, $\theta$ lies in the ($M\text{-}1$)-simplex: $\Theta = \Delta^{(M\text{-}1)}$, hence natural choices for its probability measure are \eg\ the Dirichlet or the Logit Normal distributions.
In the following, we show that under Dirichlet assumptions, we can derive an analytical and differentiable form of the risk of a stochastic MV.

\subsection{Exact risk under Dirichlet assumptions}
First of all, we recall that the probability density function of the Dirichlet distribution is defined by: 
\begin{equation}
    \theta \sim \Dc(\alpha_1,\dots,\alpha_M),\quad\;
    \rho(\theta) = \frac{1}{B(\alpha)} \prod_{j=1}^M (\theta_j)^{\alpha_j-1},
\end{equation}
with $\alpha = [\alpha_j \in \Rb^{+}]_{j=1}^M$ the vector of concentration parameters and $B(\alpha)$ a normalization factor (see App~\ref{app:utils} for its definition).
Notice that by taking $\alpha$ as the vector of all ones, the distribution corresponds to a uniform distribution over the ($M\text{-}1$)-simplex $\Delta^{M-1}$.

Under these assumptions for the MV distribution, a closed form can be derived for the expected risk:

\begin{lemma}~\label{eq:betainc}
For a given $(x, y) \sim \Pc$, let $w = \{j | h_j(x) \neq y \}$ be the set of indices of the base classifiers that misclassify $(x, y)$ and $c = \{j | h_j(x) = y \}$ be the set of indices of the base classifiers that correctly classify $(x, y)$. 
The expected error (or $01$-loss) for $(x, y)$ of the stochastic majority vote under $\theta \sim \mathcal{D}(\alpha_1,\dots,\alpha_M)$ is equal to
\begin{equation}
   \int_\Theta \;\mathds{1}(W_\theta(x, y) \geq 0.5) \rho(d\theta) =  I_{0.5}\left(\sum_{j\in c} \alpha_j, \sum_{j\in w} \alpha_j\right) ,
\end{equation}
with $I_{0.5}(\cdot)$ the regularized incomplete beta function evaluated at $0.5$.
\end{lemma}

\begin{proof}
We rewrite $W_\theta$ as 
    $W_\theta(x, y) = \sumj \theta_j \mathds{1}(h_j(x) \neq y) = \sum_{j \in w} \theta_j$,
and use the aggregation property of Dirichlet distributions to show that $W_\theta$ follows a bivariate Dirichlet distribution (\aka Beta distribution):

\begin{lemma}\label{eq:aggrdir}
If $\theta \sim \mathcal{D}(\alpha_1, \dots, \alpha_M)$, then for any $j \in [M]$ and $j' \in [1, M] \setminus \{j\}$ the variable $\theta'$ formed by dropping $\theta_j$ and $\theta_{j'}$ and adding their sum also follows a Dirichlet distribution
$$ (\theta_1, \dots,\theta_M, \theta_j+\theta_{j'}) \sim \mathcal{D}(\alpha_1, \dots, \alpha_M, \alpha_j + \alpha_{j'}).$$
\end{lemma}

A proof of this property can be found in~\url{https://vannevar.ece.uw.edu/techsite/papers/documents/UWEETR-2010-0006.pdf}.
Hence $W_\theta$ follows a Beta distribution over the two sets of wrong and correct base classifiers:

\begin{align}
  \nonumber 
  \mathbb{P}[W_\theta(x, y) = \omega] =\ & \mathbb{P}\left[\sum_{j \in w} \theta_j = \omega\right] \mathbb{P}\left[\sum_{j \in c} \theta_j = 1 - \omega\right] 
  \\
    \Longrightarrow\ & W_\theta(x, y) \sim \Dc\left(\sum_{j \in w} \alpha_j, \sum_{j \in c} \alpha_j\right) \quad \text{by aggregation}.
\end{align}

Finally, notice that the expected error is related to the cumulative probability function of $W_\theta$, the incomplete beta function $I_p: \Rb^{+} \times \Rb^{+} \to [0, 1]$:

\begin{align}
    \int_\Theta \;\mathds{1}(W_\theta(x, y) \geq 0.5) \rho(d\theta) &= \int_{0.5}^1  \mathbb{P}[d W_\theta(x, y)] \; \\
    &= 1 - I_{0.5}\left(\sum_{j \in w} \alpha_j, \sum_{j \in c} \alpha_j\right) = I_{0.5}\left(\sum_{j \in c} \alpha_j, \sum_{j \in w} \alpha_j\right) \label{eq:betaincsym}.
\end{align}

Equation~\eqref{eq:betaincsym} follows by symmetry of the incomplete beta function: $I_p(a, b) = 1 - I_{1-p}(b, a)$.
\end{proof}

The expected risk $R(\rho)$ can be then expressed as follows:
\begin{equation}
    R(\rho) = \Ex_{\Pc} \int_\Theta \;\mathds{1}(W_\theta(x, y) \geq 0.5) \rho(d\theta) = \Ex_{\Pc} I_{0.5}\left(\sum_{j \in c} \alpha_j, \sum_{j \in w} \alpha_j\right).
\end{equation}
Importantly, this exact form of the risk is differentiable, hence can be directly optimized by gradient-based methods.

\subsection{Monte Carlo approximated risk}
We now propose a relaxed Monte Carlo (MC) optimization scheme for those distributions that do not admit an analytical form of the expected risk, unlike the Dirichlet one.
This second strategy is also suited to speed up training, in some cases, as the derivatives of the exact risk depend on the hyper-geometric function and can be slow to evaluate (see App.~\ref{app:derivatives}). 
With the approximated scheme, in order to update $\alpha$ by gradient descent we need to relax the true risk as the gradients of the $01$-loss are always null for discrete $W_\theta$.
In practice, we make use of a \emph{tempered} sigmoid loss $\sigma_c(x) = \frac{1}{1+\exp(-cx)}$ with slope parameter $c \in \Rb^{+}$.
De facto this corresponds to solving a relaxation of the problem and not its exact form~\citep{Nesterov05}.
At each iteration of the MC optimization algorithm we perform:

\begin{enumerate}
    \item Draw a sample $\{\theta_t \sim \rho(\alpha)\}_{t=1}^T$ using the implicit reparameterization trick~\citep{figurnov2018implicit, jankowiak2018pathwise};\\
    \item Compute the relaxed empirical risk $\sum_{t=1}^T \hat{R}_{\sigma_c}(\theta) = \sum_{t=1}^T \sumi \sigma_c(W_{\theta_t}(x_i, y_i)-0.5)$;\\
    \item Update $\alpha$ by gradient descent.
\end{enumerate}

Notice that when considering Dirichlet distributions for the posterior and the prior, at inference time the empirical PAC-Bayesian bounds can still be evaluated using the exact form of Lemma~\ref{eq:betainc}.

A drawback of the approximated scheme is that it has a complexity linear in the number of MC draws $T$, but also linear in the number of predictors $M$, as sampling over the simplex requires sampling from $O(M)$ distributions, one per base classifier.
As an example, sampling from a Dirichlet over the ($M\text{-}1$)-simplex is usually implemented as sampling from $M$ Gamma distributions and normalizing the samples so that they lie on the simplex. 
In contrast, the exact formulation's complexity is constant in $M$ as it depends only on the sets of wrong and of correct predictors, hence on a constant number of variables ($2$) no matter the number of predictors $M$.
In Section~\ref{sec:exact-vs-MC} we empirically study the trade-off between training time and accuracy, showing in which regimes it is more convenient to optimize the relaxed MC risk than optimizing the exact one, and viceversa.

\section{PAC-Bayesian generalization guarantees}\label{sec:bounds}
We now derive PAC-Bayesian generalization upper bounds for the proposed stochastic MV.
In our context, upper bounds can be derived for studying the gap between true and empirical risk considering a prior distribution $\pi$ over the hypothesis space $\Theta$.
In this paper, we make use of one of the tightest classical PAC-Bayesian bound~\citep{seeger2002pac,maurer2004note}:
\begin{theorem}[Seeger's bound]\label{eq:seeger}
For any $\pi$ over $\Theta$ and $\delta \in (0,1)$ with probability at least $1{-}\delta$ over samples \mbox{$S = \{(x_i, y_i) {\sim} \Pc\}_{i=1}^n$} of size $n$ we have simultaneously for any posterior $\rho$ over $\Theta$:
\begin{equation}
    \int_{\Theta} R(f_\theta) \rho(d\theta) \leq \mathrm{kl}^{-1}\left(\int_\Theta \hat{R}(f_\theta) \rho(d\theta)\ ,\  \frac{\mathrm{KL}(\rho, \pi) + \ln\left(\frac{2\sqrt{n}}{\delta}\right)}{n}\right),
\end{equation}
with $\hat{R}(f_\theta) = \frac{1}{n} \sumi \mathds{1}(W_\theta(x_i, y_i) \geq 0.5)$ the empirical risk on sample $S$, $\mathrm{KL}(\rho, \pi) = \int_{\Theta} \rho(\theta) \log \frac{\rho(\theta)}{\pi(\theta)} \; d\theta$ the KL divergence and $\mathrm{kl}^{-1}(q, \epsilon)$ the inverse of the binary KL divergence defined as
$ \mathrm{kl}^{-1}(q, \epsilon) = \max \{p \in [0, 1] \;|\; \mathrm{kl}(q, p) \leq \epsilon\}.$
\end{theorem}
A proof of Theorem~\ref{eq:seeger} can be found in \citet{seeger2002pac}.
The $\mathrm{kl}^{-1}$ function can be evaluated via the bisection method and optimized by gradient descent, as proposed in~\citet{reeb2018learning}.
Note that our contributions do not restrict the choice of generalization bound.

Importantly, Theorem~\ref{eq:seeger} is valid when the prior $\pi$ is independent from the data.
Thus it cannot be evaluated with base classifiers learned from the training sample.
However, it is known that considering a data-dependent prior can lead to tighter PAC-Bayes bounds~\citep{dziugaite2021role}.
Following recent works on PAC-Bayesian bounds with data-dependent priors~\citep{thiemann2017strongly,mhammedi2019pac}, we derive a cross-bounding certificate that allows us to learn and evaluate the set of base classifiers without held-out data.
More precisely, we split the training data $S$ into two subsets ($S_{\le m} {=} \{(x_i, y_i) \in S\}_{i=1}^m$ and $S_{> m} {=} \{(x_i, y_i) \in S\}_{i=m{+}1}^n$) and we learn a set of base classifiers on each data split independently (determining the hypothesis spaces \mbox{$\Theta_{\le m}$ and $\Theta_{> m}$}).
We refer to the prior distribution over $\Theta_{\le m}$ as $\pi_{\le m}$ and to the prior distribution over $\Theta_{> m}$ as $\pi_{> m}$.
In the same way, we can then define a posterior distribution per hypothesis space: $\rho_{\le m}$ and $\rho_{> m}$.
The following theorem shows that we can bound the expected risk of any convex combination of the two posteriors, as long as their empirical risks are evaluated on the data split that was not used for learning their respective priors.

\begin{theorem}[Seeger's bound with informed priors]\label{eq:informed-prior}
Let $\pi_{\le m}$ and $\rho_{\le m}$ be the prior and posterior distributions on $\Theta_{\le m}$, and $\pi_{>m}$ and $\rho_{>m}$ the prior and posterior distributions on $\Theta_{>m}$.
For any $p {\in} (0, 1)$ and $\delta \in (0,1)$ with probability at least $1{-}\delta$ over samples \mbox{$S {=} \{(x_i, y_i) {\sim} \Pc\}_{i=1}^n$} 
we have 
\begin{multline*}
    \kl\!\left(p \hat{R}(\rho_{>m})+ (1-p) \hat{R}(\rho_{\le m}) \middle\| p R_{\le m}(\rho_{>m}) + (1-p) _{> m}(\rho_{\le m}) \right) \\ \le  \frac{p\; KL(\rho_{>m}, \pi_{>m})}{m}+\frac{(1-p)\;KL(\rho_{\le m}, \pi_{\le m})}{n-m} + \frac{\ln\tfrac{4\sqrt{m(n{-}m)}}{\delta}}{n},
\end{multline*}
with $R(\rho_{>m}) = \int_{\Theta_{>m}}\!\! R(f_\theta) \rho(d\theta)$, and  $R(\rho_{\le m}) = \int_{\Theta_{\le m}}\!\! R(f_\theta) \rho(d\theta)$,\\
and  $\hat{R}_{\le m}(\rho_{>m}) {=} \int_{\Theta_{>m}}\!\! \hat{R}(f_\theta) \rho(d\theta)$, and  $\hat{R}_{\le m}(\rho_{>m}) {=} \int_{\Theta_{\le m}}\!\! \hat{R}(f_\theta) \rho(d\theta)$.
\end{theorem}
The result follows by $\kl$'s convexity. The complete proof is reported in App.~\ref{ap:cross-prior}. 
In practice, following~\citet{mhammedi2019pac} we set $m=\frac{n}{2}$ and $p=\frac{m}{n}$, and we learn the base classifiers by empirical risk minimization.

\paragraph{Comparison with existing bounds.}
Until now, we considered bounds for the expected risk over the space of MVs but not for a single realization $f_\theta$, unlike state-of-the-art methods.
Under Dirichlet assumptions, we can bound the risk of the expected MV $f_{\hat{\theta}}$ by twice the expected risk (see App~\ref{app:deterministic}):
$R(f_{\hat{\theta}}) = R(\Ex_{\rho(\theta)} f_\theta) \leq 2 \Ex_{\rho(\theta)} R(f_\theta)$. 
The obtained oracle bound would comprise an irreducible factor, but its KL term would not degrade, unlike state-of-the-art bounds that account for voter correlation~\citep{lacasse2006pac,lacasse2010learning,NEURIPS2020_38685413}. 
Indeed, the empirical bound does not introduce additional factors on the KL term, such as \emph{Second Order} which has a $2\; KL()$ term and \emph{Binomial} which has a $N\; KL()$ term.
The empirical bound on the deterministic MV could be also refined leveraging 
works on the disintegration 
of PAC-Bayesian guarantees~\citep{BlanchardFleuret2007,Catoni2007,rivasplata2020pac,viallard2021general}.

A downside of our method is that the complexity of the KL term grows with the number of base classifiers $M$, unlike the KL on categoricals that tends to $0$ in its limit.
This results in making the generalization bounds increasingly looser and conservative with growing $M$, even for low empirical risks.
From a generalization perspective, our guarantees are hence able to reflect the complexity of the model, expressed as the size of the hypothesis space $M$.
However, our risk certificates do not account for redundancy in the voter set.
For instance, they are not able to distinguish scenarios where base classifiers are highly correlated (hence less complex hypothesis space) from scenarios where base classifiers are independent (more complex).
An expedient for ensuring that generalization bounds are tight consists in learning the hypothesis space:
the number of base predictors can then be limited without degrading the performance of the majority vote.

\begin{figure}[t]
    \centering
    \includegraphics[width=\textwidth]{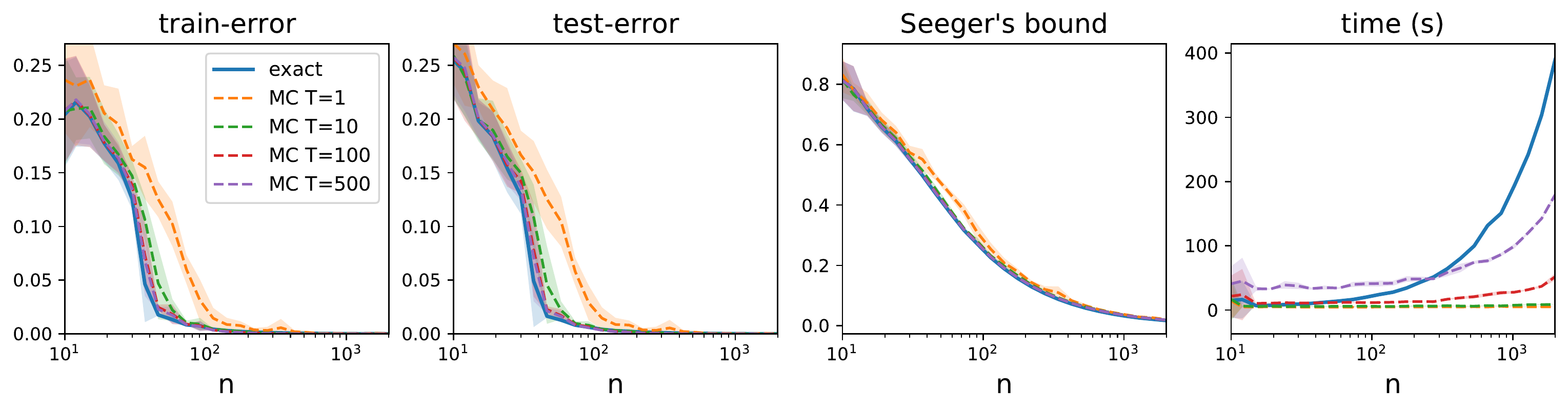}\\
    \includegraphics[width=\textwidth]{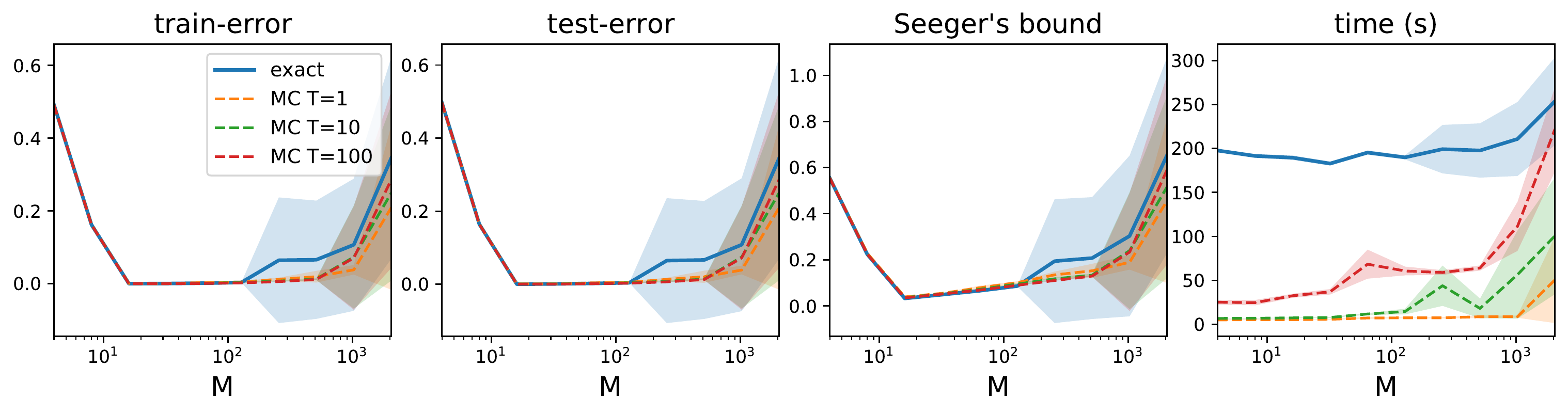}
    \caption{Average performance for $10$ trials of \emph{exact} and \emph{MC} variants of our method as a function of the number of training points $n$ (top, $M$ fixed to $16$) or of decision stumps $M$ (bottom, $n=1000$), both represented in logarithmic scale, for different number of MC draws $T$.}
    \label{fig:moons-MC}
\end{figure}

\section{Experiments}\label{sec:expe}
In this section, we empirically evaluate \coolname, and we compare its generalization bounds and test errors to those obtained with PAC-Bayesian methods learning majority votes.
We show that our method allows to derive generalization bounds that are consistently tight (\ie close to the test errors) and non-vacuous (\ie smaller than $1$) both when studying ensembles of weak predictors and when studying ensembles of strong ones.

In the following, we consider Dirichlet distributions for the prior and the posterior of our method, and refer to the model obtained by optimizing the exact risk as \emph{exact} and the one obtained by optimizing the approximated one as \emph{MC}.
We consider as baselines the PAC-Bayesian 
methods described in Section~\ref{sec:related}: 
We refer to the methods optimizing the First Order, Second Order and Binomial empirical bound as \emph{FO}, \emph{SO} and \emph{Bin} respectively.
We do not compare with the C-Bound, as it is hard to optimize on large-scale datasets and existing algorithms are suited only for binary classification.
All generalization bounds are evaluated with a probability $1 {-} \delta {=} 0.95$ and all prior distributions are set to the uniform (we provide a study for different priors in App.~\ref{app:prior}).
The posterior parameters ($\alpha$ for our method, $\theta$ for the others) are initialized uniformly in $[0.01, 2]$ (and normalized to sum to $1$ for \emph{SO}, \emph{FO} and \emph{Bin}).
Finally, for \emph{MC} the sigmoid's slope parameter $c$ is set to $100$ and for \emph{Bin} the number of voters drawn at each iteration is set to $N{=}100$.
Code, available at \url{https://github.com/vzantedeschi/StocMV}, was implemented in pytorch~\citep{PaszkeGMLBCKLGA19} and all experiments were run on a virtual machine with $8$ vCPUs and $128Gb$ of RAM.

\subsection{Comparison of exact and MC variants}\label{sec:exact-vs-MC}
For this set of experiments, we optimize Seeger's Bound (Equation~\eqref{eq:seeger}) by (batch) Gradient Descent, for $1,000$ iterations and with learning rate equal to $0.1$.
We study the performance of our method on the binary classification \textbf{two-moons} dataset, with $2$ features, $2$ classes and $\mathcal{N}(0, 0.05)$ Gaussian noise, for which we draw $n$ points for training, and $1,000$ points for testing.
Figure~\ref{fig:moons-MC} reports a comparison of \emph{exact} and \emph{MC} variants in terms of error, generalization bound and training time (in seconds).
Increasing the number of MC draws $T$ unsurprisingly allows to recover \emph{exact}'s performance, and at lower computational cost for reasonable values of $M$ and $T$.
In general, as \emph{MC} is easily parallelizable, its training time has better dependence on $n$ than \emph{exact}'s one, however it increases with $M$ at a worse rate.
When the training sample is large enough (here for $n > 10^2$), \emph{MC} achieves \emph{exact}'s errors and bounds even for $T=1$.
We also observe that the error rates and bounds gradually degrade for higher values of $M$ for both methods. 
This is due to the KL term increasing with $M$, as highlighted in Section~\ref{sec:bounds}, becoming a too strong regularization during training and making the bound looser. 

\subsection{Experiments on real benchmarks}
We now compare the considered methods on real datasets and on two different scenarios, depending on the type of PAC-Bayesian bounds that are evaluated: When making use of \textit{data-independent priors}, we chose as voters axis-aligned decision stumps, with thresholds evenly spread over the input space ($10$ per feature);
When making use of \textit{data-dependent priors}, we build Random Forests~\citep{breiman2001random} as set of voters, each with $M{=}100$ trees learned bagging $\frac{n}{2}$ points and sampling $\sqrt{d}$ random features to ensure voter diversity, optimizing Gini impurity score and, unless stated otherwise, without bounding their maximal depth. 

\begin{figure}
    \centering
    \subcaptionbox{Binary -- data independent prior\label{fig:independent}}{\includegraphics[width=\textwidth]{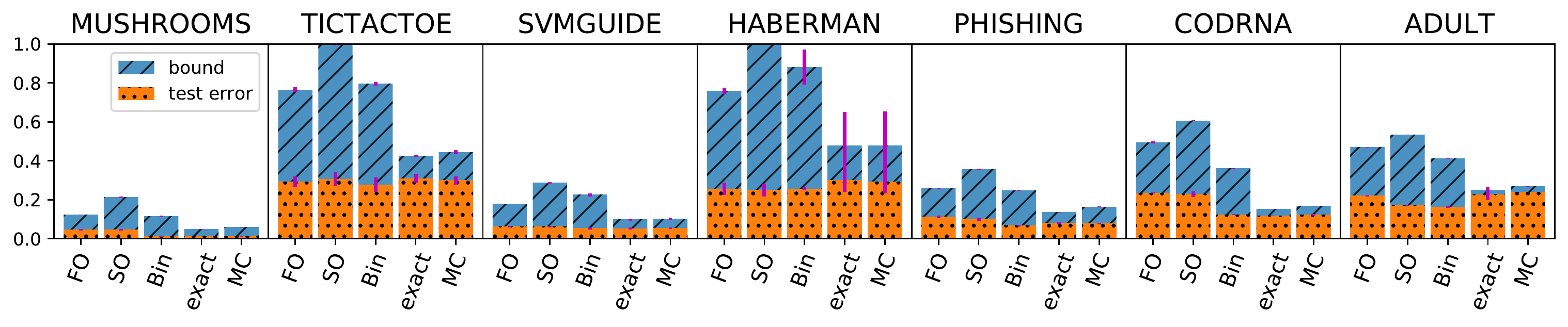}}
    \subcaptionbox{Multi-class -- data dependent prior\label{fig:dependent}}{\includegraphics[width=7cm]{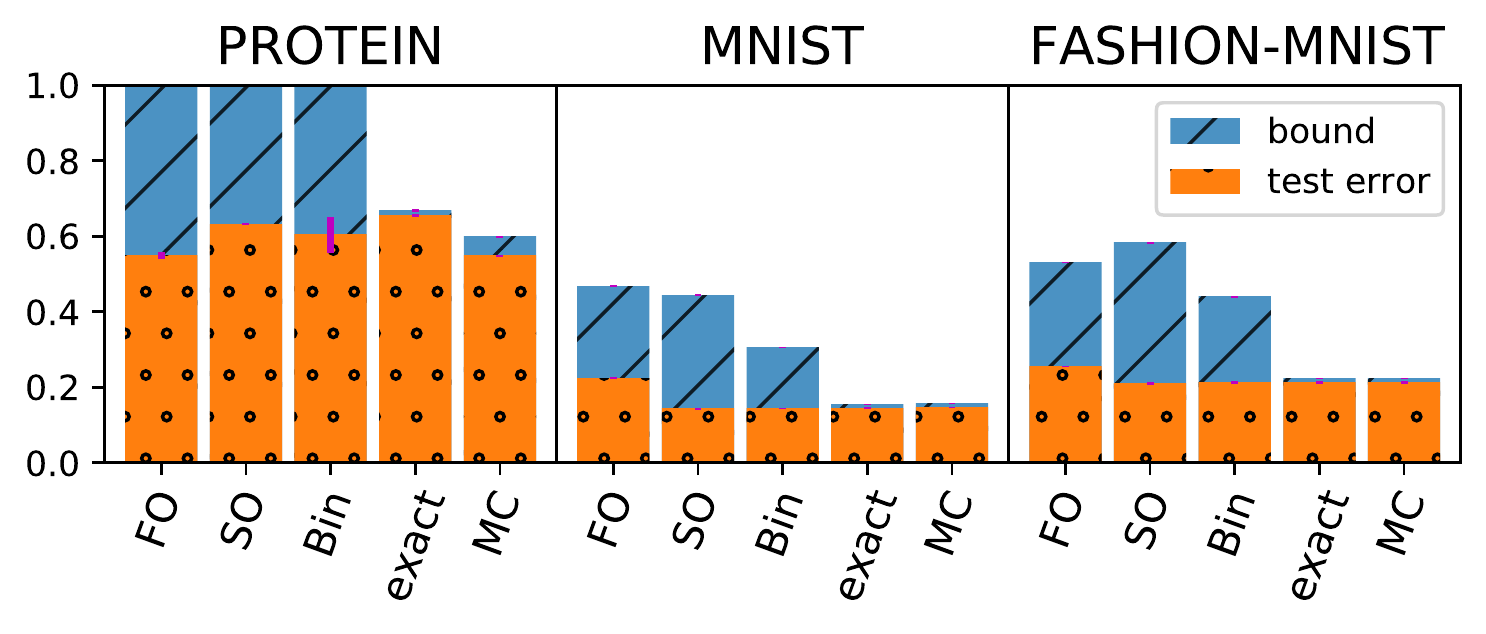}\hfill \includegraphics[width=5.4cm]{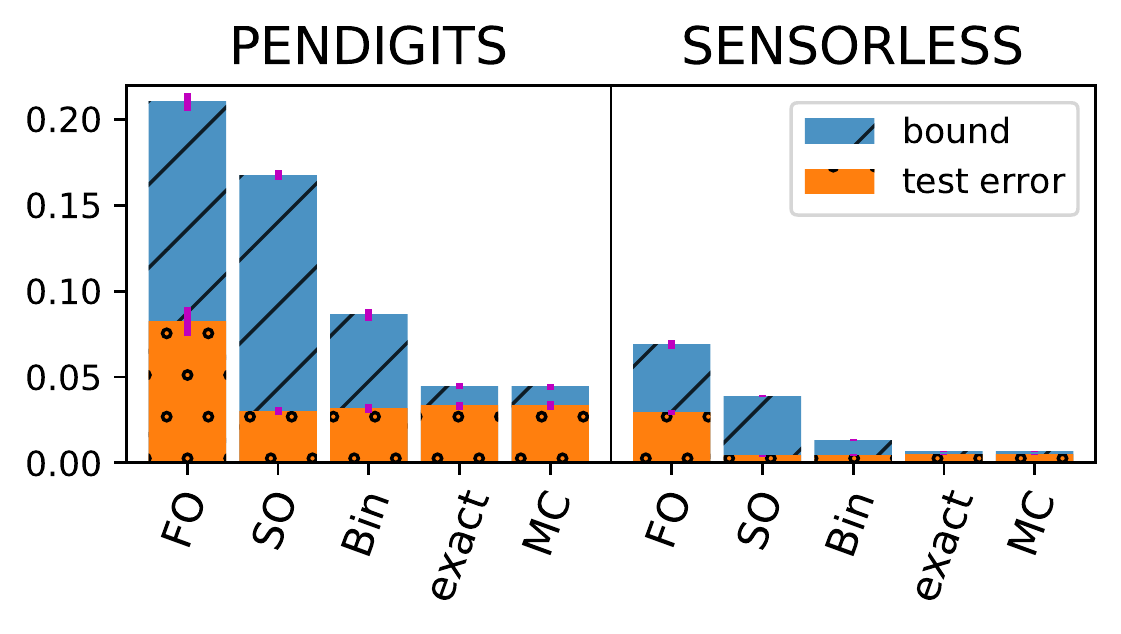}}
    \caption{Comparison in terms of test error rates and PAC-Bayesian bound values.
    We report the means (bars) and standard deviations (vertical, magenta lines) over $10$ different runs.}\label{fig:real}
\end{figure}

We consider several classification datasets from UCI~\citep{Dua2019}, LIBSVM\footnote{\url{https://www.csie.ntu.edu.tw/~cjlin/libsvm/}} and Zalando~\citep{xiao2017}, of different number of features and of instances.
Their descriptions and details on any pre-processing are provided in App.~\ref{app:datasets}.
We train the models by Stochastic Gradient Descent (SGD) using Adam~\citep{KingmaB14} with $(0.9, 0.999)$ running average coefficients, batch size equal to $1024$ and learning rate equal to $0.1$ with a scheduler reducing this parameter of a factor of $10$ with $2$ epochs patience.
We fix the maximal number of epochs to $100$ and patience equal to $25$ for early stopping, and for \emph{MC} we fix $T=10$ to increase randomness.

We report the test errors and generalization bounds in Figure~\ref{fig:real} (additional results are reported in the appendix, in Tables~\ref{tab:binary-real} and~\ref{tab:multic-real} and Figure~\ref{fig:real-all}):
We compare the different methods on binary datasets and with data-independent priors in Figure~\ref{fig:independent}, and on multi-class datasets and with data-dependent priors in Figure~\ref{fig:dependent}.
First we notice that the bounds obtained by our method are consistently non vacuous and tighter than those obtained by the baselines on all datasets.
Regarding the error rates, our method's performance is generally aligned with the baselines, while it achieved error rates significantly lower than \emph{FO} and \emph{SO} on the perfectly separable \emph{two-moons} dataset.
Sensitivity to noise could explain why our method does not outperform the baselines on the studied real problems, as these usually present label and input noise. 
Indeed our learning algorithm optimizes the $01$-loss, which does not distinguish points with margins close or far from $0.5$ because of its discontinuity in $W_\theta = 0.5$.
Preliminary results reported in App~\ref{app:noise} seem to confirm this supposition.

Finally, to gain a better understanding of the relation between base classifier strength and performance of the models obtained with the different methods, we further study their performance and generalization guarantees with varying voter strength.
As hypothesis set, we learn Random Forests with $200$ decision trees in total, as before, but for this experiment we bound their maximal depth between $1$ and $10$.
Constraining the tree depth allows to indirectly control how well the voters fit the training set (as shown in Appendix, Figure~\ref{fig:strength-app}). 
Following~\citet{lorenzen2019pac}, we assess the voter strength by computing the expected accuracy of a random voter.
All methods' results improve overall when voters get stronger, even though \emph{FO} at a slower pace.
Notice that on the considered datasets \emph{SO} is the most sensitive method, particularly suffering from weak base predictors.
Our method generally provides test errors comparable with the best baselines and consistently tighter bounds.

\begin{figure}[t!]
    \centering
    \includegraphics[width=0.49\textwidth]{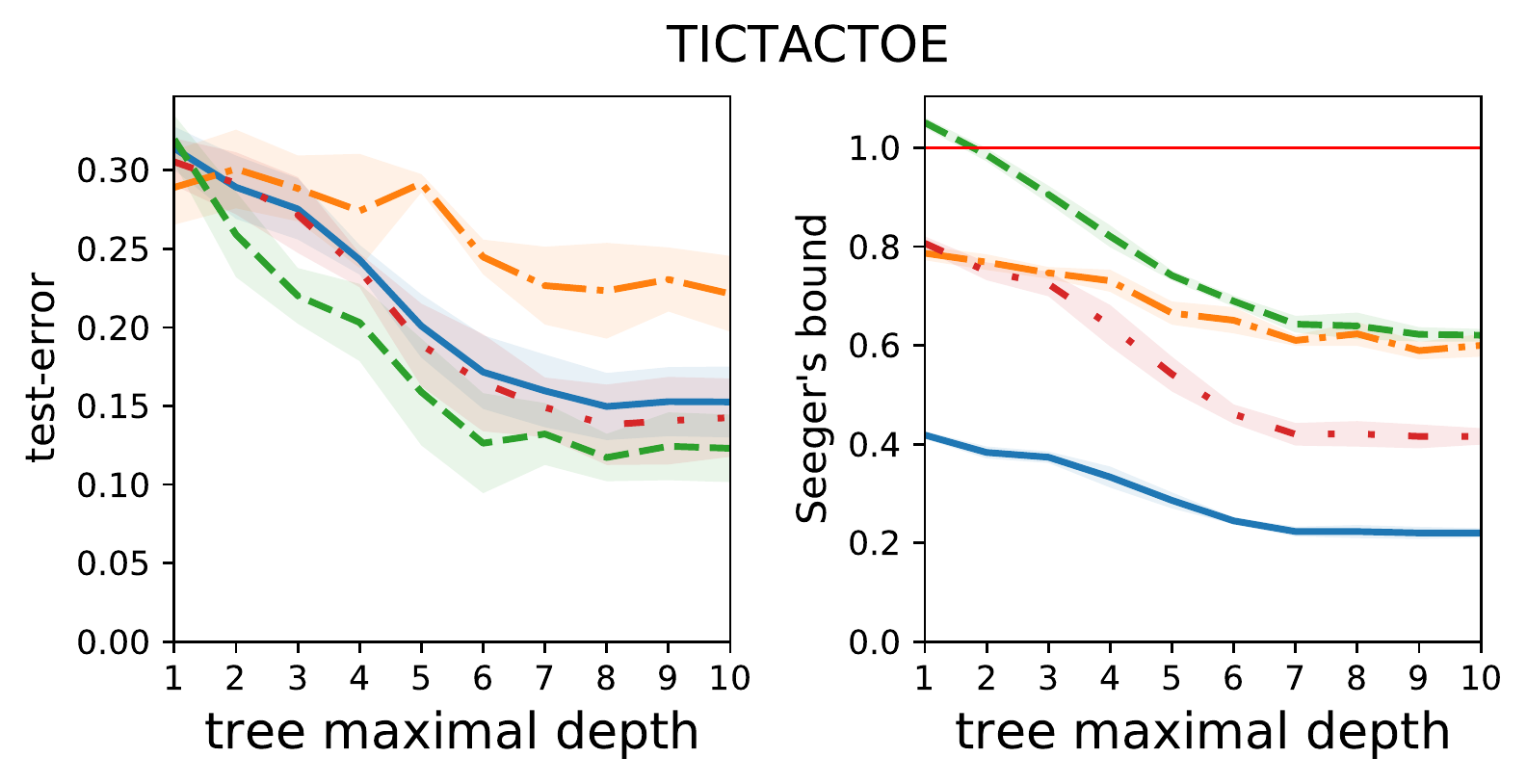}
    \includegraphics[width=0.49\textwidth]{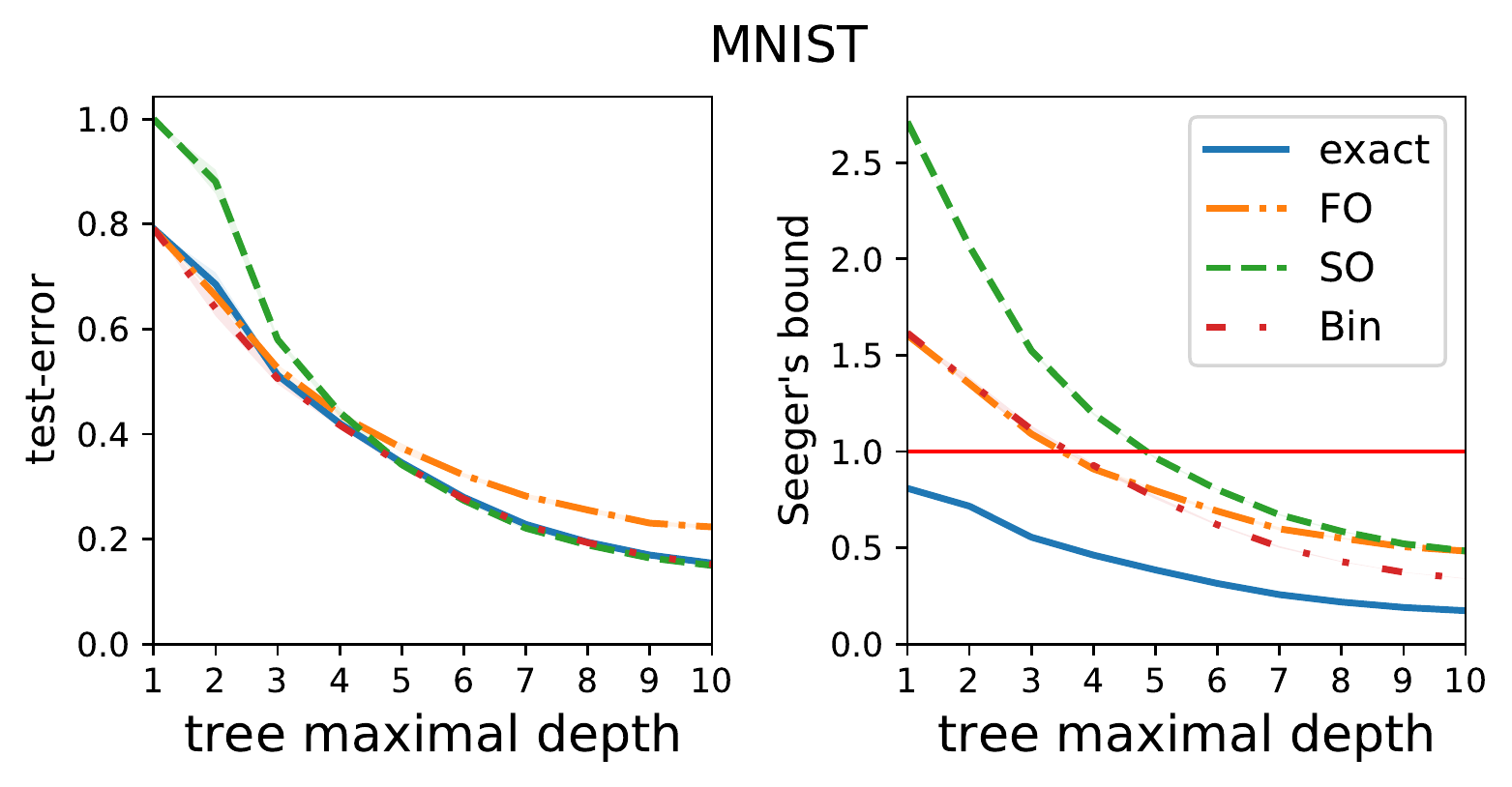}
    \caption{Test error rates and PAC-Bayesian bound values as a function of the maximal depth of the voters (decision trees here).
    We report the means and standard deviations over $4$ different runs, and mark with a red horizontal line the threshold above which Seeger's bounds are vacuous.
    Additional results are available in the appendix.}\label{fig:strength}
\end{figure}

\section{Future work}
We propose a stochastic version of the classical majority vote classifier, and we directly analyze and optimize its expected risk through the PAC-Bayesian framework.
The benefits on the model accuracy of this direct optimization are however reduced in presence of input noise, and fostering robustness in noisy contexts is the subject of future work.
Another potential improvement would consist in tackling the discussed looseness of our generalization bounds with increasing number of base predictors, by accounting for redundancy in the hypothesis space.


\begin{ack} 
We thank the anonymous reviewers for their constructive feedback and support.
This work was partially funded by the French Project APRIORI ANR-18-CE23-0015.
Experiments presented in this paper were carried out using the Grid'5000 testbed, supported by a scientific interest group hosted by Inria and including CNRS, RENATER and several Universities as well as other organizations (see \url{https://www.grid5000.fr}).
Pascal Germain is supported by the Canada CIFAR AI Chair Program, and the NSERC Discovery grant RGPIN-2020-07223.
Benjamin Guedj acknowledges partial support by the U.S. Army Research Laboratory and the U.S. Army Research Office, and by the U.K. Ministry of Defence and the U.K. Engineering and Physical Sciences Research Council (EPSRC) under grant number EP/R013616/1. Benjamin Guedj acknowledges partial support from the French National Agency for Research, grants ANR-18-CE40-0016-01 and ANR-18-CE23-0015-02.
\end{ack}

\bibliographystyle{plainnat}
\bibliography{biblio}

\newpage
\appendix

\section{Derivation details under Dirichlet assumptions}
\label{app:utils}

\subsection{Dirichlet distribution}
The probability density function of the Dirichlet distribution is defined as follows:
\begin{equation}
    \theta \sim \mathcal{D}(\alpha_1,\dots,\alpha_M),\quad\;
    \rho(\theta) = \frac{1}{B(\alpha)} \prod_{j=1}^M (\theta_j)^{\alpha_j-1},
\end{equation}
with $\alpha = [\alpha_j \in \Rb^{+}]_{j=1}^M$ the vector of the distribution parameters and $B(\alpha)$ the normalized multi-variate beta function, defined using the gamma function $\Gamma$:

\begin{equation}
    B(\alpha) = \frac{\prod_{j=1}^M \Gamma(\alpha_j)}{\Gamma(\sum_{j=1}^M \alpha_j)},
    \quad\quad\Gamma(a) = \int_{0}^\infty t^{a-1} e^{-t} dt.
\end{equation}

\subsection{KL divergence between Dirichlet distributions}

Let $\rho(\theta) = \Dc(\alpha)$ and $\pi(\theta) = \Dc(\beta)$, with $\alpha_0 = \sumj \alpha_j$ and $\beta_0 = \sumj \beta_j$.
The KL divergence between $\rho$ and $\pi$ is equal to:
\begin{equation}
    KL(\rho, \pi) = \ln(\Gamma(\alpha_0)) - \sumj \ln(\Gamma(\alpha_j)) - \ln(\Gamma(\beta_0)) + \sumj \ln(\Gamma(\beta_j)) + \sumj (\alpha_j - \beta_j) (\psi(\alpha_j) - \psi(\alpha_0))
\end{equation}
with $\psi(a)$ the digamma function: the first derivative of $\ln(\Gamma(a))$ (see Figure~\ref{fig:gamma}).

\begin{proof}We have
\begin{align}
    KL(\rho, \pi) &= \int_\Theta \rho(\theta) \ln \left(\frac{\rho(\theta)}{\pi(\theta)}\right) d \theta\\
    &= \int_\Theta \ln \left( \frac{B(\beta)}{B(\alpha)} \frac{\prodj \theta_j^{\alpha_j -1}}{\prodj \theta_j^{\beta_j -1}} \right) \rho(\theta) d \theta \\
    &= \int_\Theta \left( \ln B(\beta) - \ln B(\alpha) + \ln \left( \prodj \theta_j^{\alpha_j - \beta_j} \right) \right) \rho(\theta) d \theta \\
     &= \int_\Theta \left( \ln B(\beta) - \ln B(\alpha) + \sumj (\alpha_j - \beta_j) \ln \theta_j \right) \rho(\theta) d \theta \\
     &= \ln B(\beta) - \ln B(\alpha) + \sumj (\alpha_j - \beta_j) \int_\Theta \ln \theta_j \; \rho(\theta) d \theta \\
     &= \sumj \ln(\Gamma(\beta_j)) - \ln(\Gamma(\beta_0)) + \ln(\Gamma(\alpha_0)) - \sumj \ln(\Gamma(\alpha_j)) \nonumber \\
     &\quad+ \sumj (\alpha_j - \beta_j) (\psi(\alpha_j) - \psi(\alpha_0)) \label{eq:geom-mean}.
\end{align}

Equation~\eqref{eq:geom-mean} follows by definition of Dirichlet's geometric mean: $$\int_\Theta \ln \theta_j \; \rho(\theta) d \theta = \psi(\alpha_j) - \psi(\alpha_0).$$

\end{proof}

\begin{figure}[t]
\centering
\begin{minipage}{.45\textwidth}
\centering
    \includegraphics[width=\textwidth]{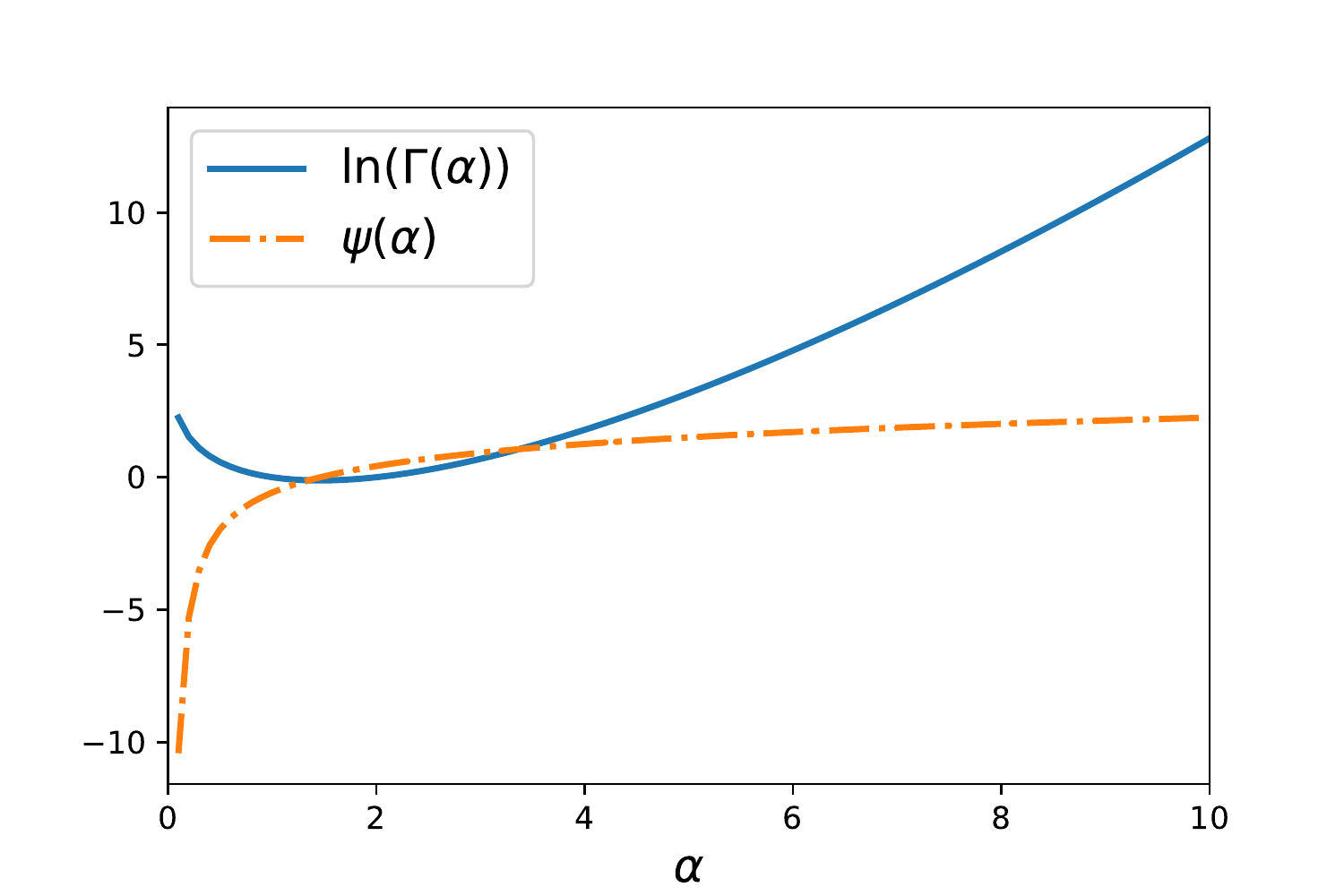}
    \caption{Functions of the gamma family.}
    \label{fig:gamma}
\end{minipage}\hfill
\begin{minipage}{.45\textwidth}
  \centering
  \includegraphics[width=\textwidth]{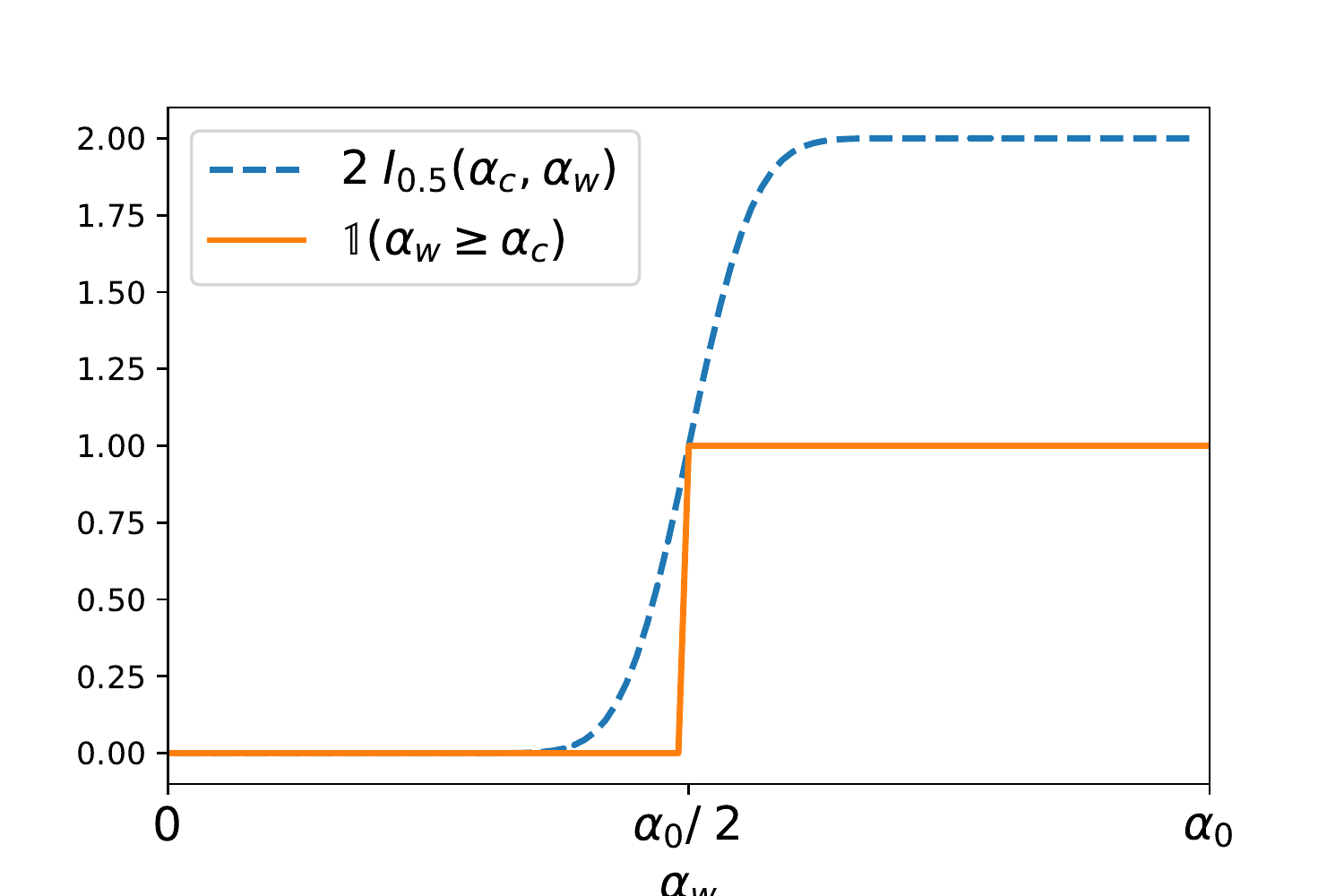}
    \caption{Oracle bound on the expected MV.}
    \label{fig:expectedMV}
\end{minipage}
\end{figure}

\subsection{Partial derivatives for gradient-based optimization}
\label{app:derivatives}
When optimizing our objective functions by gradient descent, we make use of the following partial derivatives \wrt the Dirichlet parameters $\{\alpha_j\}_{j=1}^M$.

The partial derivatives of the KL divergence are
\begin{align}
    \nabla_{\alpha_j} KL(\rho, \pi) &= \psi(\alpha_0) -  \psi(\alpha_j) + \psi(\alpha_j) - \psi(\alpha_0) + (\alpha_j - \beta_j)(\psi'(\alpha_j) - \psi'(\alpha_0)) \\
    &= (\alpha_j - \beta_j)(\psi'(\alpha_j) -  \psi'(\alpha_0)).
\end{align}

For a given data point $(x, y) \sim \Pc$, recall that $w = \{j | h_j(x) \neq y \}$ is the set of indices of the base classifiers that misclassify $(x, y)$ and $c = \{j | h_j(x) = y \}$ is the set of indices of the base classifiers that correctly classify $(x, y)$.
Let us define $w(\alpha) = \sum_{j \in w} \alpha_j$ and $c(\alpha) = \sum_{j \in c} \alpha_j = \alpha_0 - w(\alpha)$.
For any $j \in c$ we have

\begin{align}
    \nabla_{\alpha_j} I_{0.5}&(\ca, \wa) = \nonumber\\
    &(\ln 0.5 - \psi(\ca) + \psi(\alpha_0)) I_{0.5}(\ca, \wa) \nonumber\\
    &- 0.5^{\ca} \frac{\Gamma(\ca) \Gamma(\alpha_0)}{\Gamma(\wa)} \,\pFq{3}{2}(\ca,\ca,1-\wa;\ca+1 ,\ca+1;0.5)
\end{align}

and for any $j \in w$ we have

\begin{align}
    \nabla_{\alpha_j} I_{0.5}&(\ca, \wa) = \nonumber\\
    &(- \ln 0.5 + \psi(\wa) - \psi(\alpha_0)) I_{0.5}(\ca, \wa) \nonumber\\
    &+ 0.5^{\wa} \frac{\Gamma(\wa) \Gamma(\alpha_0)}{\Gamma(\ca)} \,\pFq{3}{2}(\wa,\wa,1-\ca;\wa+1 ,\wa+1;0.5)
\end{align}
where $\pFq{3}{2}(a, b, c; d, e; z)$ is the generalized hyper-geometric function:
$$\pFq{3}{2}(a, b, c; d, e; z) = \sum_{t=1}^T \frac{(a)^t \;(b)^t \;(c)^t\; z^t}{(d)^t \;(e)^t \;t!}$$
with $(.)^t$ the rising factorial.

The hyper-geometric function can be slow to evaluate, as its convergence rate varies depending on $\ca$ and $\wa$.
A possible strategy for speeding up this evaluation, apart from parallelizing computations, would be dynamic programming, \ie~storing the gradients of the incomplete beta function, as it is likely to be evaluated several times for the same $\ca$ and $\wa$.

\subsection{Oracle bound on expected Majority Vote}
\label{app:deterministic}
We now derive an oracle bound for the risk of the expected majority vote, given a distribution $\rho(\theta)$ having a Dirichlet form and concentration vector $\alpha = \{\alpha_j\}_{j=1}^M$.
The oracle bound can then be used to derive empirical PAC-Bayesian bounds on the true risk of the expected MV.
The expected MV is parameterized by the mean weighting vector $\hat{\theta}$, which for a Dirichlet distribution is given by 
$$\hat{\theta} = \; \Ex_\rho \; \theta =\; \frac{\alpha}{\alpha_0} \quad\quad \text{with} \quad \alpha_0 = \sumj \alpha_j.$$
For a given $(x, y) \sim \Pc$, let us define $w(x, y) = \{j | h_j(x) \neq y \}$ the set of indices of the base classifiers that misclassify $(x, y)$ and $c(x, y) = \{j | h_j(x) = y \}$ be the set of indices of the base classifiers that correctly classify $(x, y)$. 
The risk of the expected MV can then be measured as follows:
\begin{align}
    R(f_{\hat{\theta}}) &= \mathbb{P}\left[W_{\hat{\theta}} \geq 0.5\right] \\
    &= \mathbb{P}\left[\Ex \; W_{\theta} \geq 0.5\right] \label{eq:linwtheta}\\
    &= \mathbb{P}_{(x, y) \sim \Pc} \left[\sum_{j | h_j(x) \neq y}\frac{\alpha_j}{\alpha_0} \geq 0.5 \right] \\
    &= \mathbb{P}_{(x, y) \sim \Pc}\left[\sum_{j \in w(x, y)}\alpha_j \geq \sum_{j \in c(x, y)} \alpha_j\right].
\end{align}
Equation~\eqref{eq:linwtheta} follows by linearity of $W_{\theta}$.

We recall that the expected risk is given by 
$$\Ex_{(x, y) \sim \Pc} \; I_{0.5}\left(\sum_{j \in c(x, y)}\alpha_j,\sum_{j \in w(x, y)}\alpha_j \right) = \Ex_{(x, y) \sim \Pc} \; I_{0.5}\left(\alpha_0 - \alpha_{w(x,y)}, \alpha_{w(x,y)}\right)$$
with $\alpha_{w(x,y)} = \sum_{j \in w(x, y)}\alpha_j$.
As $I_{0.5}()$ is monotonically increasing in its second argument and $I_{0.5}\left(\frac{\alpha_0}{2}, \frac{\alpha_0}{2}\right) = 0.5$, we can then relate the two risks:

\begin{align}
    R(f_{\hat{\theta}}) \leq 2 \Ex_\rho R(f_\theta).
\end{align}

See Figure~\ref{fig:expectedMV} for an illustration.

\section{Proof of Theorem~\ref{eq:informed-prior}}
\label{ap:cross-prior}
Let us define the following empirical risks
\begin{align*}
    \hat{R}_{\le m}(\rho_{>m}) = \int_{\Theta_{>m}} \hat{R}_{\le m}(f_\theta) \rho_{>m}(d\theta) \text{ and } \hat{R}_{> m}(\rho_{\le m}) = \int_{\Theta_{\le m}} \hat{R}_{> m}(f_\theta) \rho_{\le m}(d\theta)
\end{align*}

and true risks
\begin{align*}
    R(\rho_{>m}) = \int_{\Theta_{>m}} R(f_\theta) \rho_{>m}(d\theta) \text{ and } R(\rho_{\le m}) = \int_{\Theta_{\le m}} R(f_\theta) \rho_{\le m}(d\theta).
\end{align*}

\setcounter{theorem}{1}
\begin{theorem}[Seeger's bound with informed priors]
Let $\pi_{\le m}$ and $\rho_{\le m}$ be the prior and posterior distributions on $\Theta_{\le m}$ and $\pi_{>m}$ and $\rho_{>m}$ the prior and posterior distributions on $\Theta_{>m}$.
For any $p \in (0, 1)$ and $\delta \in (0,1)$ with probability at least $1{-}\delta$ over samples \mbox{$S = \{(x_i, y_i) {\sim} \Pc\}_{i=1}^n$} of size $n$ we have simultaneously:
\begin{multline*}
    \kl\!\left(p R(\rho_{>m})+ (1-p) R(\rho_{\le m}) \middle\| p \hat{R}_{\le m}(\rho_{>m}) + (1-p) \hat{R}_{> m}(\rho_{\le m}) \right) \\ \le  \frac{p\; KL(\rho_{>m}, \pi_{>m})}{m}+\frac{(1-p)\;KL(\rho_{\le m}, \pi_{\le m})}{n-m} + \frac{\ln\tfrac{4\sqrt{m(n{-}m)}}{\delta}}{n}
\end{multline*}

with $R(\rho_{>m}) = \int_{\Theta_{>m}} R(f_\theta) \rho(d\theta)$, and $R(\rho_{\le m}) = \int_{\Theta_{\le m}} R(f_\theta) \rho(d\theta)$,\\ and $\hat{R}_{\le m}(\rho_{>m}) = \int_{\Theta_{>m}} \hat{R}(f_\theta) \rho(d\theta)$, and $\hat{R}_{\le m}(\rho_{>m}) = \int_{\Theta_{\le m}} \hat{R}(f_\theta) \rho(d\theta)$.
\end{theorem}

\begin{proof}
From the joint convexity of the binary $\kl$ divergence, we have for any $p \in [0, 1]$
\begin{align*}
    &\kl\!\left(p \hat{R}(\rho_{>m})+ (1-p) \hat{R}(\rho_{\le m}) \;\middle\|\; p R_{\le m}(\rho_{>m}) + (1-p) R_{> m}(\rho_{\le m}) \right)\\
   \le\ \ &p\kl\!\left(\hat{R}(\rho_{>m})\middle\|R_{\le m}(\rho_{>m})\right)  
    + (1-p)\kl\!\left(\hat{R}(\rho_{\le m}) \middle\| R_{> m}(\rho_{\le m}) \right)\!.
\end{align*}

For $\forall\rho_{>m}$ defined on $\Theta_{>m}$:
\begin{align}
    \mathbb{P}_{S_{\le m}\sim {\cal P}^m}\left[p\; \kl\!\left(\hat{R}(\rho_{>m})\middle\|R_{\le m}(\rho_{>m})\right) \le p\frac{KL(\rho_{>m}, \pi_{>m})+\ln\tfrac{4\sqrt{m}}{\delta}}{m}\right] \ge 1{-}\tfrac{\delta}{2} \nonumber\\
    \Rightarrow \mathbb{P}_{S\sim {\cal P}^n}\left[p\; \kl\!\left(\hat{R}(\rho_{>m})\middle\|R_{\le m}(\rho_{>m})\right) \le p\frac{KL(\rho_{>m}, \pi_{>m})+\ln\tfrac{4\sqrt{m}}{\delta}}{m}\right] \ge 1{-}\tfrac{\delta}{2};\label{eq:p1}
\end{align}

and for $\forall\rho_{\le m}$ defined on $\Theta_{\le m}$:
\begin{align}
    \mathbb{P}_{S_{>m}\sim {\cal P}^{n{-}m}}\left[(1-p) \kl\!\left(\hat{R}(\rho_{\le m}) \middle\| R_{> m}(\rho_{\le m}) \right) \le (1-p)\frac{KL(\rho_{\le m}, \pi_{\le m})+\ln\tfrac{4\sqrt{n-m}}{\delta}}{n-m}\right] \ge 1{-}\tfrac{\delta}{2} \nonumber\\
    \Rightarrow \mathbb{P}_{S\sim {\cal P}^n}\left[(1-p)\; \kl\!\left(\hat{R}(\rho_{\le m}) \middle\| R_{> m}(\rho_{\le m}) \right) \le (1-p)\frac{KL(\rho_{\le m}, \pi_{\le m})+\ln\tfrac{4\sqrt{n-m}}{\delta}}{n-m}\right] \ge 1{-}\tfrac{\delta}{2}.\label{eq:p2}
\end{align}

Combining Equation~\eqref{eq:p1} and Equation~\eqref{eq:p2} using the union bound, we obtain the desired result with $1-\delta$ probability. 
\end{proof}

\section{Additional experimental results}

\subsection{Analysis of the role of bound regularization on the posterior}

\paragraph{Synthetic dataset.}
We study the behavior of the method during training on a simple toyset, built from two normal distributions $\mathcal{N}_1([-1, 0], \mathrm{diag}([0.1, 1]))$ and $\mathcal{N}_2([1, 0], \mathrm{diag}([0.1, 1]))$, one per class, so that the two classes are almost perfectly separable on the first dimension.
from each class distribution for training, $n$ points are drawn, and $1,000$ points for testing.

\paragraph{Base predictors and prior.}
We fix the set of base voters to $M{=}4$ axis-aligned decision stumps: $2$ per class and centered in $0$ on each dimension, so that the problem is well specified as the optimal classifier is in the predictor set.
Then, we set the parameters $\beta$ of the prior distribution all to $0.1$ to encourage sparse solutions for the posterior, and we initialize the posterior's parameters by drawing $\alpha_j \sim \mathcal{U}(0.01, 2)$.

\paragraph{Study of the impact of optimizing the PAC-Bayesian bound.}

\begin{figure}
    \centering
    \includegraphics[width=0.325\textwidth]{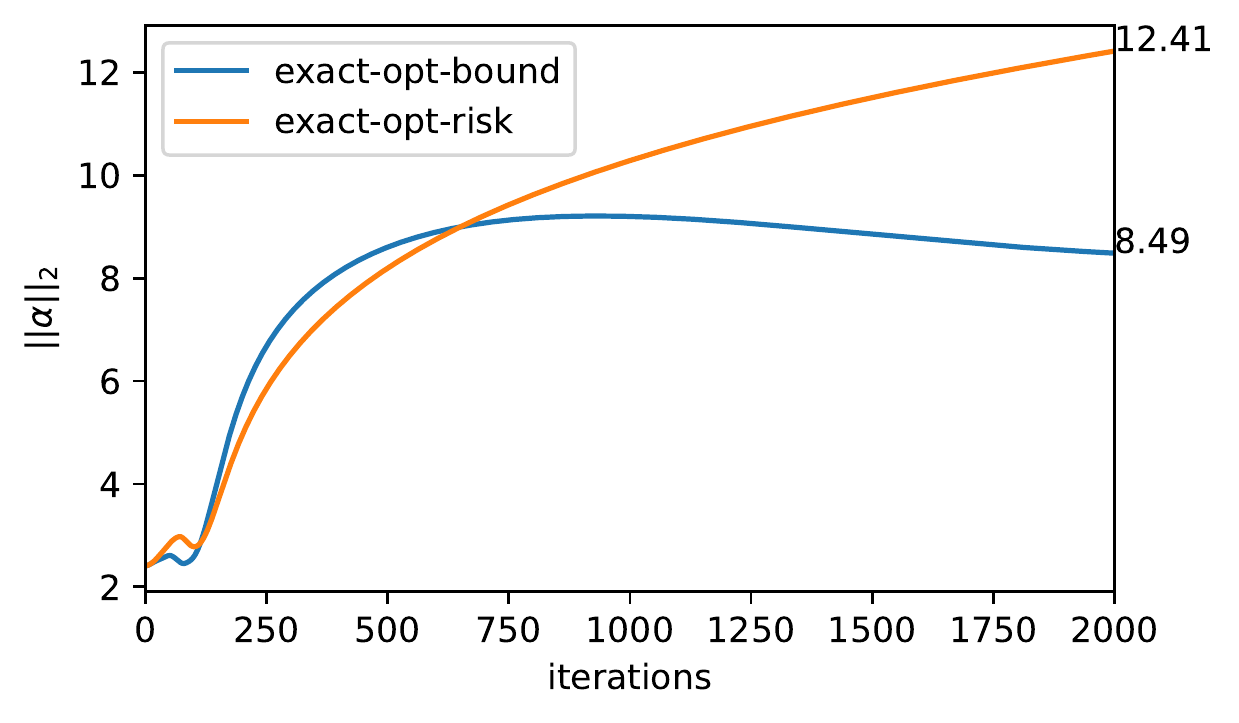}
    \includegraphics[width=0.325\textwidth]{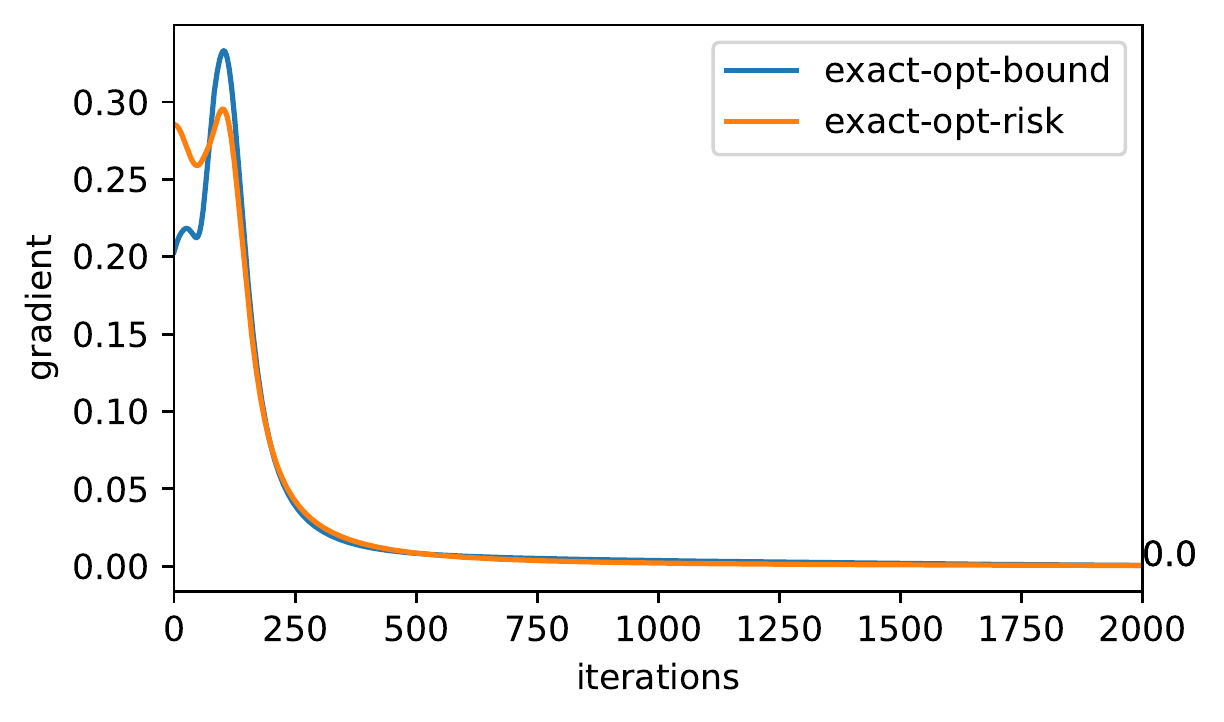}
    \includegraphics[width=0.325\textwidth]{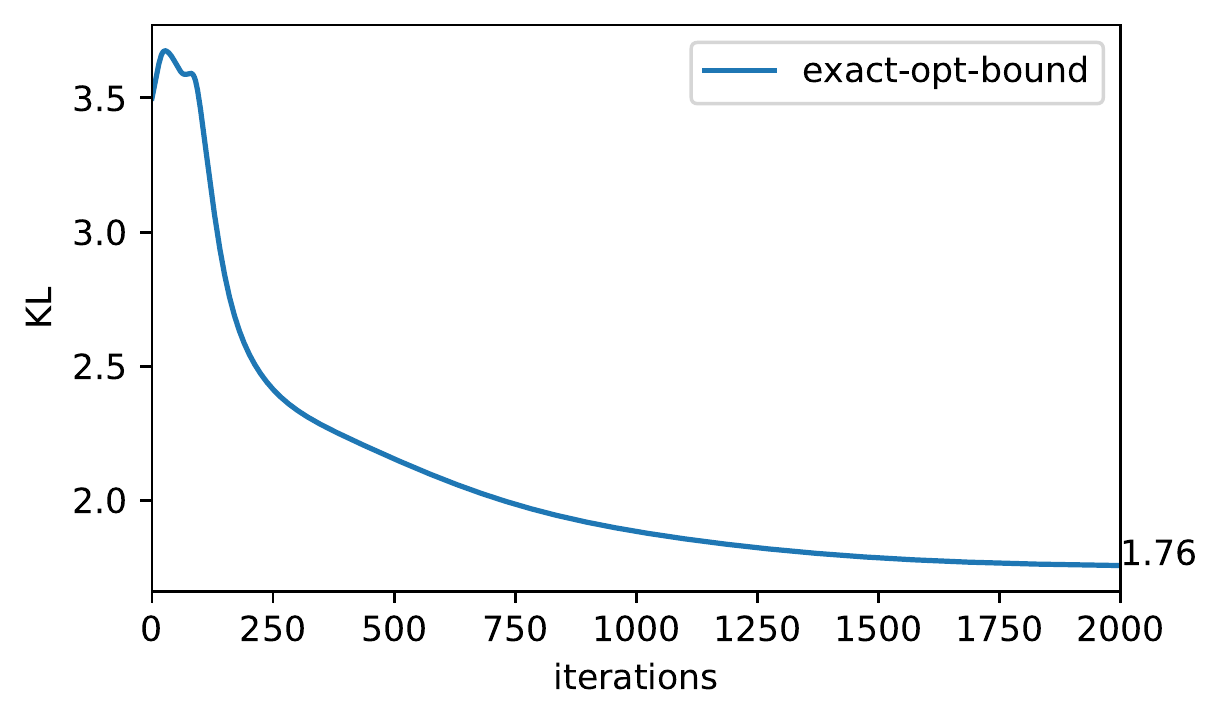}
    
    \caption{Evolution of $\|\alpha\|_2$ (left), its gradient (center) and $KL(\alpha, \beta)$ (right) during training using the \emph{exact} method.
    In each plot we compare the values obtained when optimizing Seeger's Bound (Equation~\eqref{eq:seeger}) with those obtained when optimizing only the empirical risk.}
    \label{fig:normals-seeger}
\end{figure}

In Figure~\ref{fig:normals-seeger}, we compare the evolution of $\alpha$ during training for the exact method and $n{=}50$.
In each plot we compare the posterior parameters $\alpha$ obtained when optimizing Bound (Equation~\eqref{eq:seeger}) with the one obtained when optimizing only the empirical risk (hence, without any regularization).
Namely, we study the evolution during training of $\|\alpha\|_2$, of $\alpha$'s gradients and of the KL divergence \wrt prior: $KL(\rho, \pi)$.
We notice that without regularization, $\alpha$ diverges, while it tends to a constant and smaller value when optimizing the bound.
This behaviour is not due to optimization instability, as shown by the gradient smoothly tending to $0$ for both methods.
Instead, it can be explained considering that without regularization the stochastic MV tends to concentrate around the optimal MV, \ie a single $\theta$.
This behavior results in reducing the variance of the distribution, thus increasing the concentration parameters $\alpha$.
In contrast, when optimizing the bound the model fits the empirical risk and finds optimal solutions that are also close to the prior.

\paragraph{Comparison with baselines.}

\begin{figure}
    \centering
    \includegraphics[width=\textwidth]{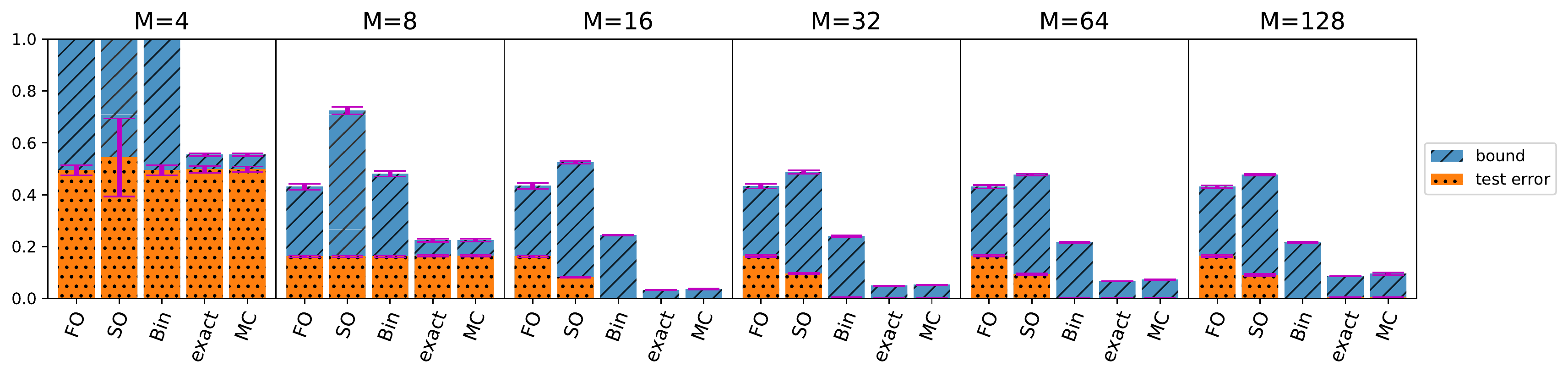}
    \caption{Comparison in terms of test error rates and PAC-Bayesian bound values.
    Each block corresponds to a different number of predictors $M$.
    We report the means (bars) and standard deviations (vertical, magenta lines) over $10$ different runs.}
    \label{fig:moons}
\end{figure}

We report a comparison of \emph{FO}, \emph{SO}, \emph{Bin} and our method on the binary classification \textbf{two-moons} dataset.
We report in Figure~\ref{fig:moons} the results obtained after optimizing the respective PAC-Bayesian bounds.
Notice that our bounds are tighter than the baselines, for any value of $M$: 
this is principally due to the fact that our method consistently obtains lower error rates, being able to better fit the training set (as shown in Figure~\ref{fig:moons-pred}), but also to the fact that our bound is not based on an oracle upper bound on the true risk, unlike the baselines.

\paragraph{Sensitivity to noise.}\label{app:noise}
We observe that on \emph{two-moons} our method outperforms the baselines also in terms of error rates.
However, this does not seem to be the case on real benchmarks where its test errors are generally aligned with those of the baselines.
We conduct an additional experiment to study whether this phenomenon is due to sensitivity to noise, such as input noise.
Indeed our learning algorithm optimizes the $01$-loss, which does not distinguish points with margins close or far from $0.5$ because of its discontinuity in $W_\theta = 0.5$.
In Figure~\ref{fig:noise}, we assess training and test errors, and Seeger's bound values with increasing input noise.
For this experiment, we generated training and test sets as before, and with Gaussian noise $\mathcal{N}(0, \sigma^2)$ added to the inputs.
As expected, all methods degrade with increasing noise.
In particular, the test errors of \emph{exact} and \emph{Bin} worsen the fastest and the benefits of using them vanish from $\sigma^2 > 0.35$. 

\begin{figure}
    \centering
    \includegraphics[width=\textwidth]{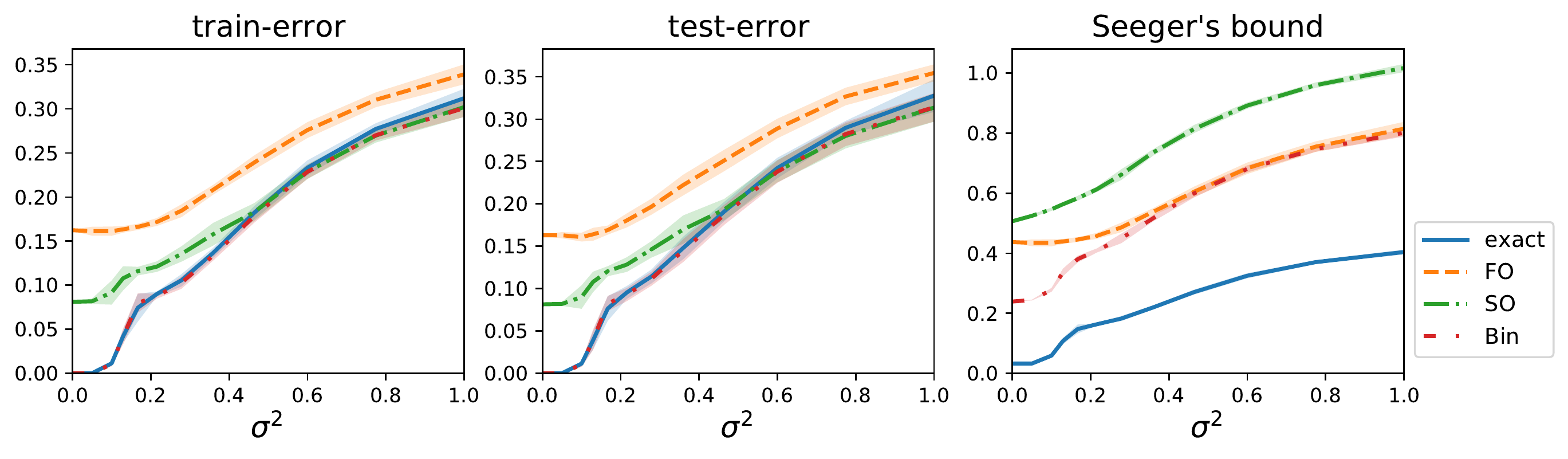}
    \caption{Comparison in terms of error rates and PAC-Bayesian bounds, depending on the magnitude of input noise $\mathcal{N}(0, \sigma^2)$.
    We report the means and standard deviations over $10$ different runs.}
    \label{fig:noise}
\end{figure}

\subsection{Additional results on real benchmarks}

\paragraph{Dataset descriptions}\label{app:datasets}
We consider several classification datasets from UCI~\citep{Dua2019}, LIBSVM~\url{https://www.csie.ntu.edu.tw/~cjlin/libsvm/} and Zalando~\citep{xiao2017}, of different number of features and of points:
\begin{itemize}
    \item \emph{Haberman} (UCI): prediction of survival of $n=306$ patients who had undergone surgery from $d=3$ anonymized features;
    \item \emph{TicTacToe} (UCI): determination of a win for player $x$ at TicTacToe game of any of the $n=958$ board configurations ($d=9$ categorical states); 
	\item \emph{Svmguide1} (LIBSVM): $d=4$ features, $n=7,089$ instances and $2$ classes (no description available);
	\item \emph{Mushrooms} (UCI): prediction of edibility of $n=8,124$ mushroom sample, given their $d=22$ categorical features describing their aspect;  
	\item \emph{Phishing} (LIBSVM): prediction of phishing websites ($n=2456$ websites and $d=68$ binary encoded features);
	\item \emph{Adult} (LIBSVM a1a): determining whether a person earns more than 50K a year ($n=32,561$ people and $d=123$ binary features);
	\item \emph{CodRNA} (LIBSVM): detection of non-coding RNAs among $n=59,535$ instances and from $d=$ features;
	\item \emph{Pendigits} (UCI): recognition of hand-written digits ($10$ classes, $d=9$ features and $n= 12,992$);
	\item \emph{Protein} (LIBSVM): $d=357$ features, $n=24,387$ instances and $3$ classes;
	\item \emph{Shuttle} (UCI): $d=9$ features, $n=58,000$ and $7$ classes;
	\item \emph{Sensorless} (LIBSVM): prediction of motor condition ($n=58,509$ instances and $11$ classes), with intact and defective components, from $d=48$ features extracted from electric current drive signals;
	\item \emph{MNIST} (LIBSVM): prediction of hand-written digits ($n=70,000$ instances and $10$ classes) from $d=28 \times 28$ gray-scale images;
	\item \emph{Fashion-MNIST} (Zalando): prediction of cloth articles ($n=70,000$ instances and $10$ classes) from $d=28 \times 28$ gray-scale images.
\end{itemize}

At each run of an algorithm, we randomly split a dataset in training and test sets of sizes $80\%-20\%$ respectively.
Note that we do not make use of a validation set, as we use the PAC-Bayesian bounds as estimate of the test error for model selection.
Finally, we convert all categorical features to numerical using an ordinal encoder and z-score all features using the statistics of the training set.

\paragraph{Choice of prior.}\label{app:prior}

In all the other experiments, we fixed the prior distribution (parameterized by $[\beta_j]_{j=1}^M$) to the uniform, \ie $ \beta_j = 1, \; \forall j$.
This choice was to make the comparison with the baselines as fair as possible, as their prior was also fixed to the uniform (categorical).
However, we can bias the sparsity of the posterior, or conversely its concentration, by choosing a different value for the prior distribution parameters.
In some cases, tuning the prior parameters allows to obtain better performance, as reported in Figure~\ref{fig:prior}.
In particular, on \emph{Protein} encouraging sparser solutions generally provides better results, confirmed by the fact that the best baseline on this dataset, \emph{FO}, is known to output sparse solutions.
On the contrary, on datasets where methods accounting for voter correlation outperform \emph{FO}, such as on \emph{MNIST}, encouraging solutions to be concentrated and close to the simplex mean yields better performance. 
In general, these results suggest that the choice of prior distribution has a high impact on the learned model's performance and tuning its concentration parameters would be a viable option for improving results.

\begin{figure}[hb!]
    \centering
    \includegraphics[width=\textwidth]{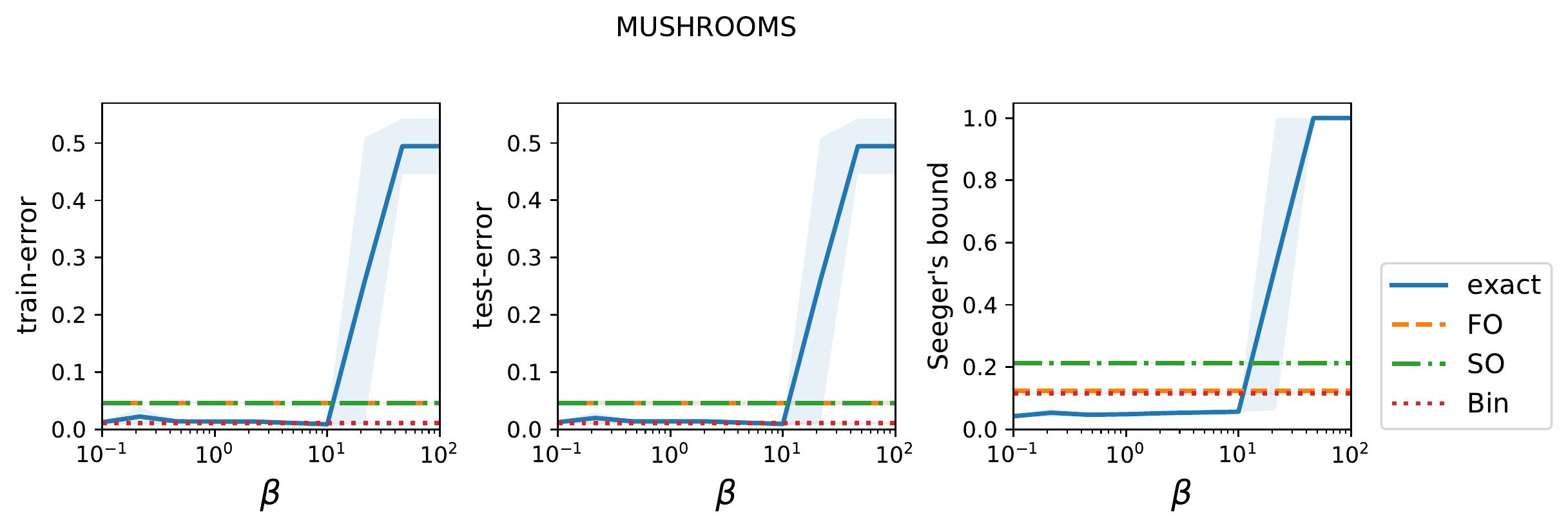}
    \includegraphics[width=\textwidth]{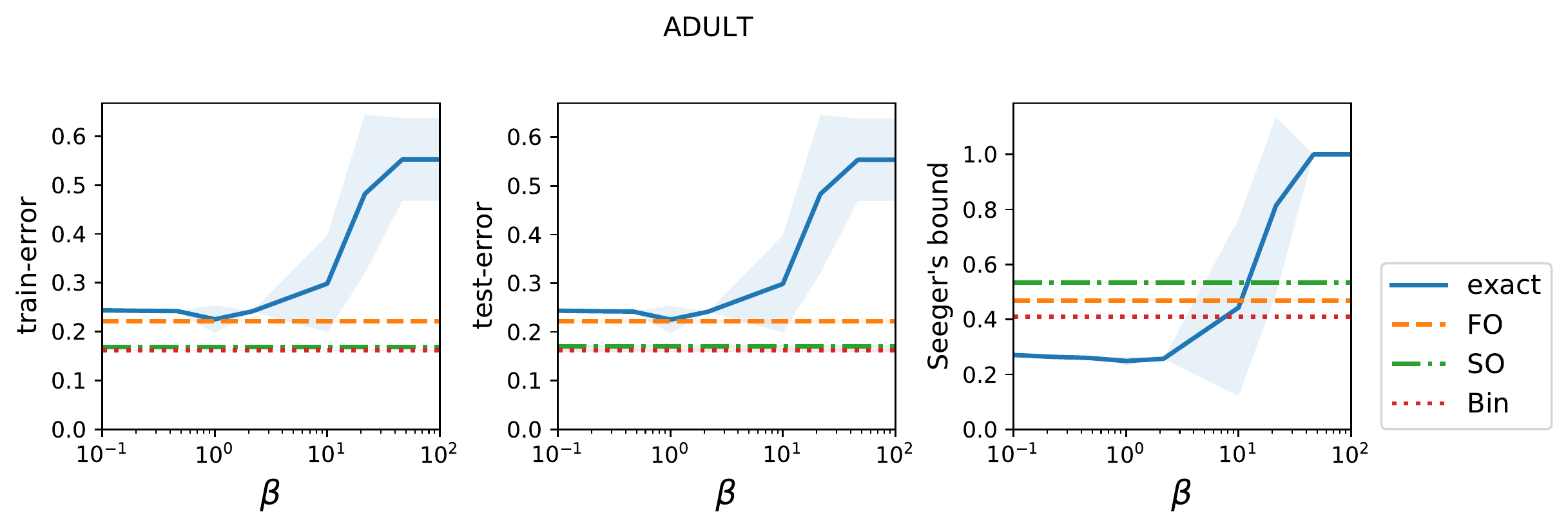}
    \includegraphics[width=\textwidth]{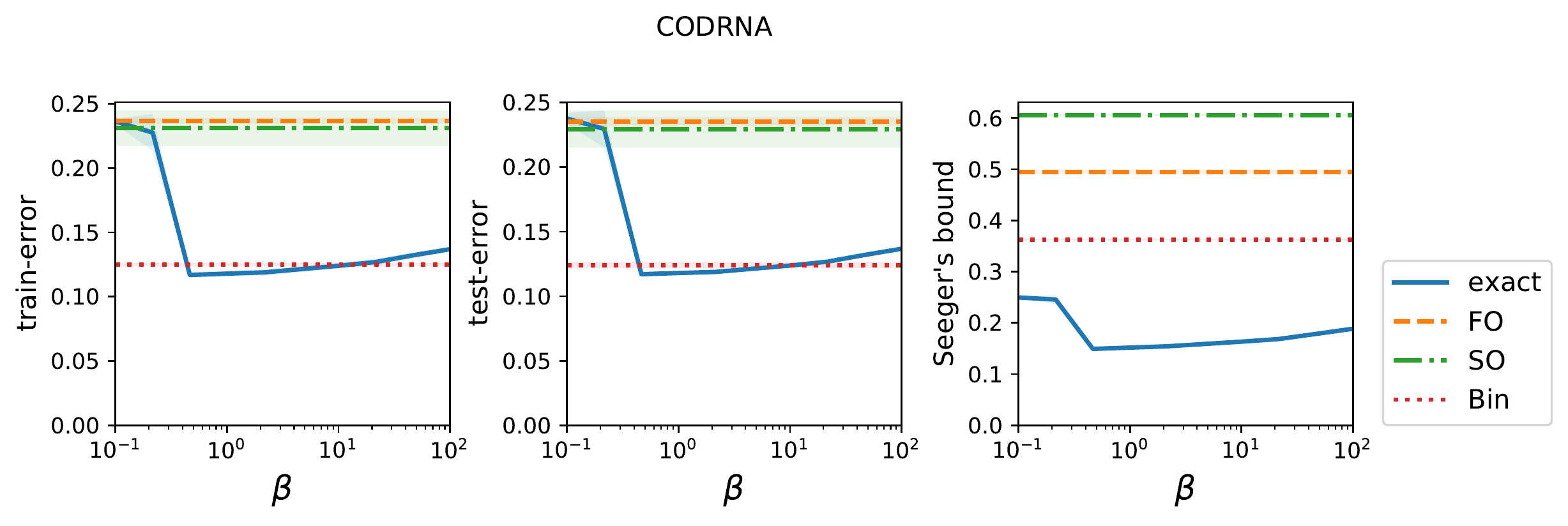}
\end{figure}
\begin{figure}[h!]
    \centering
    \includegraphics[width=\textwidth]{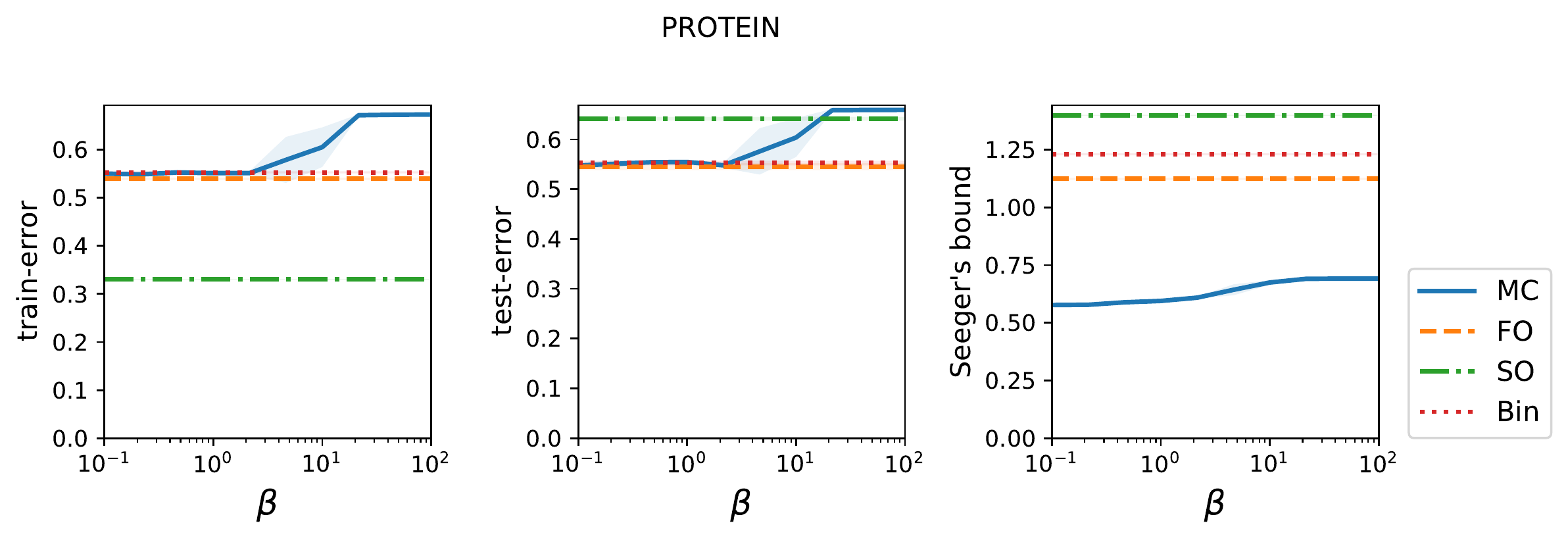}
    \includegraphics[width=\textwidth]{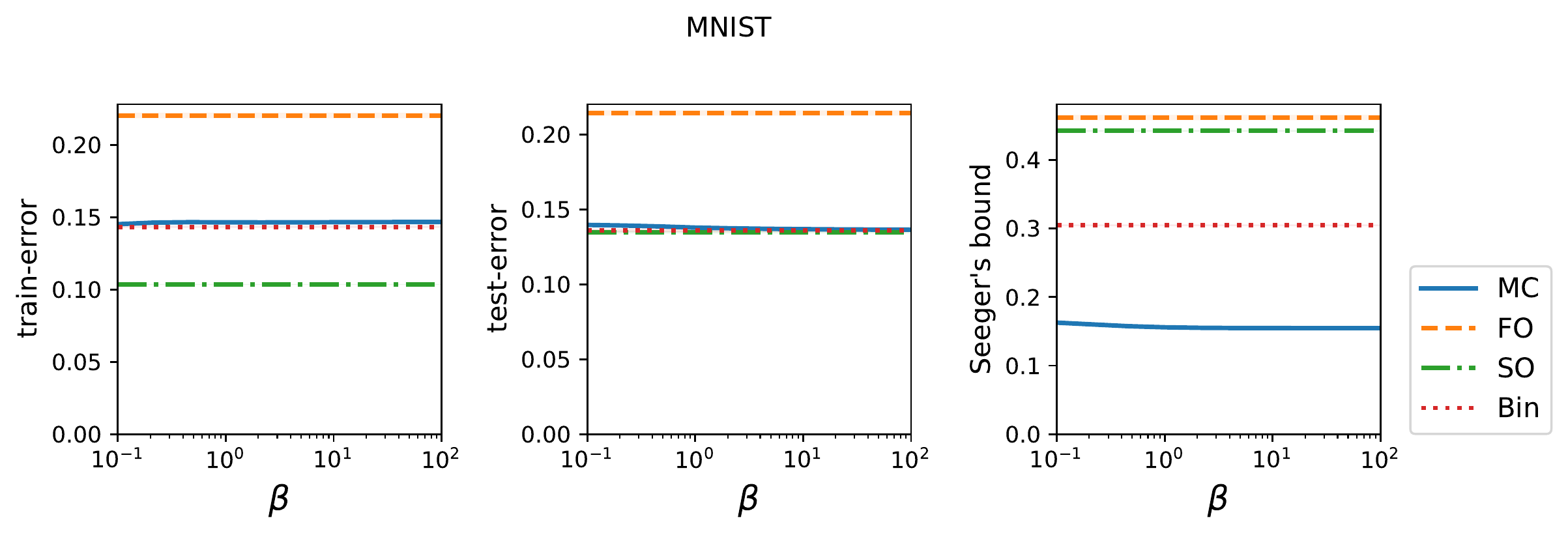}
    \caption{Study of impact of prior on posterior's performance. 
    We fix all $M$ prior's parameters to $\beta$, on the x axis: 
    the smaller $\beta$, the sparser the posterior is encouraged to be.
    We plot average and standard deviations over $4$ trials.}\label{fig:prior}
\end{figure}

\subsection{Impact of voter strength}
We report the complete study on the impact of voter strength on the learned models.
More precisely we provide results for additional datasets as well as the study of the expected strength of a voter as a function of the tree maximal depth. 
Recall that as hypothesis set, we learn a Random Forest with $100$ decision trees for which we bound the maximal depth between $1$ and $10$.
In Figure~\ref{fig:strength-app}, we can see that limiting the maximal depth is an effective way for controlling the strength of the voters, measured as the expected accuracy of a random voter.
Apart from \emph{Protein}, where decision trees do not seem to be a good choice of base predictor, increasing the strength of the voters generally yields more powerful ensembles for all methods.
Our method has error rates comparable with the best baselines and enjoys tight and non-vacuous generalization guarantees for any tree depth.
Finally, by comparing \emph{SO}'s training and test errors we notice that this method tends to overfit the dataset especially when the base classifiers are weaker (tree depth close to $1$).

\begin{figure}[h!]
    \centering
    \includegraphics[width=0.25\textwidth]{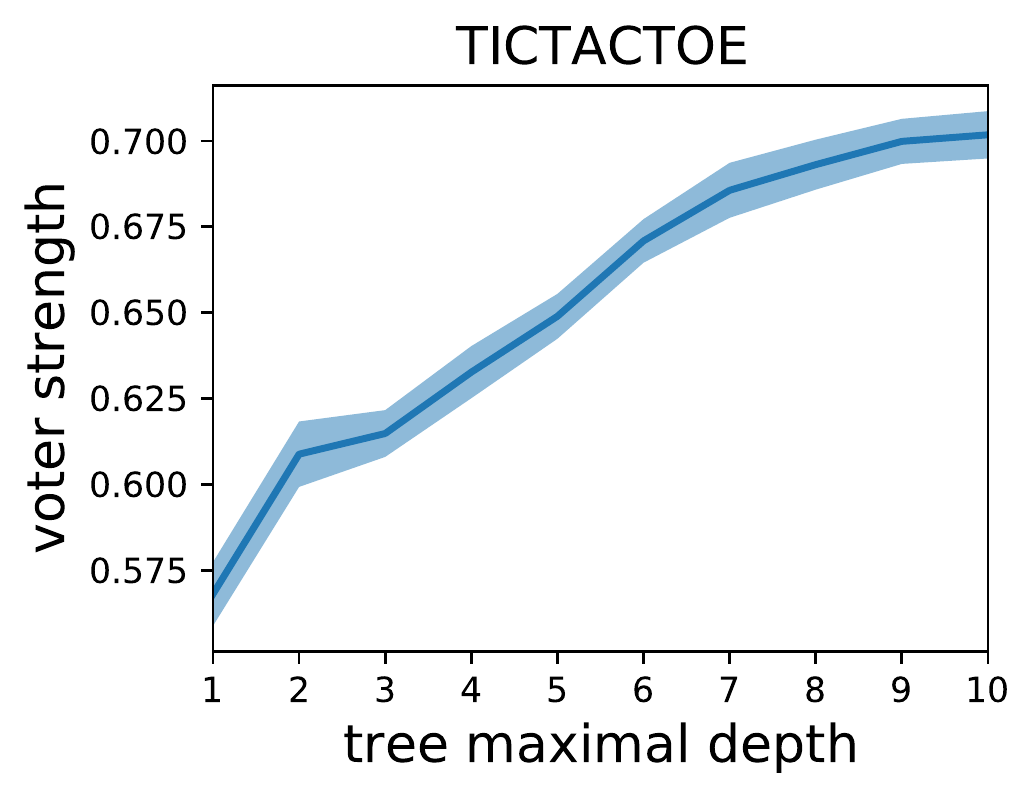}
    \includegraphics[width=0.7\textwidth]{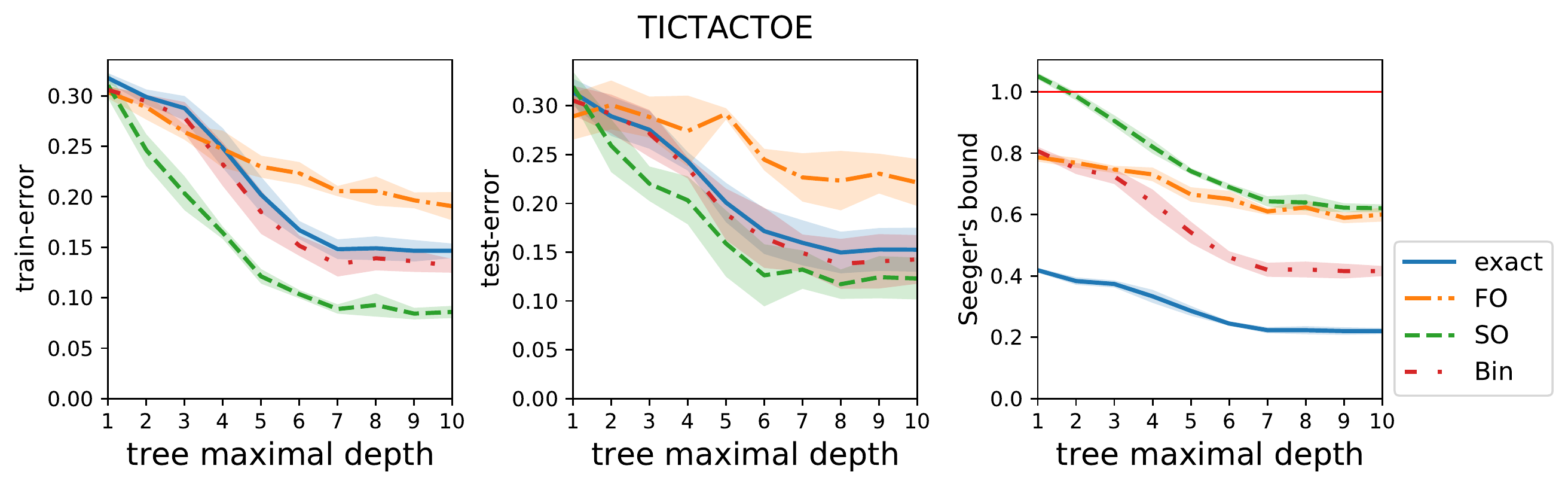}\\
    \includegraphics[width=0.25\textwidth]{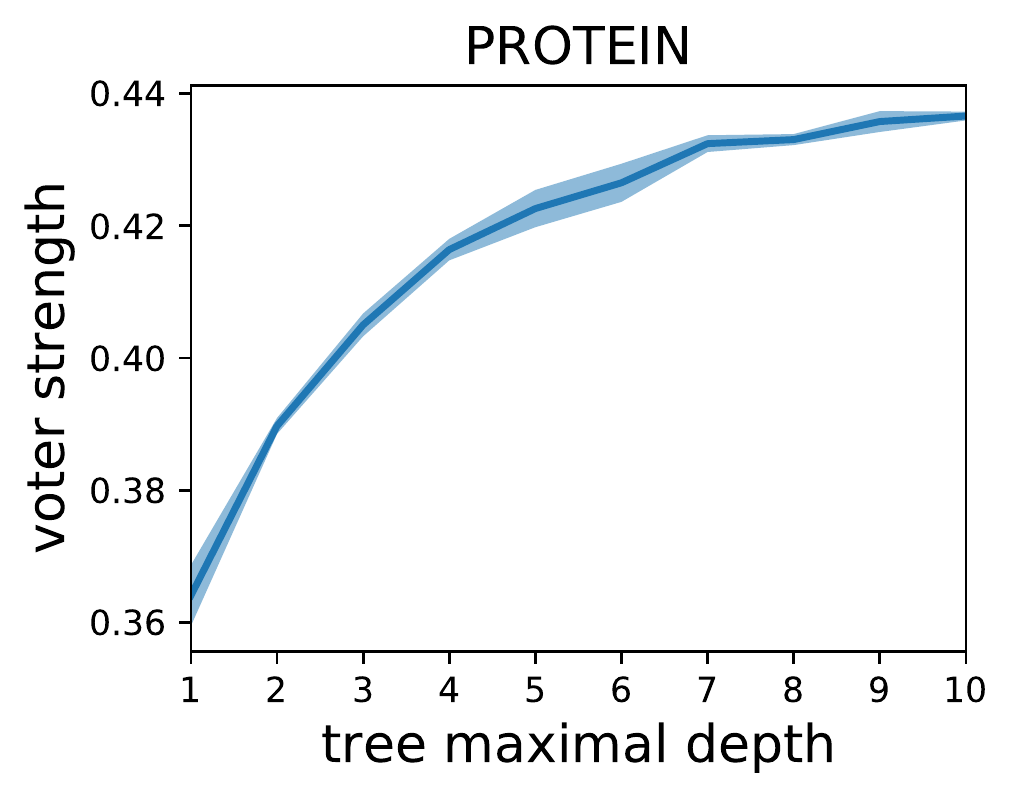}
    \includegraphics[width=0.7\textwidth]{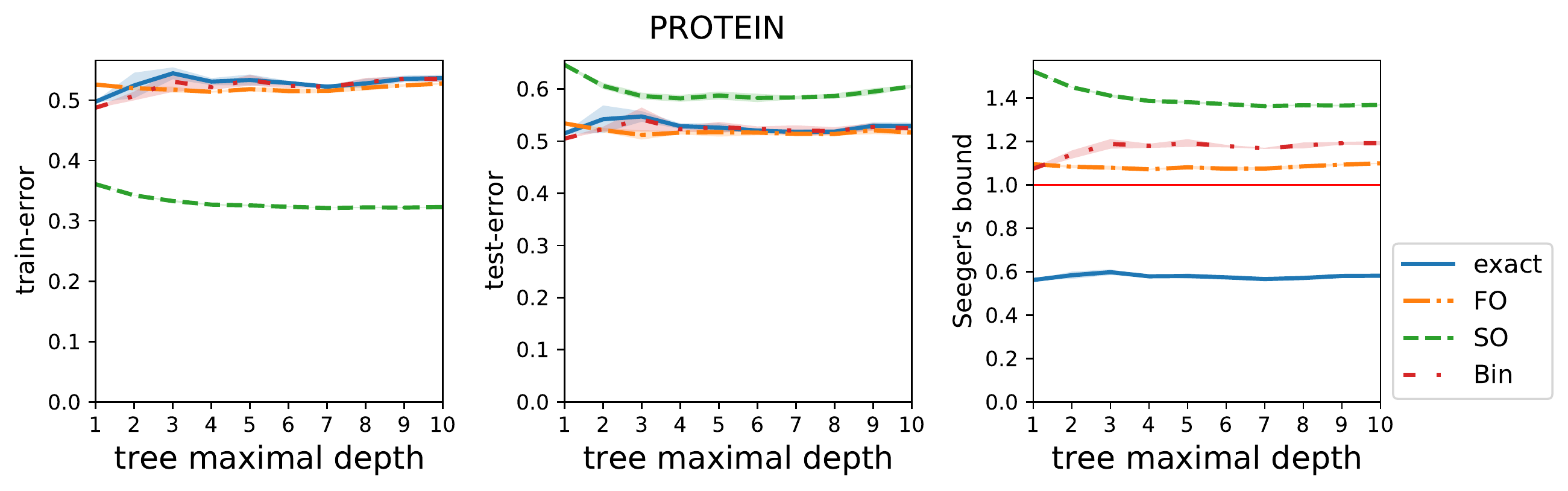}\\
    \includegraphics[width=0.25\textwidth]{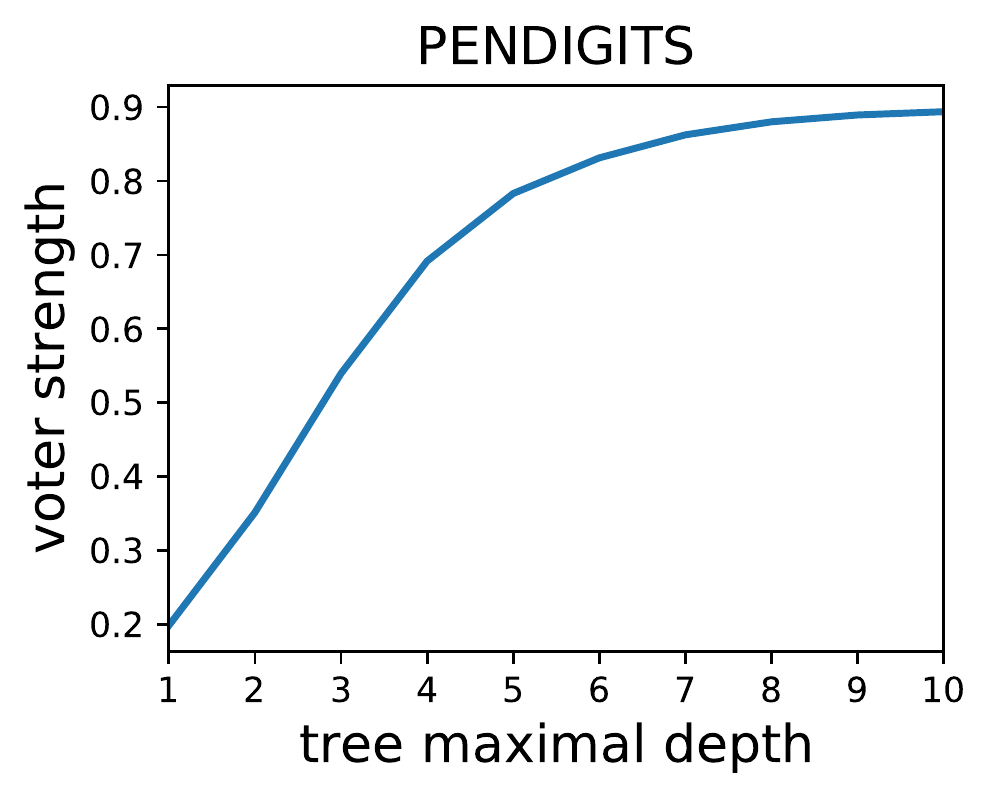}
    \includegraphics[width=0.7\textwidth]{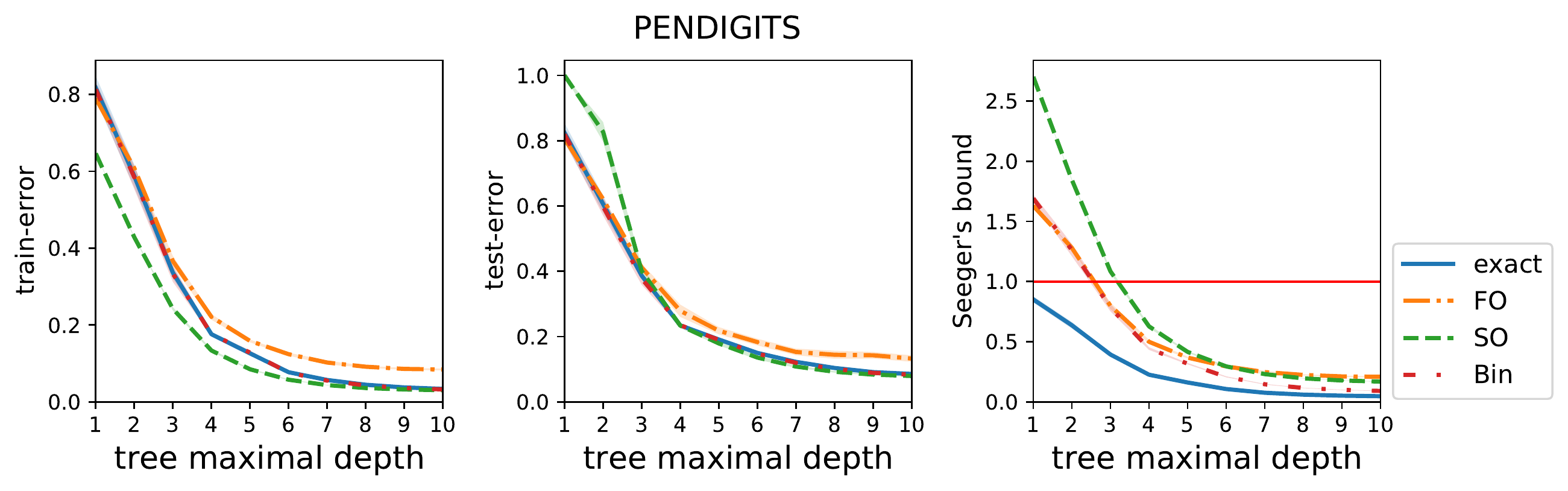}\\
    \includegraphics[width=0.25\textwidth]{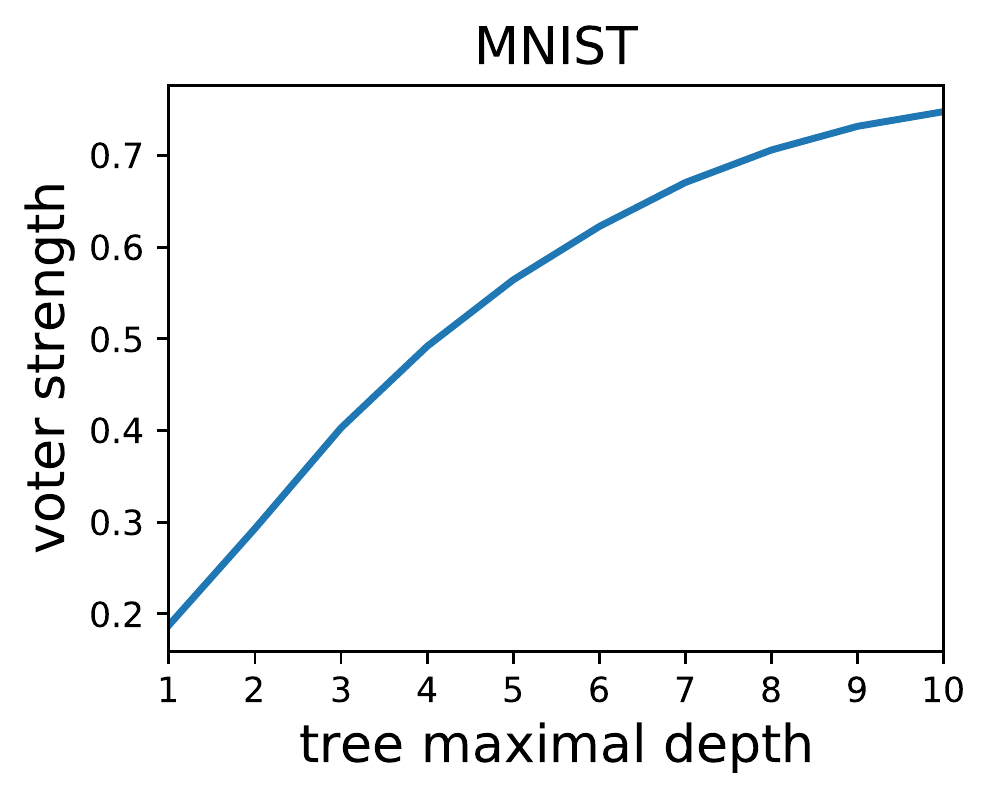}
    \includegraphics[width=0.7\textwidth]{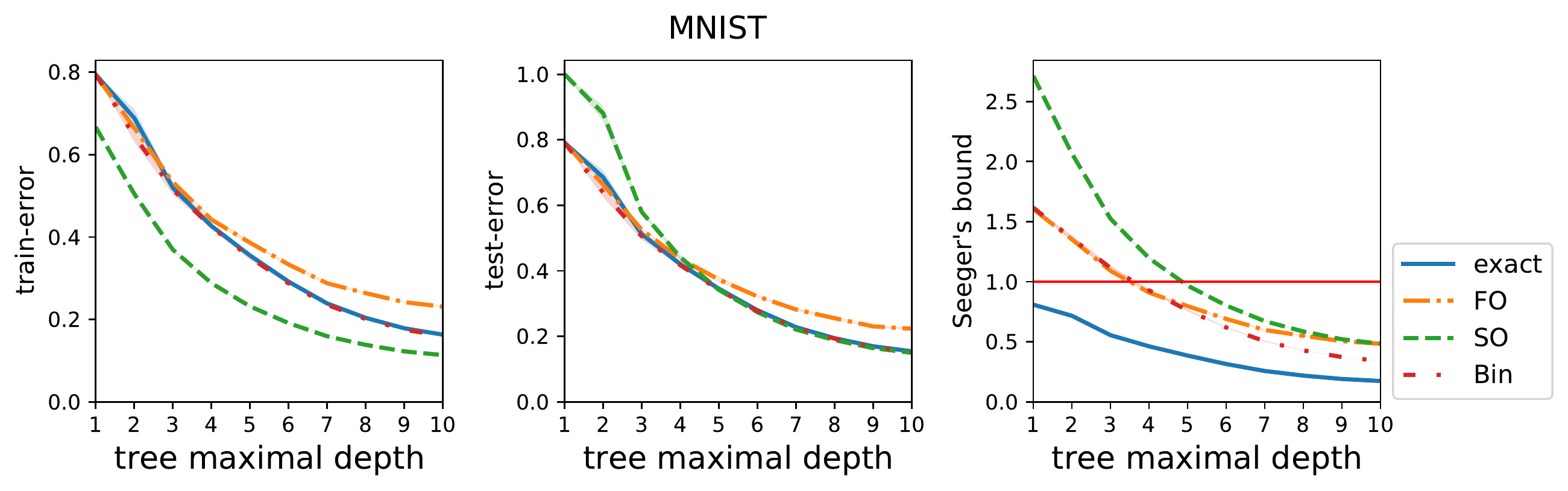}\\
    \includegraphics[width=0.25\textwidth]{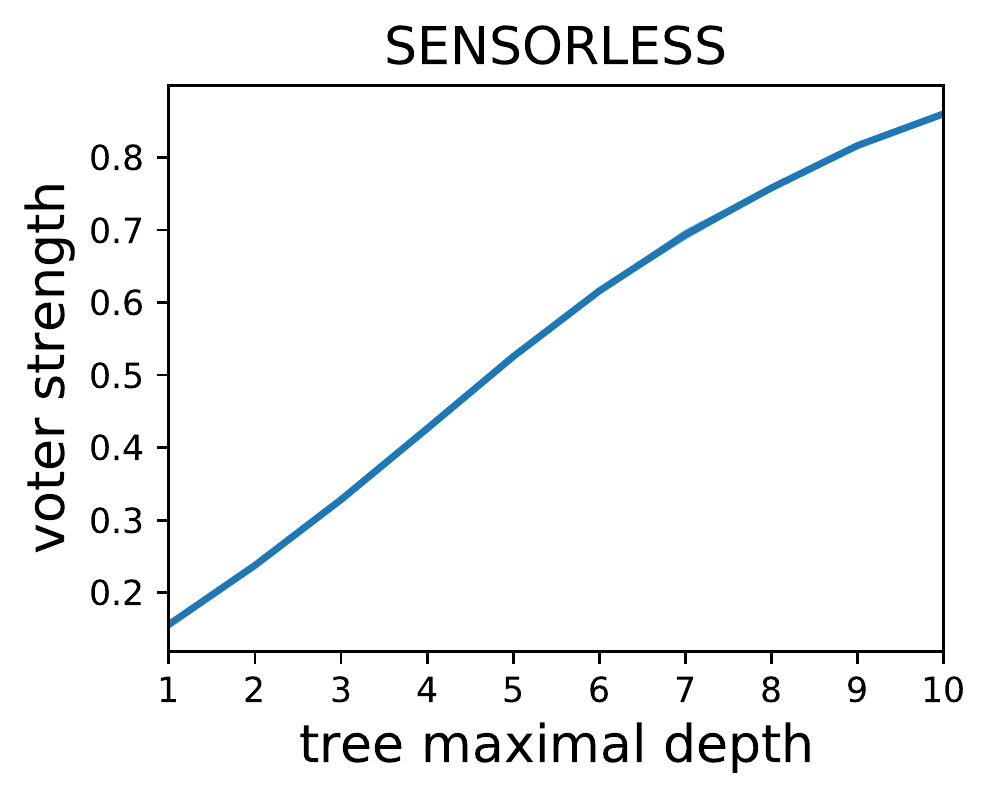}
    \includegraphics[width=0.7\textwidth]{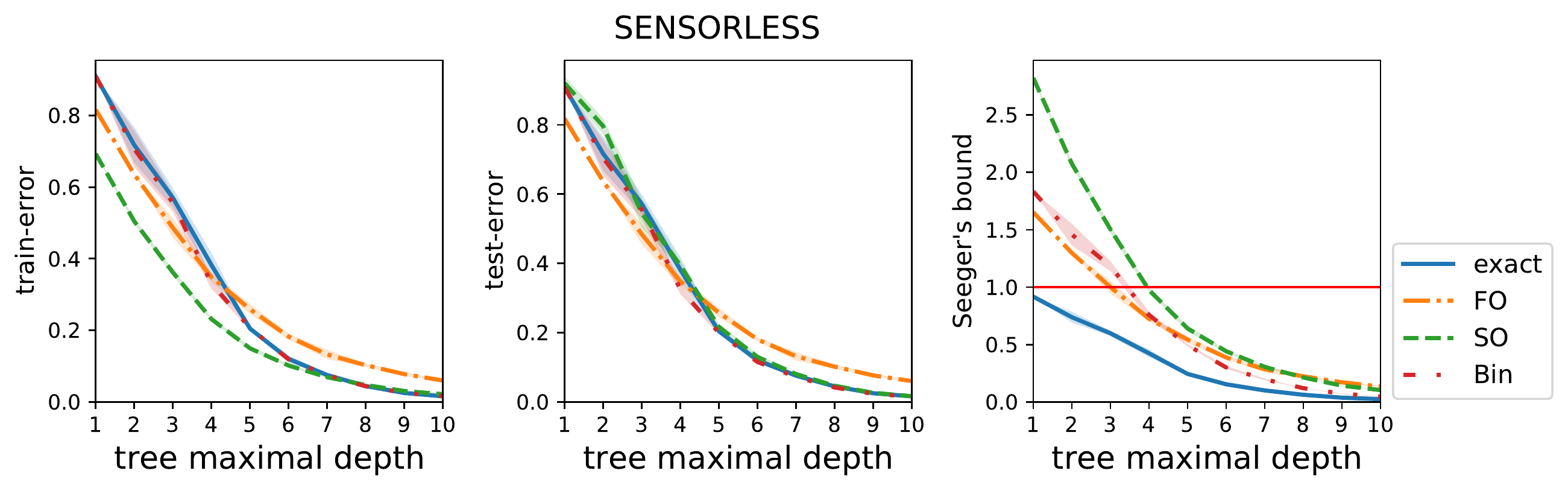}
    \caption{Comparison voter strength (1st column), training error (2nd column), test error (3rd column) and Seeger's bound (4th column) as a function of the tree maximal depth.
    We mark with a red horizontal line the threshold above which the bounds are vacuous.
    Results are averaged over $4$ trials.}\label{fig:strength-app}
\end{figure}

\subsection{Model entropy and complexity, Summary of main results}
We provide two additional elements for assessing the differences between models obtained with our method and models obtained with the upper-bound baselines.
Figure~\ref{fig:real-all} reports the values of the entropy of the obtained posteriors, the KL term in the PAC-Bayes bounds, the training error, the test error and the bound value.
The first measurement (entropy) assesses the diversity of the obtained posteriors: the higher the entropy, the higher the number of selected base classifiers. 
The entropy is generally the highest for our models and the lowest for \emph{FO} which has already been shown to select very few base classifiers.
The second measurement (KL divergence) is provided to verify that our method obtains tighter generalization guarantees because it does not consist in an upper bound of the $01$-loss and not because it obtains models with lower complexity.
Indeed, the posteriors optimized with our variants \emph{exact} and \emph{MC} do not necessarily have low KL divergence w.r.t. the prior.

We finally report the detailed comparison on real benchmarks in Figure~\ref{fig:real-all} and Tables~\ref{tab:binary-real} and~\ref{tab:multic-real}.

\begin{figure}[h]
    \centering
    \includegraphics[width=0.95\textwidth]{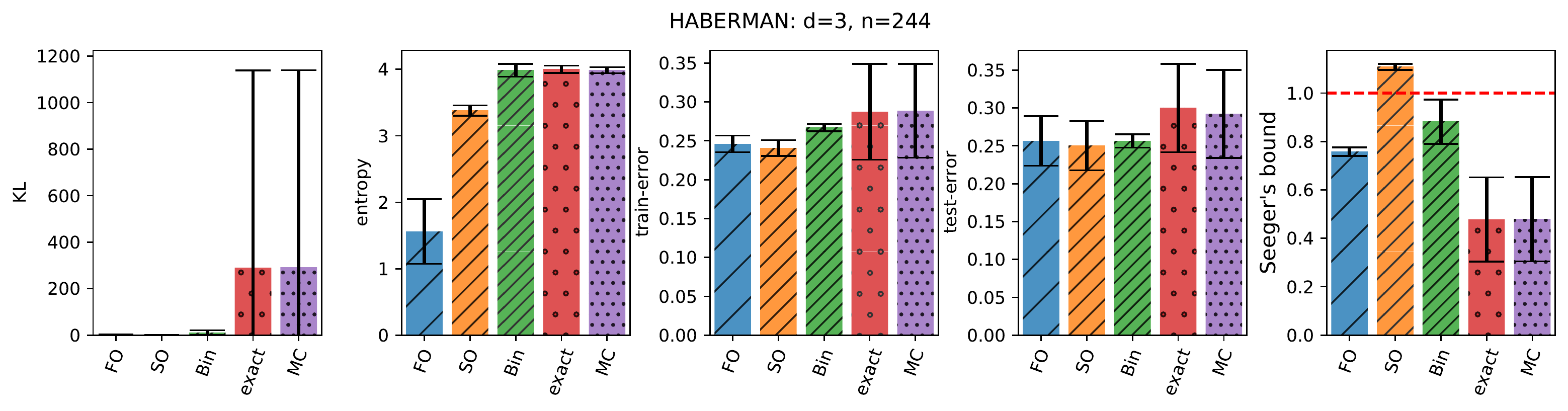}
    \includegraphics[width=0.95\textwidth]{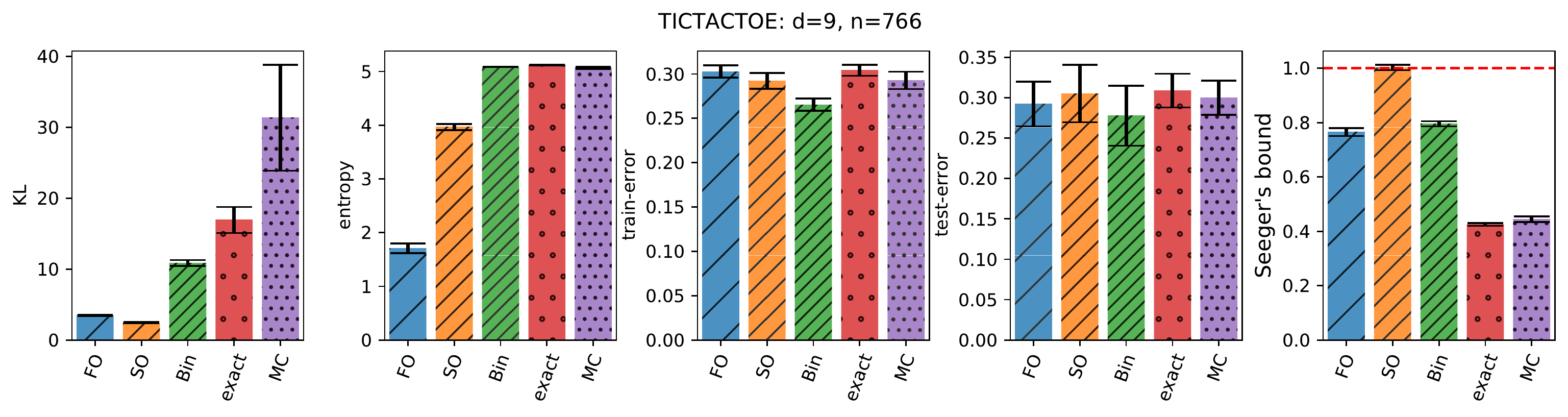}
    \includegraphics[width=0.95\textwidth]{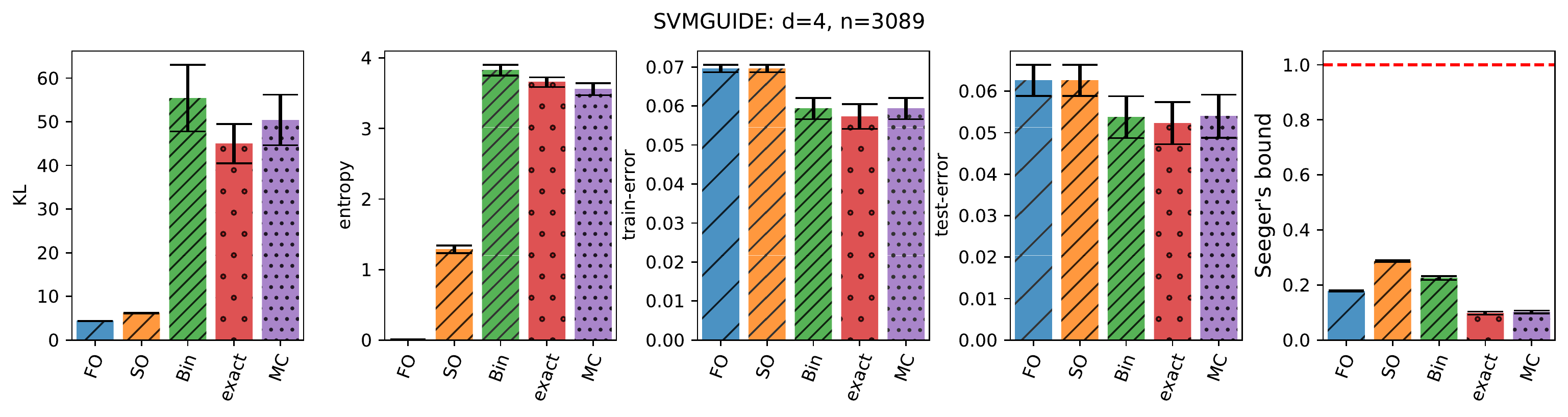}
    \includegraphics[width=0.95\textwidth]{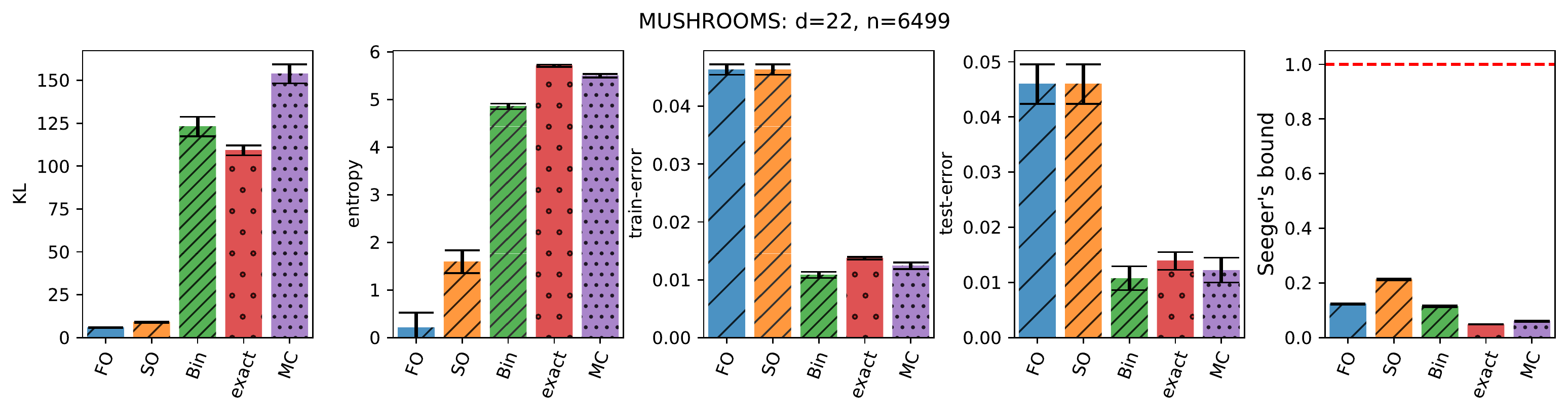}
    \includegraphics[width=0.95\textwidth]{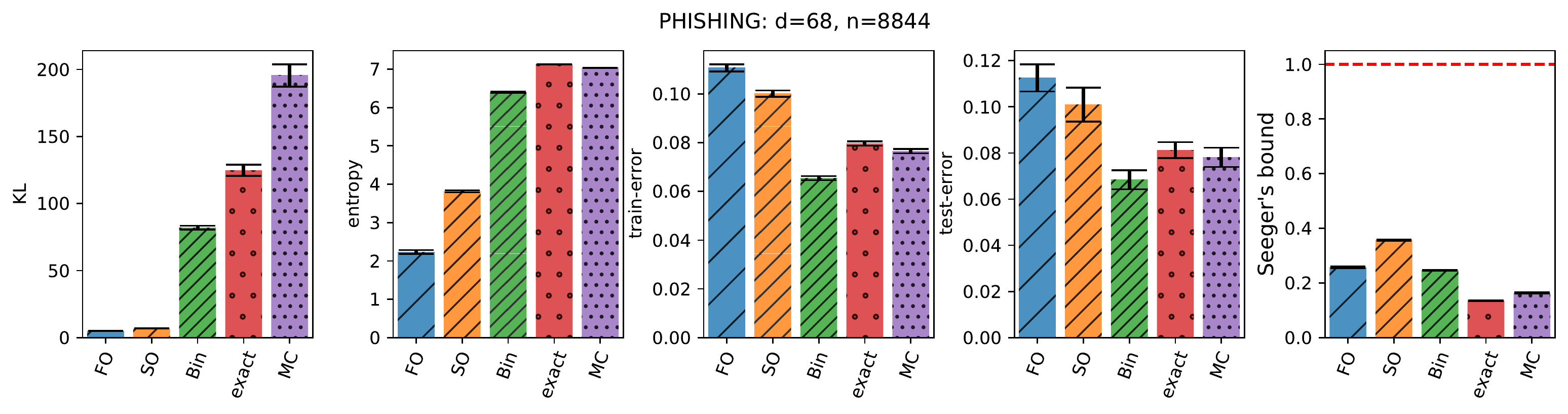}
\end{figure}
\begin{figure}[t!]
    \centering
    \includegraphics[width=0.95\textwidth]{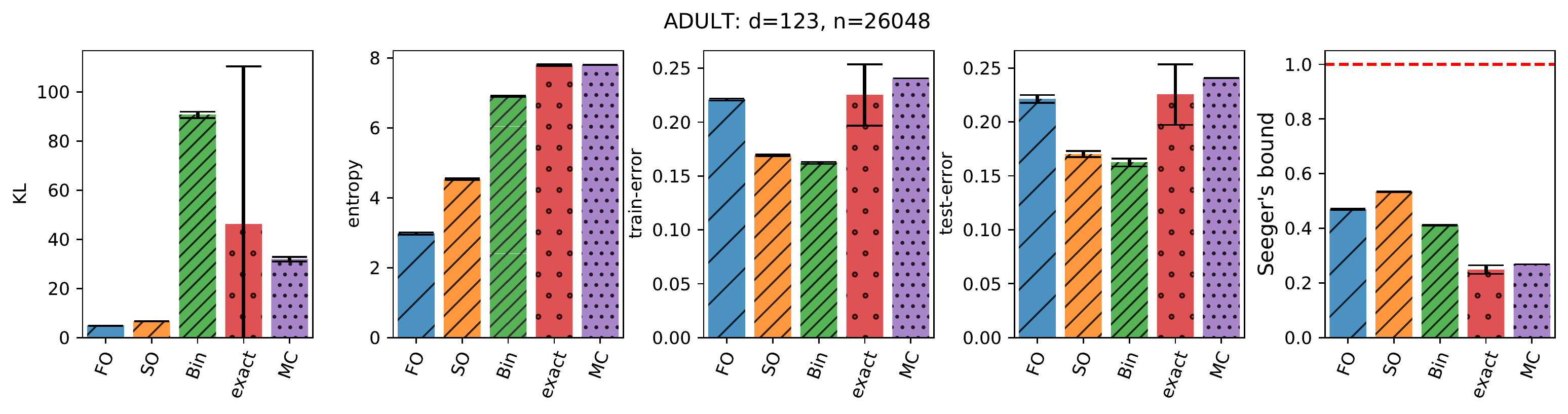}
    \includegraphics[width=0.95\textwidth]{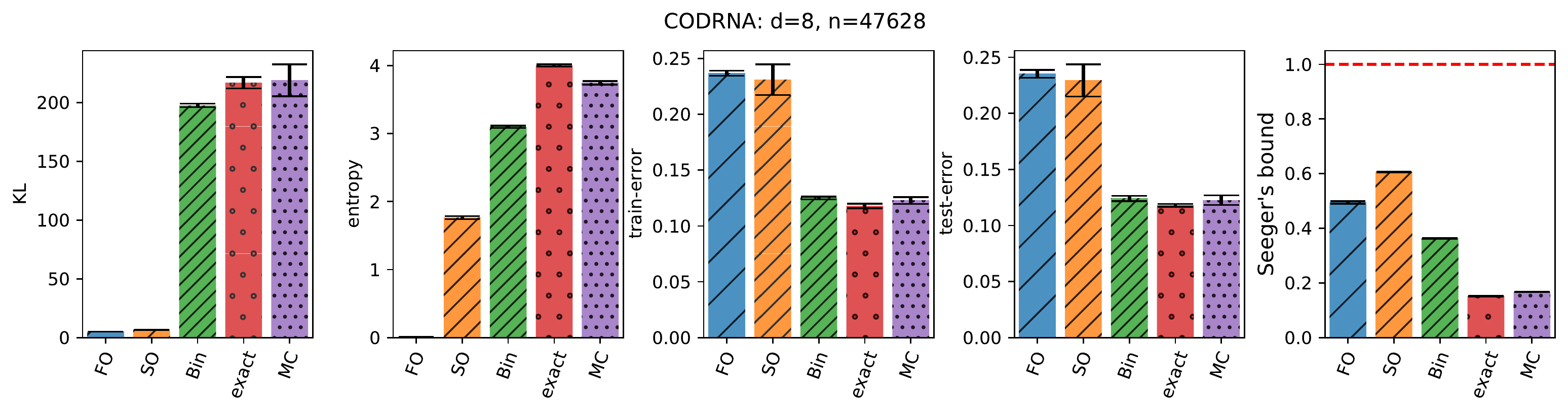}
    \includegraphics[width=0.95\textwidth]{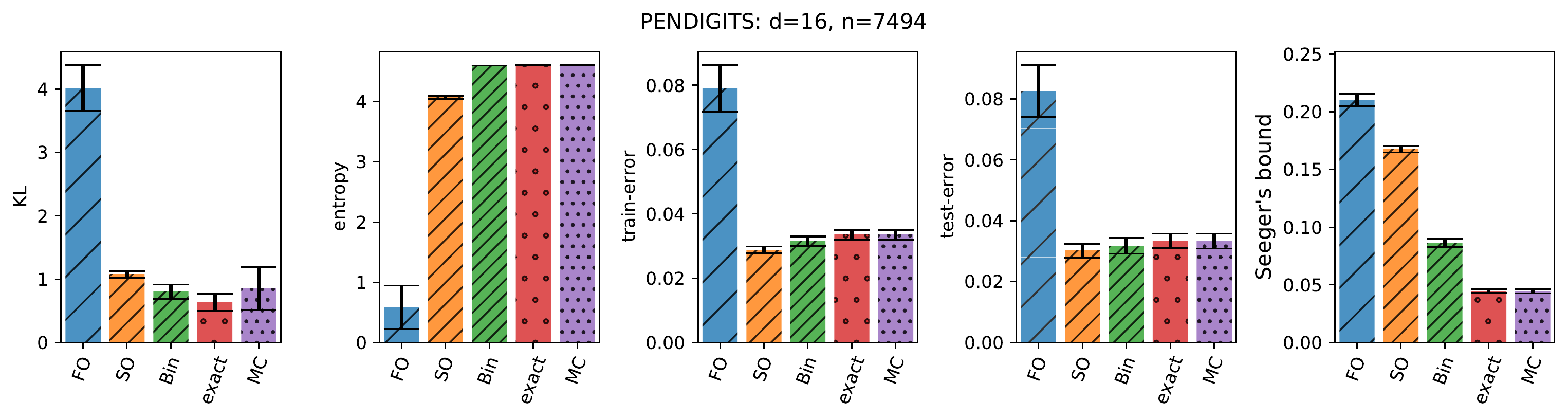}
    \includegraphics[width=0.95\textwidth]{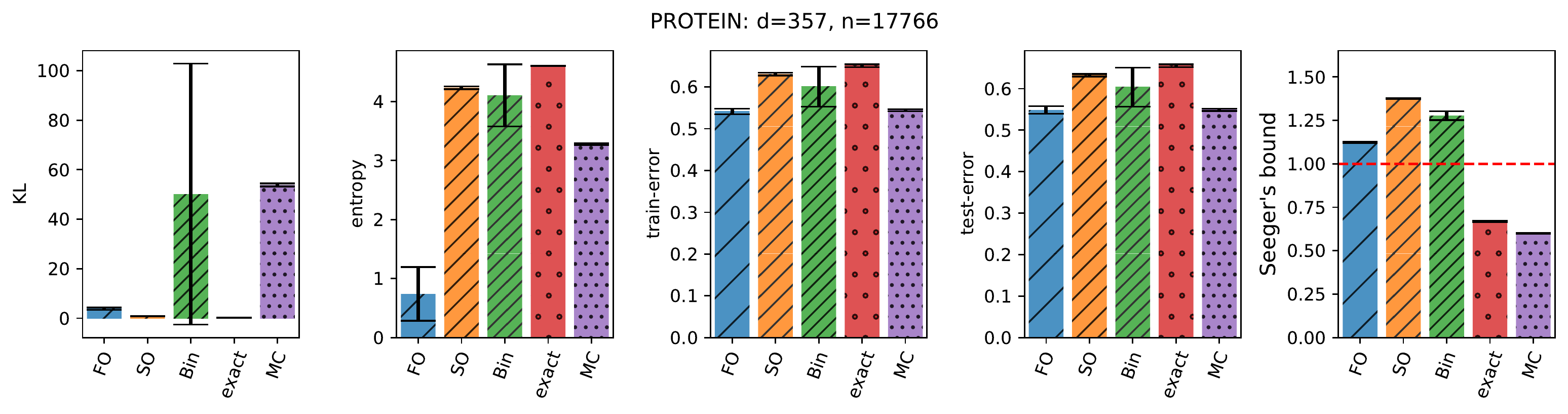}
    \includegraphics[width=0.95\textwidth]{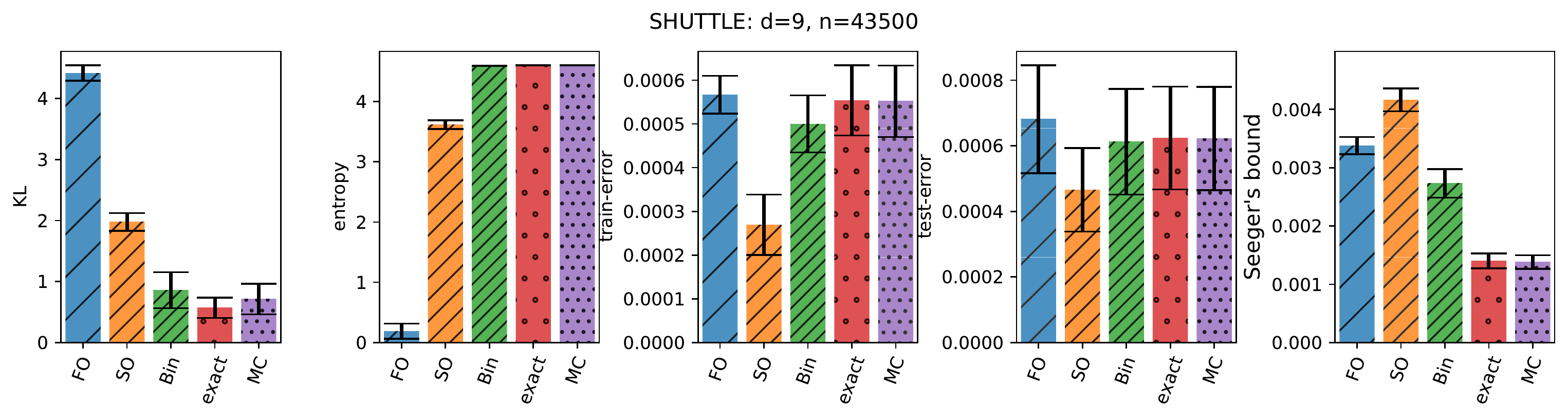}
\end{figure}
\begin{figure}[t!]
    \centering
    \includegraphics[width=0.95\textwidth]{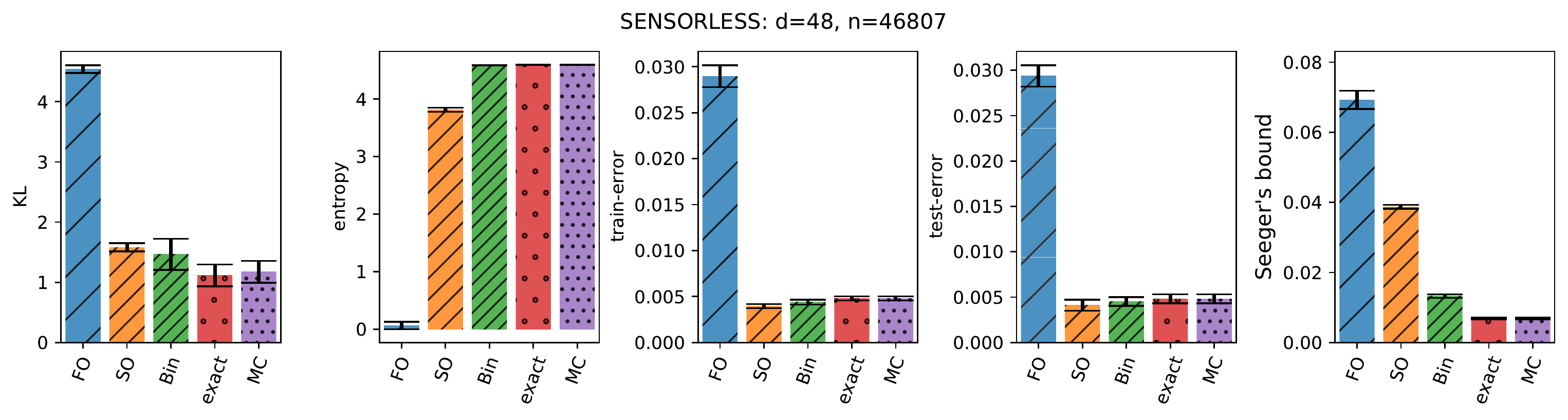}
    \includegraphics[width=0.95\textwidth]{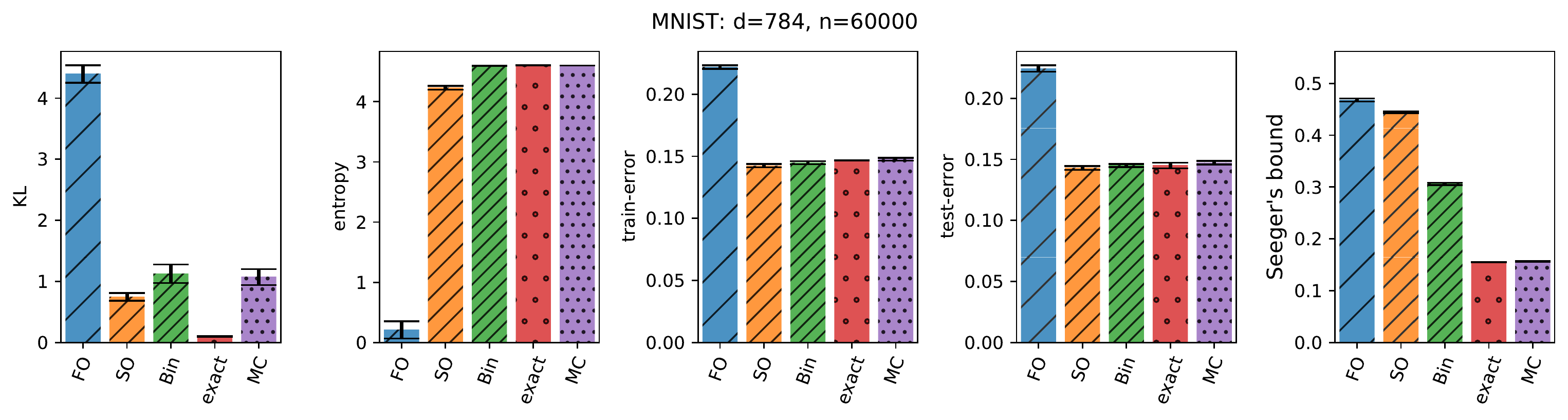}
    \includegraphics[width=0.95\textwidth]{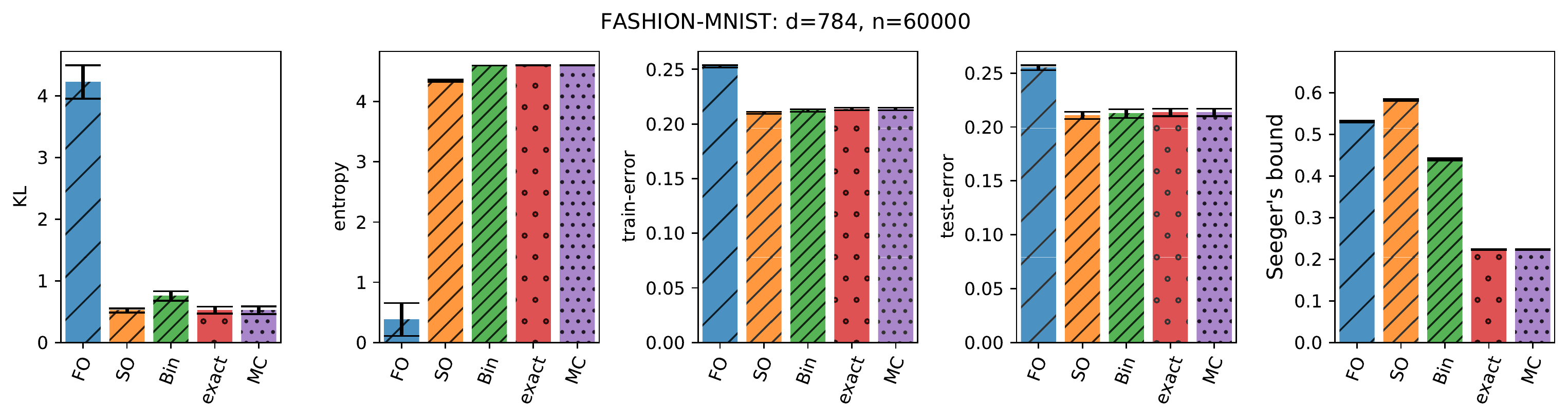}
    \caption{Comparison of First Order (\emph{FO}), Second Order (\emph{SO}) and Binomial (\emph{Bin}) and our methods (\emph{exact}, \emph{MC}) in terms of training and test error rates and PAC-Bayesian bound values. For \emph{exact} and \emph{MC} the entropy is computed for the average MV given the learned Dirichlet distribution.
    Each row of subfigures corresponds to a dataset, where we marked its number of features $d$ and number of training instances $n$.
    The dashed horizontal line in the rightmost column plots marks the threshold above which the bounds are vacuous.
    We report the means (bars) and standard deviations (vertical, black lines) over $10$ different runs.}\label{fig:real-all}
\end{figure}

\label{ap:real}
\begin{landscape}
\begin{table}
  \caption{Summary of the main results on binary datasets with data agnostic prior. For the PAC-Bayesian methods, we report the performance of the model obtained after optimizing Seeger's Bound (Theorem~\eqref{eq:seeger}). Bayesian Naive Bayes (BayesianNB) corresponds to the weighting strategy proposed in~\citet{berend15a}. 
  For each method we report the average ($\pm$ standard deviations) of the test errors and generalization bounds (for the PAC-Bayesian methods) over 10 runs.
  We bold the results (test error and bound value) that are significantly better (smaller) than the other results for a dataset. 
  Notice that the PAC-Bayesian methods generally improve upon BayesianNB and that our method consistently provides the tightest and non-vacuous generalization bounds over all the datasets.}
  \label{tab:binary-real}
  \centering
  \begin{tabular}{clccccccc}
    \toprule
    & Method     & HABERMAN & TICTACTOE & SVMGUIDE & MUSHROOMS & PHISHING & CODRNA & ADULT\\
    \midrule
    \parbox[t]{2mm}{\multirow{6}{*}{\rotatebox[origin=c]{90}{Test error}}}
    & Bayesian NB & $25.65 \pm 2.10$  & $30.21 \pm 3.06$  & $34.51 \pm 0.00$  & $11.65 \pm 0.59$  & $44.49 \pm 0.05$  & $66.67 \pm 0.00$  & $24.07 \pm 0.00$ \\
    & First Order & $25.65 \pm 3.26$  & $29.22 \pm 2.77$  & $6.26 \pm 0.37$  & $4.60 \pm 0.36$  & $11.26 \pm 0.58$  & $23.52 \pm 0.35$  & $22.14 \pm 0.36$ \\
    & Second Order & $25.00 \pm 3.25$  & $30.52 \pm 3.55$  & $6.26 \pm 0.37$  & $4.60 \pm 0.36$  & $10.10 \pm 0.73$  & $22.93 \pm 1.43$  & $17.02 \pm 0.29$ \\
    & Binomial & $25.65 \pm 0.87$  & $27.76 \pm 3.71$  & $\bm{5.37 \pm 0.50}$  & $\bm{1.08 \pm 0.22}$  & $\bm{6.84 \pm 0.41}$  & $12.41 \pm 0.23$  & $\bm{16.24 \pm 0.37}$ \\
    & ours-exact & $30.00 \pm 5.83$  & $30.88 \pm 2.08$  & $\bm{5.23 \pm 0.51}$  & $\bm{1.39 \pm 0.16}$  & $8.13 \pm 0.35$  & $\bm{11.80 \pm 0.13}$  & $22.54 \pm 2.81$ \\
    & ours-MC & $29.22 \pm 5.81$  & $30.00 \pm 2.11$  & $\bm{5.40 \pm 0.52}$  & $\bm{1.22 \pm 0.23}$  & $7.81 \pm 0.42$  & $12.25 \pm 0.43$  & $24.07 \pm 0.00$ \\

    \midrule
    \parbox[t]{2mm}{\multirow{5}{*}{\rotatebox[origin=c]{90}{Seeger's b.}}}
    & First Order & $75.76 \pm 1.75$  & $76.52 \pm 1.44$  & $17.84 \pm 0.21$  & $12.31 \pm 0.19$  & $25.75 \pm 0.31$  & $49.45 \pm 0.47$  & $46.93 \pm 0.18$ \\
    & Second Order & $110.83 \pm 1.24$  & $100.25 \pm 0.99$  & $28.67 \pm 0.33$  & $21.27 \pm 0.31$  & $35.67 \pm 0.21$  & $60.59 \pm 0.20$  & $53.33 \pm 0.09$ \\
    & Binomial & $88.17 \pm 9.12$  & $79.55 \pm 0.90$  & $22.52 \pm 0.70$  & $11.49 \pm 0.27$  & $24.64 \pm 0.20$  & $36.24 \pm 0.13$  & $41.06 \pm 0.13$ \\
    & ours-exact & $\bm{47.83 \pm 17.40}$  & $\bm{42.54 \pm 0.52}$  & $\bm{9.79 \pm 0.57}$  & $\bm{4.85 \pm 0.09}$  & $\bm{13.49 \pm 0.15}$  & $\bm{15.17 \pm 0.21}$  & $\bm{24.87 \pm 1.56}$ \\
    & ours-MC & $\bm{47.89 \pm 17.38}$  & $44.42 \pm 1.02$  & $\bm{10.21 \pm 0.54}$  & $5.99 \pm 0.16$  & $16.34 \pm 0.19$  & $16.72 \pm 0.12$  & $26.83 \pm 0.03$ \\
    \bottomrule
  \end{tabular}
\end{table}
\end{landscape}

\begin{landscape}
\begin{table}
  \caption{Summary of the main results on multi-class datasets with informed priors. For the PAC-Bayesian methods, we report the performance of the model obtained after optimizing Seeger's Bound with informed priors (Theorem~\eqref{eq:informed-prior}). Bayesian Naive Bayes (BayesianNB) corresponds to the weighting strategy proposed in~\citet{berend15a}. 
  For each method we report the average ($\pm$ standard deviations) of the test errors and generalization bounds (for the PAC-Bayesian methods) over $10$ runs.
  We bold the results (test error and bound value) that are significantly better (smaller) than the other results for a dataset. 
  Notice that the PAC-Bayesian methods generally improve upon BayesianNB and that our method consistently provides the tightest and non-vacuous generalization bounds over all the datasets.}
  \label{tab:multic-real}
  \centering
  \begin{tabular}{clcccccc}
    \toprule
    & Method     & PENDIGITS & PROTEIN & SHUTTLE & SENSORLESS & MNIST & FASHION-MNIST\\
    \midrule
    \parbox[t]{2mm}{\multirow{6}{*}{\rotatebox[origin=c]{90}{Test error}}}
    & Bayesian NB & $82.41 \pm 0.11$  & $82.09 \pm 0.00$  & $20.90 \pm 0.03$  & $86.27 \pm 0.21$  & $86.09 \pm 0.09$  & $89.14 \pm 0.05$ \\
    & First Order & $8.25 \pm 0.85$  & $\bm{54.86 \pm 0.93}$  & $0.07 \pm 0.02$  & $2.94 \pm 0.12$  & $22.44 \pm 0.26$  & $25.50 \pm 0.23$ \\
    & Second Order & $3.01 \pm 0.23$  & $63.26 \pm 0.30$  & $0.05 \pm 0.01$  & $0.41 \pm 0.06$  & $14.29 \pm 0.16$  & $21.08 \pm 0.34$ \\
    & Binomial & $3.17 \pm 0.26$  & $60.37 \pm 4.69$  & $0.06 \pm 0.02$  & $0.45 \pm 0.05$  & $14.49 \pm 0.13$  & $21.25 \pm 0.41$ \\
    & ours-exact & $3.33 \pm 0.24$  & $65.54 \pm 0.33$  & $0.06 \pm 0.02$  & $0.48 \pm 0.05$  & $14.49 \pm 0.23$  & $21.34 \pm 0.35$ \\
    & ours-MC & $3.33 \pm 0.24$  & $\bm{54.92 \pm 0.28}$  & $0.06 \pm 0.02$  & $0.48 \pm 0.05$  & $14.73 \pm 0.15$  & $21.34 \pm 0.35$ \\
    \midrule
    \parbox[t]{2mm}{\multirow{5}{*}{\rotatebox[origin=c]{90}{Seeger's b.}}}
    & First Order & $21.04 \pm 0.52$  & $112.34 \pm 0.25$  & $0.34 \pm 0.01$  & $6.92 \pm 0.26$  & $46.81 \pm 0.30$  & $53.08 \pm 0.21$ \\
    & Second Order & $16.77 \pm 0.28$  & $137.60 \pm 0.26$  & $0.42 \pm 0.02$  & $3.87 \pm 0.06$  & $44.44 \pm 0.18$  & $58.26 \pm 0.19$ \\
    & Binomial & $8.64 \pm 0.35$  & $127.74 \pm 2.59$  & $0.27 \pm 0.02$  & $1.33 \pm 0.05$  & $30.59 \pm 0.23$  & $43.99 \pm 0.23$ \\
    & ours-exact & $\bm{4.46 \pm 0.18}$  & $66.86 \pm 0.34$  & $\bm{0.14 \pm 0.01}$  & $\bm{0.69 \pm 0.03}$  & $\bm{15.50 \pm 0.02}$  & $\bm{22.35 \pm 0.12}$ \\
    & ours-MC & $\bm{4.45 \pm 0.17}$  & $\bm{59.93 \pm 0.24}$  & $\bm{0.14 \pm 0.01}$  & $\bm{0.69 \pm 0.03}$  & $\bm{15.64 \pm 0.11}$  & $\bm{22.33 \pm 0.12}$ \\

    \bottomrule
  \end{tabular}
\end{table}
\end{landscape}

\end{document}